%% file: SurveyBCB12.tex
\documentclass[11pt]{hackednow}

\usepackage{amsmath}
\usepackage{amssymb,amsbsy}
\usepackage{graphicx}
\usepackage{booktabs}
\usepackage{dsfont}
\usepackage{natbib}
\usepackage{pdfsync}
\usepackage{hyperref}

\input{Commands}

\input{ncbCommands}

\begin{document}

\isbn{xxxxxxxxxxx}

\DOI{xxxxxx}


\abstract{
Multi-armed bandit problems are the most basic examples of sequential decision problems with an exploration-exploitation trade-off. This is the balance between staying with the option that gave highest payoffs in the past and exploring new options that might give higher payoffs in the future. Although the study of bandit problems dates back to the Thirties, exploration-exploitation trade-offs arise in several modern applications, such as ad placement, website optimization, and packet routing. Mathematically, a multi-armed bandit is defined by the payoff process associated with each option. In this survey, we focus on two extreme cases in which the analysis of regret is particularly simple and elegant: i.i.d.\ payoffs and adversarial payoffs. Besides the basic setting of finitely many actions, we also analyze some of the most important variants and extensions, such as the contextual bandit model.
}

\articletitle{Regret Analysis of Stochastic and Nonstochastic Multi-armed Bandit Problems}

\authorname1{S\'ebastien Bubeck}
\authorname2{Nicol{\`o} Cesa-Bianchi}

\author1address2ndline{Department of Operations Research and Financial Engineering, Princeton University, Princeton 08544, USA, sbubeck@princeton.edu}
\author2address2ndline{Dipartimento di Informatica, Universit{\`a} degli Studi di Milano, Milano 20135, Italy, nicolo.cesa-bianchi@unimi.it}

\relax \input sample.jtl 
\volume{xx}
\issue{xx}
\copyrightowner{xxxxxxxxx}
\pubyear{xxxx}

\maketitle

\cleardoublepage \pagenumbering{roman}

\tableofcontents

\clearpage

\setcounter{page}{0}
\pagenumbering{arabic}

\newtheorem{theorem}{Theorem}[chapter]
\newtheorem{definition}{Definition}[chapter]

\chapter{Introduction}
\label{intro}
\input{introduction}

\chapter{Stochastic bandits: fundamental results}
\label{stochastic}
\input{stochastic}

\chapter{Adversarial bandits: fundamental results}
\label{adversarial}
\input{adversarial}

\chapter{Contextual bandits}
\label{contextual}
\input{contextual}

\chapter{Linear bandits}
\label{linear}
\input{linear}

\chapter{Nonlinear bandits}
\label{nonlinear}
\input{nonlinear}

\chapter{Variants}
\label{variants}
\input{variants}


\begin{acknowledgements}
We would like to thank Mike Jordan for proposing to write this survey and James Finlay for keeping us on track. The table of contents was laid down with the help of G\'abor Lugosi.
We would also like to express our warm thanks to the reviewers, whose insightful comments have led to many substantial improvements. And a special thank goes to the ``non-anonymous'' referee Alexandrs Slivkins, whose review stands out for breadth and depth. Thank you Alexandrs, you have done a really exceptional job!

Nicol\`o Cesa-Bianchi gratefully acknowledges partial support by the PASCAL2 Network of Excellence under EC grant no.\ 216886.
\end{acknowledgements}

\bibliographystyle{plainnat}
\bibliography{newbib}

\end{document}

%% file: Commands.tex
\newtheorem{lemma}{Lemma}[chapter]

\newcommand{\K}{\mathrm{KL}}
\newcommand{\kl}{\mathrm{kl}}

\renewcommand{\phi}{\varphi}

\renewcommand{\P}{\mathbb{P}}
\newcommand{\E}{\mathbb{E}}
\newcommand{\N}{\mathbb{N}}
\newcommand{\R}{\mathbb{R}}

\newcommand{\cF}{\mathcal{F}}
\newcommand{\cL}{\mathcal{L}}
\newcommand{\cX}{\mathcal{X}}

\newcommand{\cC}{\mathcal{C}}
\newcommand{\cS}{\mathcal{S}}
\newcommand{\cE}{\mathcal{E}}
\newcommand{\cK}{\mathcal{K}}
\newcommand{\cD}{\mathcal{D}}

\newcommand{\oD}{\overline{\mathcal{D}}}
\newcommand{\oR}{\overline{R}}

\def\ds1{\mathds{1}}
\renewcommand{\epsilon}{\varepsilon}

\newcommand{\wh}{\widehat}
\newcommand{\argmax}{\mathop{\mathrm{argmax}}}
\newcommand{\argmin}{\mathop{\mathrm{argmin}}}

\renewcommand{\tilde}{\widetilde}

\newlength{\minipagewidth}
\newcommand{\bookbox}[1]{
\par\medskip\noindent
\framebox[\textwidth]{
\begin{minipage}{\minipagewidth}
{#1}
\end{minipage} } \par\medskip }

\newcommand{\beq}{\begin{equation}}
\newcommand{\eeq}{\end{equation}}

\newcommand{\beqa}{\begin{eqnarray}}
\newcommand{\eeqa}{\end{eqnarray}}

\newcommand{\beqan}{\begin{eqnarray*}}
\newcommand{\eeqan}{\end{eqnarray*}}

\def\ba#1\ea{\begin{align*}#1\end{align*}} 
\def\banum#1\eanum{\begin{align}#1\end{align}} 

%% file: ncbCommands.tex
\newcommand{\dt}{\displaystyle}
\newcommand{\arginf}{\mathop{\mathrm{arginf}}}
\newcommand{\theset}[2]{\left\{{#1}\,:\,{#2}\right\}}
\newcommand{\inner}[2]{\bigl\langle{#1},{#2}\bigr\rangle}
\newcommand{\ve}{\varepsilon}
\newcommand{\scO}{\mathcal{O}}
\newcommand{\conv}{\mathrm{Conv}}

\newcommand{\Ball}{\mathbb{B}}
\newcommand{\Sphere}{\mathbb{S}}
\newcommand{\norm}[1]{\left\|{#1}\right\|}
\newcommand{\oRx}{\oR^{\mathrm{ctx}}}
\newcommand{\Rx}{R^{\mathrm{ctx}}}
\newcommand{\oRS}{\oR^{\cS}}

\newcommand{\xhat}{\wh{x}}
\newcommand{\yhat}{\wh{y}}

\newcommand{\muhat}{\wh{\mu}}

\newcommand{\tX}{\widetilde{X}}
\newcommand{\tloss}{\widetilde{\ell}}
\newcommand{\gtilde}{\widetilde{g}}

%% file: introduction.tex
A multi-armed bandit problem (or, simply, a bandit problem) is a sequential allocation problem defined by a set of actions. At each time step, a unit resource is allocated to an action and some observable payoff is obtained. The goal is to maximize the total payoff obtained in a sequence of allocations. The name \textsl{bandit} refers to the colloquial term for a slot machine (``one-armed bandit'' in American slang). In a casino, a sequential allocation problem is obtained when the player is facing many slot machines at once (a ``multi-armed bandit''), and must repeatedly choose where to insert the next coin.

Bandit problems are basic instances of sequential decision making with limited information, and naturally address the fundamental trade-off between exploration and exploitation in sequential experiments. Indeed, the player must balance the exploitation of actions that did well in the past and the exploration of actions that might give higher payoffs in the future.

Although the original motivation of \cite{Tho33} for studying bandit problems came from clinical trials (when different treatments are available for a certain disease and one must decide which treatment to use on the next patient), modern technologies have created many opportunities for new applications, and bandit problems now play an important role in several industrial domains. In particular, online services are natural targets for bandit algorithms, because there one can benefit from adapting the service to the individual sequence of requests. We now describe a few concrete examples in various domains.

Ad placement is the problem of deciding which advertisement to display on the web page delivered to the next visitor of a website. Similarly, website optimization deals with the problem of sequentially choosing design elements (font, images, layout) for the web page. Here the payoff is associated with visitor's actions, e.g., clickthroughs or other desired behaviors. Of course there are important differences with the basic bandit problem: in ad placement the pool of available ads (bandit arms) may change over time, and there might be a limit on the number of times each ad could be displayed.

In source routing a sequence of packets must be routed from a source host to a destination host in a given network, and the protocol allows to choose a specific source-destination path for each packet to be sent. The (negative) payoff is the time it takes to deliver a packet, and depends additively on the congestion of the edges in the chosen path.

In computer game-playing, each move is chosen by simulating and evaluating many possible game continuations after the move. Algorithms for bandits (more specifically, for a tree-based version of the bandit problem) can be used to explore more efficiently the huge tree of game continuations by focusing on the most promising subtrees. This idea has been successfully implemented in the MoGo player of \cite{GWMT06}, which plays Go at world-class level. MoGo is based on the UCT strategy for hierarchical bandits of \cite{KS06}, which is in turn derived from the UCB bandit algorithm ---see Chapter~\ref{stochastic}.

%
There are three fundamental formalizations of the bandit problem depending on the assumed nature of the reward process: stochastic, adversarial, and Markovian. Three distinct playing strategies have been shown to effectively address each specific bandit model: the UCB algorithm in the stochastic case, the Exp3 randomized algorithm in the adversarial case, and the so-called Gittins indices in the Markovian case. In this survey, we focus on stochastic and adversarial bandits, and refer the reader to the survey by \cite{mahajan2008multi} or to the recent monograph by \cite{GGW11} for an extensive analysis of Markovian bandits.

In order to analyze the behavior of a player or forecaster (i.e., the agent implementing a bandit strategy), we may compare its performance with that of an optimal strategy that, for any horizon of $n$ time steps, consistently plays the arm that is best in the first $n$ steps. In other terms, we may study the \textsl{regret} of the forecaster for not playing always optimally. More specifically, given $K \ge 2$ arms and given sequences $X_{i,1},X_{i,2},\dots$ of unknown rewards associated with each arm $i=1,\dots,K$, we study forecasters that at each time step $t=1,2,\dots$ select an arm $I_t$ and receive the associated reward $X_{I_t,t}$. The regret after $n$ plays $I_1,\dots,I_n$ is defined by
\begin{equation} \label{eq:regret}
    R_n = \max_{i=1,\hdots,K} \sum_{t=1}^n X_{i,t} - \sum_{t=1}^n X_{I_t,t}~.
\end{equation}
If the time horizon is not known in advance we say that the forecaster is {\em anytime}.

In general, both rewards $X_{i,t}$ and forecaster's choices $I_t$ might be stochastic. This allows to distinguish between the two following notions of averaged regret: the \textsl{expected regret}
\begin{equation} \label{eq:exp-regret}
    \E\,R_n = \E\left[\max_{i=1,\hdots,K} \sum_{t=1}^n X_{i,t} - \sum_{t=1}^n X_{I_t,t}\right]
\end{equation}
and the \textsl{pseudo-regret}
\begin{equation} \label{eq:pseudoregret}
\overline{R}_n = \max_{i=1,\dots,K} \E\left[ \sum_{t=1}^n X_{i,t} - \sum_{t=1}^n X_{I_t,t}\right]~.
\end{equation}
In both definitions, the expectation is taken with respect to the random draw of both rewards and forecaster's actions. Note that pseudo-regret is a weaker notion of regret, since one compares to the optimal action in expectation. The expected regret, instead, is the expectation of the regret with respect to the action which is optimal on the sequence of reward realizations. More formally one has $\oR_n \leq \E R_n$. 

In the original formalization of \cite{Rob52}, which builds on the work of~\cite{wald1947sequential} ---see also~\cite{arrow1949bayes}, each arm $i=1,\dots,K$ corresponds to an unknown probability distribution $\nu_i$ on $[0,1]$, and rewards $X_{i,t}$ are independent draws from the distribution $\nu_i$ corresponding to the selected arm.
\begin{center}
\bookbox{\small
\textbf{The stochastic bandit problem}

\smallskip\noindent
\textsl{Known parameters:} number of arms $K$ and (possibly) number of rounds $n \ge K$. 
\\
\textsl{Unknown parameters:} $K$ probability distributions $\nu_1,\ldots,\nu_K$ on $[0,1]$.

\smallskip\noindent
For each round $t=1,2,\ldots$
\begin{itemize}
\item[(1)]
the forecaster chooses $I_t \in \{1,\ldots,K \}$;
\item[(2)]
given $I_t$, the environment draws the reward $X_{I_t,t} \sim \nu_{I_t}$ independently from the past and reveals it to the forecaster.
\end{itemize}
%
}
\end{center}
%
For $i=1,\dots,K$ we denote by $\mu_i$ the mean of $\nu_i$ (mean reward of arm $i$).
Let
\[
    \mu^*=\max_{i=1,\hdots,K} \mu_i \qquad\text{and}\qquad i^* \in \argmax_{i=1,\hdots,K} \mu_i~.
\]
In the stochastic setting, it is easy to see that the pseudo-regret can be written as
\begin{equation} \label{eq:pregretstocha}
\overline{R}_n = n \mu^* - \sum_{t=1}^n \E\bigl[\mu_{I_t}\bigr]~.
\end{equation}
The analysis of the stochastic bandit model was pioneered in the seminal paper of~\cite{LR85}, who introduced the technique of upper confidence bounds for the asymptotic analysis of regret. In Chapter~\ref{stochastic} we describe this technique using the simpler formulation of~\cite{Agr95}, which naturally lends itself to a finite-time analysis.

In parallel to the research on stochastic bandits, a game-theoretic formulation of the trade-off between exploration and exploitation has been independently investigated, although for quite some time this alternative formulation was not recognized as an instance of the multi-armed bandit problem. In order to motivate these game-theoretic bandits, consider again the initial example of gambling on slot machines. We now assume that we are in a rigged casino, where for each slot machine $i=1,\hdots,K$ and time step $t \ge 1$ the owner sets the gain $X_{i,t}$ to some arbitrary (and possibly maliciously chosen) value $g_{i,t} \in [0,1]$. Note that it is not in the interest of the owner to simply set all the gains to zero (otherwise, no gamblers would go to that casino). Now recall that a forecaster selects sequentially one arm $I_t \in \{1,\hdots,K\}$ at each time step $t=1,2,\dots$ and observes (and earns) the gain $g_{I_t,t}$. Is it still possible to minimize regret in such a setting?

Following a standard terminology, we call adversary, or opponent, the mechanism setting the sequence of gains for each arm. If this mechanism is independent of the forecaster's actions, then we call it an \textsl{oblivious} adversary. In general, however, the adversary may adapt to the forecaster's past behaviour, in which case we speak of a \textsl{non-oblivious} adversary. For instance, in the rigged casino the owner may observe the way a gambler plays in order to design even more evil sequences of gains.
Clearly, the distinction between oblivious and non-oblivious adversary is only meaningful when the player is randomized (if the player is deterministic, then the adversary can pick a bad sequence of gains right at the beginning of the game by simulating the player's future actions). Note, however, that in presence of a non-oblivious adversary the interpretation of regret is ambiguous. Indeed, in this case the assignment of gains $g_{i,t}$ to arms $i=1,\dots,K$ made by the adversary at each step $t$ is allowed to depend on the player's past randomized actions $I_1,\dots,I_{t-1}$. In other words, $g_{i,t} = g_{i,t}(I_1,\dots,I_{t-1})$ for each $i$ and $t$. Now, the regret compares the player's cumulative gain to that obtained by playing the single best arm for the first $n$ rounds. However, had the player consistently chosen the same arm $i$ in each round, namely $I_t = i$ for $t=1,\dots,n$, the adversarial gains $g_{i,t}(I_1,\dots,I_{t-1})$ would have been possibly different than those actually experienced by the player.

The study of non-oblivious regret is mainly motivated by the connection between regret minimization and equilibria in games ---see, e.g.~\cite[Section~9]{ACFS03}. Here we just observe that game-theoretic equilibria are indeed defined similarly to regret: in equilibrium, the player has nSo incentive to behave differently provided the opponent does not react to changes in the player's behaviour. Interestingly, regret minimization has been also studied against \textsl{reactive opponents}, see for instance the works of~\cite{farias2006combining} and~\cite{arora2012online}.
\begin{center}
\bookbox{\small
\textbf{The adversarial bandit problem}

\smallskip\noindent
\textsl{Known parameters:} number of arms $K \ge 2$ and (possibly) number of rounds $n \ge K$.

\smallskip\noindent
For each round $t=1,2,\ldots$
\begin{itemize}
\item[(1)]
the forecaster chooses $I_t \in \{1,\ldots,K \}$, possibly with the help of external randomization;
\item[(2)]
simultaneously, the adversary selects a gain vector $g_t=(g_{1,t},\hdots,g_{K,t}) \in [0,1]^K$, possibly with the help of external randomization;
\item[(3)]
the forecaster receives (and observes) the reward $g_{I_t,t}$, while the gains of the other arms are not observed.
\end{itemize}
%
}
\end{center}
In this adversarial setting the goal is to obtain regret bounds in high probability or in expectation with respect to any possible randomization in the strategies used by the forecaster or the opponent, and irrespective of the opponent. In the case of a non-oblivious adversary this is not an easy task, and for this reason we usually start by bounding the pseudo-regret
%
\[
\overline{R}_n = \max_{i=1,\dots,K} \E\left[\sum_{t=1}^n g_{i,t} - \sum_{t=1}^n g_{I_t,t}\right]~.
\]
%
Note that the randomization of the adversary is not very important here since we ask for bounds which hold for any opponent.
On the other hand, it is fundamental to allow randomization for the forecaster ---see Chapter~\ref{adversarial} for details and basic results in the adversarial bandit model.
This adversarial, or non-stochastic, version of the bandit problem was originally proposed as a way of playing an unknown game against an opponent. The problem of playing a game repeatedly, now a classical topic in game theory, was initiated by the groundbreaking work of James Hannan and David Blackwell. In Hannan's seminal paper~\cite{Han57}, the game (i.e., the payoff matrix) is assumed to be known by the player, who also observes the opponent's moves in each play. Later, \cite{Ban68} considered the problem of a repeated unknown game, where in each game round the player only observes its own payoff. This problem turns out to be exactly equivalent to the adversarial bandit problem with a non-oblivious adversary. Simpler strategies for playing unknown games were more recently proposed by \cite{FoVo98} and \cite{HaMa99,HaMa00}. Approximately at the same time, the problem was re-discovered in computer science by~\cite{ACFS03}. It was them who made apparent the connection to stochastic bandits by coining the term nonstochastic multi-armed bandit problem. 

The third fundamental model of multi-armed bandits assumes that the reward processes are neither i.i.d.\ (like in stochastic bandits) nor adversarial. More precisely, arms are associated with $K$ Markov processes, each with its own state space. Each time an arm $i$ is chosen in state $s$, a stochastic reward is drawn from a probability distribution $\nu_{i,s}$, and the state of the reward process for arm $i$ changes in a Markovian fashion, based on an underlying stochastic transition matrix $M_i$. Both reward and new state are revealed to the player. On the other hand, the state of arms that are not chosen remains unchanged. Going back to our initial interpretation of bandits as sequential resource allocation processes, here we may think of $K$ competing projects that are sequentially allocated a unit resource of work. However, unlike the previous bandit models, in this case the state of a project that gets the resource may change. Moreover, the underlying stochastic transition matrices $M_i$ are typically assumed to be known, thus the optimal policy can be computed via dynamic programming and the problem is essentially of computational nature. The seminal result of~\cite{gittins1979bandit} provides an optimal greedy policy which can be computed efficiently.

A notable special case of Markovian bandits is that of Bayesian bandits. These are parametric stochastic bandits, where the parameters of the reward distributions are assumed to be drawn from known priors, and the regret is computed by also averaging over the draw of parameters from the prior. The Markovian state change associated with the selection of an arm corresponds here to updating the posterior distribution of rewards for that arm after observing a new reward.

Markovian bandits are a standard model in the areas of Operations Research and Economics. However, the techniques used in their analysis are significantly different from those used to analyze stochastic and adversarial bandits. For this reason, in this survey we do not cover Markovian bandits and their many variants.

%% file: stochastic.tex
We start by recalling the basic definitions for the stochastic bandit problem. Each arm $i \in \{1, \hdots, K\}$ corresponds to an unknown probability distribution $\nu_i$. At each time step $t =1, 2, \hdots,$ the forecaster selects some arm $I_t \in \{1, \hdots, K\}$ and receives a reward $X_{I_t,t}$ drawn from $\nu_{I_t}$ (independently from the past). Denote by $\mu_i$ the mean of arm $i$ and define
\[
    \mu^*=\max_{i=1,\hdots,K} \mu_i \qquad\text{and}\qquad i^* \in \argmax_{i=1,\hdots,K} \mu_i~.
\]
We focus here on the pseudo-regret, which is defined as
\begin{equation} \label{eq:pseudoregretstocha}
\oR_n = n \mu^* - \E \sum_{t=1}^n \mu_{I_t}~.
\end{equation}
We choose the pseudo-regret as our main quantity of interest because in a stochastic framework it is more natural to compete against the optimal action in expectation, rather than the optimal action on the sequence of realized rewards (as in the definition of the plain regret~\eqref{eq:regret}). Furthermore, because of the order of magnitude of typical random fluctuations, in general one cannot hope to prove a bound on the expected regret~\eqref{eq:exp-regret} better than $\Theta\bigl(\sqrt{n}\bigr)$. On the contrary, the pseudo-regret can be controlled so well that we are able to bound it by a \textsl{logarithmic} function of $n$.

In the following we also use a different formula for the pseudo-regret. Let $T_{i}(s)=\sum_{t=1}^s \ds1_{I_t=i}$ denote the number of times the player selected arm $i$ on the first $s$ rounds. Let $\Delta_i = \mu^*-\mu_i$ be the suboptimality parameter of arm $i$. Then the pseudo-regret can be written as:
\begin{align*}
\oR_n =
\left( \sum_{i=1}^K \E\,T_i(n) \right) \mu^* - \E \sum_{i=1}^K T_i(n) \mu_i = \sum_{i=1}^K \Delta_i\, \E\,T_i(n)~.
\end{align*}

\section{Optimism in face of uncertainty}
The difficulty of the stochastic multi-armed bandit problem lies in the {\em exploration-exploitation dilemma} that the forecaster is facing. Indeed, there is an intrinsic tradeoff between {\em exploiting} the current knowledge to focus on the arm that seems to yield the highest rewards, and {\em exploring} further the other arms to identify with better precision which arm is actually the best. As we shall see, the key to obtain a good strategy for this problem is, in a certain sense, to simultaneously perform exploration and exploitation.

A simple heuristic principle for doing that is the so-called {\em optimism in face of uncertainty}. The idea is very general, and applies to many sequential decision making problems in uncertain environments. Assume that the forecaster has accumulated some data on the environment and must decide how to act next. First, a set of ``plausible'' environments which are ``consistent'' with the data (typically, through concentration inequalities) is constructed. Then, the most ``favorable'' environment is identified in this set. Based on that, the heuristic prescribes that the decision which is optimal in this most favorable and plausible environment should be made. As we see below, this principle gives simple and yet almost optimal algorithms for the stochastic multi-armed bandit problem. More complex algorithms for various extensions of the stochastic multi-armed bandit problem are also based on the same idea which, along with the exponential weighting scheme presented in Section~\ref{adversarial}, is an algorithmic cornerstone of regret analysis in bandits. 

\section{Upper Confidence Bound (UCB) strategies} \label{sec:UCB}
In this section we assume that the distribution of rewards $X$ satisfy the following moment conditions. There exists a convex function\footnote{
One can easily generalize the discussion to functions $\psi$ defined only on an interval $[0,b]$.
}
$\psi$ on the reals such that, for all $\lambda \geq 0$,
\begin{equation} \label{eq:psicond}
\ln \E\,e^{\lambda \bigl(X-\E[X]\bigr)} \leq \psi(\lambda) \quad\text{and}\quad  \ln\,\E\,e^{\lambda \bigl(\E[X] - X\bigr)} \leq \psi(\lambda)~.
\end{equation}
For example, when $X \in [0,1]$ one can take $\psi(\lambda) = \tfrac{\lambda^2}{8}$. In this case~\eqref{eq:psicond} is known as Hoeffding's lemma.

We attack the stochastic multi-armed bandit using the optimism in face of uncertainty principle. In order do so, we use assumption \eqref{eq:psicond} to construct an upper bound estimate on the mean of each arm at some fixed confidence level, and then choose the arm that looks best under this estimate. We need a standard notion from convex analysis: the Legendre-Fenchel transform of $\psi$, defined by
\[
    \psi^*(\epsilon) = \sup_{\lambda \in \R} \bigl(\lambda \epsilon - \psi(\lambda)\bigr)~.
\]
For instance, if $\psi(x) = e^x$ then $\psi^*(x) = x\ln x - x$ for $x > 0$. If $\psi(x) = \frac{1}{p}|x|^p$ then $\psi^*(x) = \frac{1}{q}|x|^q$ for any pair $1 < p,q < \infty$ such that $\frac{1}{p} + \frac{1}{q} = 1$ ---see also Section~\ref{sec:OMD}, where the same notion is used in a different bandit model. 

Let $\wh{\mu}_{i,s}$ be the sample mean of rewards obtained by pulling arm $i$ for $s$ times. Note that since the rewards are i.i.d., we have that in distribution $\wh{\mu}_{i,s}$ is equal to $\frac{1}{s} \sum_{t=1}^s X_{i,t}$.

Using Markov's inequality, from~\eqref{eq:psicond} one obtains that
\begin{equation} \label{eq:psiconcentration}
\P(\mu_i - \wh{\mu}_{i,s} > \epsilon) \leq e^{- s\,\psi^*(\epsilon)}~.
\end{equation}
In other words, with probability at least $1- \delta$,
$$\wh{\mu}_{i,s} + \left(\psi^*\right)^{-1}\left( \frac{1}{s}\ln\frac{1}{\delta}\right) > \mu_i~.$$
We thus consider the following strategy, called $(\alpha, \psi)$-UCB, where $\alpha >0$ is an input parameter: At time $t$, select
\[
I_t \in \argmax_{i=1,\hdots,K} \left[ \wh{\mu}_{i,T_i(t-1)} + \left(\psi^*\right)^{-1}\left( \frac{\alpha \ln t}{T_i(t-1)}\right) \right]~.
\]
We can prove the following simple bound.
\begin{theorem}[Pseudo-regret of $(\alpha,\psi)$-UCB] \label{th:ucb}
Assume that the reward distributions satisfy \eqref{eq:psicond}. Then $(\alpha, \psi)$-UCB with $\alpha > 2$ satisfies
$$\oR_n \leq \sum_{i \,:\, \Delta_i > 0} \left( \frac{\alpha \Delta_i}{\psi^*(\Delta_i / 2)} \ln n + \frac{\alpha}{\alpha - 2}  \right)~.$$
\end{theorem}
In case of $[0,1]$-valued random variables, taking $\psi(\lambda) = \tfrac{\lambda^2}{8}$ in~\eqref{eq:psicond} ---the Hoeffding's Lemma--- gives $\psi^*(\epsilon) = 2 \epsilon^2$, which in turns gives the following pseudo-regret bound
\begin{equation} \label{eq:regretUCBbounded}
\oR_n \leq \sum_{i : \Delta_i > 0} \left( \frac{2 \alpha}{\Delta_i} \ln n + \frac{\alpha}{\alpha - 2}  \right)~.
\end{equation}
In this important special case of bounded random variables we refer to $(\alpha,\psi)$-UCB simply as $\alpha$-UCB.
\begin{proof}
First note that if $I_t = i$, then at least one of the three following equations must be true:
\begin{align}
\label{eq:UCB1}
& \wh{\mu}_{i^*,T_{i^*}(t-1)} + \left(\psi^*\right)^{-1}\left( \frac{\alpha \ln t}{T_{i^*}(t-1)}\right) \leq \mu^*
\\
\label{eq:UCB2}
& \widehat{\mu}_{i,T_i(t-1)} > \mu_i + \left(\psi^*\right)^{-1}\left( \frac{\alpha \ln t}{T_{i}(t-1)}\right)
\\
\label{eq:UCB3}
& T_i(t-1) < \frac{\alpha \ln n}{\psi^*(\Delta_i/2)}~.
\end{align}
Indeed, assume that the three equations are all false, then we have:
\begin{align*}
\wh{\mu}_{i^*,T_{i^*}(t-1)} + \left(\psi^*\right)^{-1}\left( \frac{\alpha \ln t}{T_{i^*}(t-1)}\right) & > \mu^* \\
& = \mu_i+\Delta_i  \\
& \geq \mu_i + 2 \left(\psi^*\right)^{-1}\left( \frac{\alpha \ln t}{T_{i}(t-1)}\right) \\
& \geq \widehat{\mu}_{i,T_i(t-1)} + \left(\psi^*\right)^{-1}\left( \frac{\alpha \ln t}{T_{i}(t-1)}\right)
\end{align*}
which implies, in particular, that $I_t \neq i$. In other words, letting
\[
    u = \left\lceil \frac{\alpha \ln n}{\psi^*(\Delta_i/2)} \right\rceil
\]
we just proved
\begin{align*}
\E\,T_i(n) = \E \sum_{t=1}^n \ds1_{I_t=i} & \leq u + \E \sum_{t=u+1}^n \ds1_{I_t = i \; \mbox{and \eqref{eq:UCB3} is false}} \\
& \leq u + \E \sum_{t=u+1}^n \ds1_{\mbox{\eqref{eq:UCB1} or \eqref{eq:UCB2} is true}} \\
& = u + \sum_{t=u+1}^n \P\bigl(\mbox{\eqref{eq:UCB1} is true}\bigr) + \P\bigl(\mbox{\eqref{eq:UCB2} is true}\bigr).
\end{align*}
Thus it suffices to bound the probability of the events \eqref{eq:UCB1} and \eqref{eq:UCB2}. Using an union bound and \eqref{eq:psiconcentration} one directly obtains:
\begin{align*}
\P\bigl(\mbox{\eqref{eq:UCB1} is true}\bigr) & \leq \P\left(\exists s \in \{1,\hdots,t\} : \wh{\mu}_{i^*,s} + \left(\psi^*\right)^{-1}\left( \frac{\alpha \ln t}{s}\right)  \leq \mu^* \right) \\
& \leq \sum_{s=1}^t \P\left(\wh{\mu}_{i^*,s} + \left(\psi^*\right)^{-1}\left( \frac{\alpha \ln t}{s}\right)  \leq \mu^* \right) \\
& \leq \sum_{s=1}^t \frac{1}{t^{\alpha}}
= \frac{1}{t^{\alpha - 1}}~. 
\end{align*}
The same upper bound holds for \eqref{eq:UCB2}. Straightforward computations conclude the proof.
\end{proof}

\section{Lower bound}
\label{s:stoch-lower}
We now show that the result of the previous section is essentially unimprovable when the reward distributions are Bernoulli. For $p, q \in [0,1]$ we denote by $\kl(p,q)$ the Kullback-Leibler divergence between a Bernoulli of parameter $p$ and a Bernoulli of parameter $q$, defined as
$$\kl(p,q) =p \ln\frac{p}{q} + (1-p) \ln\frac{1-p}{1-q}~.$$
\begin{theorem}[Distribution-dependent lower bound] \label{th:LR85}
Consider a strategy that satisfies $\E\,T_i(n) = o(n^a)$ for any set of Bernoulli reward distributions, any arm $i$ with $\Delta_i > 0$, and any $a>0$. 
Then, for any set of Bernoulli reward distributions the following holds
$$\liminf_{n \to +\infty} \frac{\oR_n}{\ln n} \geq \sum_{i \,:\, \Delta_i > 0} \frac{\Delta_i}{\kl(\mu_i,\mu^*)}~.$$
\end{theorem}
In order to compare this result with \eqref{eq:regretUCBbounded} we use the following standard inequalities (the left hand side follows from Pinsker's inequality, and the right hand side simply uses $\ln x \leq x -1$),
\begin{equation} \label{eq:klbernoullis}
2 (p-q)^2 \leq \kl(p,q) \leq \frac{(p-q)^2}{q(1-q)}~.
\end{equation}
\begin{proof}
The proof is organized in three steps. For simplicity, we only consider the case of two arms.

\subsection*{First step: Notations.}
Without loss of generality assume that arm $1$ is optimal and arm $2$ is suboptimal, that is $\mu_2 < \mu_1 < 1$. Let $\epsilon > 0$. Since $x \mapsto \kl(\mu_2,x)$ is continuous one can find $\mu_2' \in (\mu_1,1)$ such that
\begin{equation} \label{eq:defmu2prime}
\kl(\mu_2,\mu_2') \leq (1+\epsilon) \kl(\mu_2,\mu_1)~.
\end{equation}
We use the notation $\E', \P'$ when we integrate with respect to the modified bandit where the parameter of arm $2$ is replaced by $\mu_2'$. We want to compare the behavior of the forecaster on the initial and modified bandits. In particular, we prove that with a big enough probability the forecaster can not distinguish between the two problems. Then, using the fact that we have a good forecaster by hypothesis, we know that the algorithm does not make too many mistakes on the modified bandit where arm $2$ is optimal. In other words, we have a lower bound on the number of times the optimal arm is played. This reasoning implies a lower bound on the number of times arm $2$ is played in the initial problem.

We now slightly change the notation for rewards and denote by $X_{2,1}, \hdots, X_{2,n}$ the sequence of random variables obtained when pulling arm $2$ for $n$ times (that is, $X_{2,s}$ is the reward obtained from the $s$-th pull).
For $s \in \{1, \hdots, n\}$, let
$$\widehat{\kl}_s=\sum_{t=1}^{s} \ln \frac{\mu_2 X_{2,t} + (1-\mu_2) (1-X_{2,t})}{\mu_2' X_{2,t} + (1-\mu_2') (1-X_{2,t})}~.$$
Note that, with respect to the initial bandit, $\widehat{\kl}_{T_2(n)}$ is the (non re-normalized) empirical estimate of $\kl(\mu_2, \mu_2')$ at time $n$, since in that case the process $(X_{2,s})$ is i.i.d.\ from a Bernoulli of parameter $\mu_2$. Another important property is the following: for any event $A$ in the $\sigma$-algebra generated by $X_{2,1},\dots,X_{2,n}$ the following change-of-measure identity holds:
\begin{equation} \label{eq:KLchapeauprop}
\P'(A) = \E \left[\ds1_{A} \exp\left(- \widehat{\kl}_{T_2(n)} \right)\right]~.
\end{equation}
In order to link the behavior of the forecaster on the initial and modified bandits we introduce the event
\begin{equation} \label{eq:cndef}
C_n = \left\{ T_2(n) < \frac{1-\epsilon}{\kl(\mu_2,\mu_2')} \ln(n) \;\; \text{and} \;\; \widehat{\kl}_{T_2(n)} \leq \left(1-\frac{\epsilon}{2}\right) \ln(n) \right\}~.
\end{equation}

\subsection*{Second step: $\P(C_n) = o(1)$.}
By \eqref{eq:KLchapeauprop} and \eqref{eq:cndef} one has
$$\P'(C_n) = \E \; \ds1_{C_n} \exp\left(- \widehat{\kl}_{T_2(n)} \right) \geq e^{- (1-\epsilon/2) \ln(n)} \P(C_n)~.$$
Introduce the shorthand
\[
    f_n = \frac{1-\epsilon}{\kl(\mu_2,\mu_2')} \ln(n)~.
\]
Using again~\eqref{eq:cndef} and Markov's inequality, the above implies
\begin{align*}
\P(C_n) \leq n^{(1-\epsilon/2)} \P'(C_n)
&\leq n^{(1-\epsilon/2)} \P'(T_2(n) < f_n)
\\ &\leq n^{(1-\epsilon/2)} \frac{\E'[n - T_2(n)]}{n - f_n}~.
\end{align*}
Now note that in the modified bandit arm $2$ is the unique optimal arm. Hence the assumption that for any bandit, any suboptimal arm $i$, and any $a>0$, the strategy satisfies $\E\,T_i(n) = o(n^a)$, implies that
$$\P(C_n) \leq n^{(1-\epsilon/2)} \frac{\E'[n - T_2(n)]}{n - f_n} = o(1)~.$$ 

\subsection*{Third step: $\P\left(T_2(n) < f_n \right) = o(1)$.}
Observe that
\begin{align}
    \P(C_n)
& \ge
    \P\left(T_2(n) < f_n \;\; \text{and} \;\; \max_{s \le f_n} \widehat{\kl}_s \leq \left(1-\frac{\epsilon}{2}\right) \ln(n) \right)
\notag
\\&=
    \P\biggl(T_2(n) < f_n
\notag
\\&
    \qquad \text{and} \;\; \frac{\kl(\mu_2,\mu_2')}{(1-\epsilon) \ln(n)} \times \max_{s \le f_n} \widehat{\kl}_s \leq \frac{1-\epsilon/2}{1-\epsilon} \kl(\mu_2, \mu_2') \biggr)~.
\label{eq:endLR}
\end{align}
Now we use the maximal version of the strong law of large numbers: for any sequence $\bigl(X_t\bigr)$ of independent real random variables with positive mean $\mu > 0$,
\[
    \lim_{n\to\infty}\frac{1}{n}\sum_{t=1}^n X_t = \mu \quad a.s. \quad\text{implies}\quad
\lim_{n\to\infty}\frac{1}{n}\max_{s=1,\dots,n}\sum_{t=1}^s X_t = \mu \quad a.s.
\]
See, e.g., \cite[Lemma 10.5]{Bub10}.

Since $\kl(\mu_2,\mu_2')>0$ and $\frac{1-\epsilon/2}{1-\epsilon} > 1$, we deduce that
$$\lim_{n \to \infty} \P \left(\frac{\kl(\mu_2,\mu_2')}{(1-\epsilon) \ln(n)} \times \max_{s \le f_n} \widehat{\kl}_s \leq \frac{1-\epsilon/2}{1-\epsilon} \kl(\mu_2, \mu_2') \right) = 1~.$$
Thus, by the result of the second step and~\eqref{eq:endLR}, we get
$$\P\left(T_2(n) < f_n\right) = o(1)~.$$
Now recalling that $f_n = \frac{1-\epsilon}{\kl(\mu_2,\mu_2')} \ln(n)$, and using~\eqref{eq:defmu2prime}, we obtain
$$\E\,T_2(n) \geq (1+o(1)) \frac{1-\epsilon}{1+\epsilon} \frac{\ln(n)}{\kl(\mu_2,\mu_1)}$$
which concludes the proof.
\end{proof}

\section{Refinements and bibliographic remarks}
The UCB strategy presented in Section \ref{sec:UCB} was introduced by \cite{ACF02} for bounded random variables. Theorem~\ref{th:LR85} is extracted from \cite{LR85}. Note that in this last paper the result is more general than ours, which is restricted to Bernoulli distributions. Although \cite{BK97} prove an even more general lower bound, Theorem \ref{th:LR85} and the UCB regret bound provide a reasonably complete solution to the problem. We now discuss some of the possible refinements. In the following, we restrict our attention to the case of bounded rewards (except in Section \ref{sec:heavytail}).

\subsection{Improved constants}
The regret bound proof for UCB can be improved in two ways. First, the union bound over the different time steps can be replaced by a ``peeling'' argument. This allows to show a logarithmic regret for any $\alpha > 1$, whereas the proof of Section~\ref{sec:UCB} requires $\alpha > 2$ ---see \cite[Section 2.2]{Bub10} for more details. A second improvement, proposed by \cite{GC11}, is to use a more subtle set of conditions than \eqref{eq:UCB1}--\eqref{eq:UCB3}. In fact, the authors take both improvements into account, and show that $\alpha$-UCB has a regret of order $\frac{\alpha}{2} \ln n$ for any $\alpha > 1$. In the limit when $\alpha$ tends to $1$, this constant is unimprovable in light of Theorem~\ref{th:LR85} and \eqref{eq:klbernoullis}.

\subsection{Second order bounds}
Although $\alpha$-UCB is essentially optimal, the gap between \eqref{eq:regretUCBbounded} and Theorem \ref{th:LR85} can be important if $\kl(\mu_{i^*}, \mu_i)$ is significantly larger than $\Delta_i^2$. Several improvements have been proposed towards closing this gap. In particular, the UCB-V algorithm of \cite{AMS09} takes into account the variance of the distributions and replaces Hoeffding's inequality by Bernstein's inequality in the derivation of UCB. A previous algorithm with similar ideas was developed by \cite{ACF02} without theoretical guarantees. A second type of approach replaces $L_2$-neighborhoods in $\alpha$-UCB by $\kl$-neighborhoods. This line of work started with \cite{HT10} where only asymptotic guarantees were provided. Later, \cite{GC11} and~\cite{MMS11} (see also \cite{CGMMS12}) independently proposed a similar algorithm, called KL-UCB, which is shown to attain the optimal rate in finite-time. More precisely, \cite{GC11} showed that KL-UCB attains a regret smaller than
\[
    \sum_{i \,:\, \Delta_i > 0} \frac{\Delta_i}{\kl(\mu_i,\mu^*)} \alpha \ln n + \mathcal{O}(1)
\]
where $\alpha >1$ is a parameter of the algorithm. Thus, KL-UCB is optimal for Bernoulli distributions, and strictly dominates $\alpha$-UCB for any bounded reward distributions.

\subsection{Distribution-free bounds}
In the limit when $\Delta_i$ tends to $0$, the upper bound in \eqref{eq:regretUCBbounded} becomes vacuous. On the other hand, it is clear that the regret incurred from pulling arm $i$ cannot be larger than $n\,\Delta_i$. Using this idea, it is easy to show that the regret of $\alpha$-UCB is always smaller than $\sqrt{\alpha n K \ln n}$ (up to a numerical constant). However, as we shall see in the next chapter, one can show a minimax lower bound on the regret of order $\sqrt{n K}$. \cite{AB09} proposed a modification of $\alpha$-UCB that gets rid of the extraneous logarithmic term in the upper bound. More precisely, let $\Delta = \min_{i \neq i^*} \Delta_i$, \cite{AB10} show that MOSS (Minimax Optimal Strategy in the Stochastic case) attains a regret smaller than 
$$\min \left\{ \sqrt{n K},\, \frac{K}{\Delta} \ln\frac{n \Delta^2}{K} \right\}$$
up to a numerical constant. The weakness of this result is that the second term in the above equation only depends on the smallest gap $\Delta$. In \cite{AO10} (see also \cite{PR11}) the authors designed a strategy, called improved UCB, with a regret of order
$$\sum_{i \,:\, \Delta_i > 0} \frac{1}{\Delta_i} \ln \bigl(n \Delta_i^2\bigr)~.$$
This latter regret bound can be better than the one for MOSS in some regimes, but it does not attain the minimax optimal rate of order $\sqrt{n K}$. It is an open problem to obtain a strategy with a regret always better than those of MOSS and improved UCB. A plausible conjecture is that a regret of order
\[
    \sum_{i \,:\, \Delta_i > 0} \frac{1}{\Delta_i} \ln \frac{n}{H}
\qquad\text{with}\qquad
    H = \sum_{i \,:\, \Delta_i > 0} \frac{1}{\Delta_i^2} 
\]
is attainable. Note that the quantity $H$ appears in other variants of the stochastic multi-armed bandit problem, see \cite{ABM10}.

\subsection{High probability bounds}
While bounds on the pseudo-regret $\oR_n$ are important, one typically wants to control the quantity $\hat{R}_n = n \mu^* - \sum_{t=1}^n \mu_{I_t}$ with high probability. Showing that $\hat{R}_n$ is close to its expectation $\oR_n$ is a challenging task, since naively one might expect the fluctuations of $\hat{R}_n$ to be of order $\sqrt{n}$, which would dominate the expectation $\oR_n$ which is only of order $\ln n$. The concentration properties of $\hat{R}_n$ for UCB are analyzed in detail in \cite{AMS09}, where it is shown that $\hat{R}_n$ concentrates around its expectation, but that there is also a polynomial (in $n$) probability that $\hat{R}_n$ is of order $n$. In fact the polynomial concentration of $\hat{R}_n$ around $\oR_n$ can be directly derived from our proof of Theorem \ref{th:ucb}. In \cite{SA11} it is proved that for anytime strategies (i.e., strategies that do not use the time horizon $n$) it is basically impossible to improve this polynomial concentration to a classical exponential concentration.
In particular this shows that, as far as high probability bounds are concerned, anytime strategies are surprisingly weaker than strategies using the time horizon information (for which exponential concentration of $\hat{R}_n$ around $\ln n$ are possible, see \cite{AMS09}). 


\subsection{$\epsilon$-greedy}
A simple and popular heuristic for bandit problems is the $\ve$-greedy strategy ---see, e.g., \cite{SB98}. The idea is very simple. First, pick a parameter $0 < \ve < 1$. Then, at each step greedily play the arm with highest empirical mean reward with probability $1-\ve$, and play a random arm with probability $\ve$. \cite{ACF02} proved that, if $\ve$ is allowed to be a certain function $\ve_t$ of the current time step $t$, namely $\ve_t = K/(d^2 t)$, then the regret grows logarithmically like $(K \ln n)/d^2$, provided $0 < d < \min_{i \neq i^*} \Delta_i$. While this bound has a suboptimal dependence on $d$, \cite{ACF02} show that this algorithm performs well in practice, but the performance degrades quickly if $d$ is not chosen as a tight lower bound of $\min_{i \neq i^*} \Delta_i$.

\subsection{Thompson sampling}
In the very first paper on the multi-armed bandit problem, \cite{Tho33}, a simple strategy was proposed for the case of Bernoulli distributions. The so-called Thompson sampling algorithm proceeds as follows. Assume a uniform prior on the parameters $\mu_i \in [0,1]$, let $\pi_{i,t}$ be the posterior distribution for $\mu_i$ at the $t^{th}$ round, and let $\theta_{i,t} \sim \pi_{i,t}$ (independently from the past given $\pi_{i,t}$). The strategy is then given by $I_t \in \argmax_{i =1,\hdots, K} \theta_{i,t}$. Recently there has been a surge of interest for this simple policy, mainly because of its flexibility to incorporate prior knowledge on the arms, see for example \cite{CLi11} and the references therein. While the theoretical behavior of Thompson sampling has remained elusive for a long time, we have now a fairly good understanding of its theoretical properties: in \cite{AG12} the first logarithmic regret bound was proved, and in \cite{KKM12} it was showed that in fact Thompson sampling attains essentially the same regret than in \eqref{eq:regretUCBbounded}. Interestingly note that while Thompson sampling comes from a Bayesian reasoning, it is analyzed with a frequentist perspective. For more on the interplay between Bayesian strategy and frequentist regret analysis we refer the reader to \cite{KCG12}.

\subsection{Heavy-tailed distributions} \label{sec:heavytail}
We showed in Section \ref{sec:UCB} how to obtain a UCB-type strategy through a bound on the moment generating function. Moreover one can see that the resulting bound in Theorem \ref{th:ucb} deteriorates as the tail of the distributions become heavier. In particular, we did not provide any result for the case of distributions for which the moment generating function is not finite. Surprisingly, it was shown in \cite{BCL12} that in fact there exists strategy with essentially the same regret than in \eqref{eq:regretUCBbounded}, as soon as the {\em variance} of the distributions are finite. More precisely, using more refined robust estimators of the mean than the basic empirical mean, one can construct a UCB-type strategy such that for distributions with moment of order $1+\epsilon$ bounded by $1$ it satisfies
\[
R_n \leq \sum_{i \,:\, \Delta_i > 0} \left(8 \left(\frac{4}{\Delta_i}\right)^{\frac{1}{\epsilon}} \ln n + 5\Delta_i\right)~.
\]
We refer the interested reader to \cite{BCL12} for more details on these 'robust' versions of UCB.

%% file: adversarial.tex
In this chapter we consider the important variant of the multi-armed bandit problem where no stochastic assumption is made on the generation of rewards. Denote by $g_{i,t}$ the reward (or gain) of arm $i$ at time step $t$. We assume all rewards are bounded, say $g_{i,t}\in [0,1]$. At each time step $t=1, 2,\hdots$, simultaneously with the player's choice of the arm $I_t \in \{1,\hdots,K\}$, an adversary assigns to each arm $i=1,\dots,K$ the reward $g_{i,t}$. Similarly to the stochastic setting, we measure the performance of the player compared to the performance of the best arm through the regret
$$R_n = \max_{i=1,\hdots,K} \sum_{t=1}^n g_{i,t} - \sum_{t=1}^n g_{I_t,t}~.$$
Sometimes we consider losses rather than gains. In this case we denote by $\ell_{i,t}$ the loss of arm $i$ at time step $t$, and the regret rewrites as
$$R_n = \sum_{t=1}^n \ell_{I_t,t} - \min_{i=1,\hdots,K} \sum_{t=1}^n \ell_{i,t}~.$$
The loss and gain versions are symmetric, in the sense that one can translate the analysis from one to the other setting via the equivalence $\ell_{i,t} = 1 - g_{i,t}$. In the following we emphasize the loss version, but we revert to the gain version whenever it makes proofs simpler.

The main goal is to achieve sublinear (in the number of rounds) bounds on the regret uniformly over all possible adversarial assignments of gains to arms. At first sight, this goal might seem hopeless. Indeed, for any deterministic forecaster there exists a sequence of losses $(\ell_{i,t})$ such that $R_n \geq n/2$. Concretely, it suffices to consider the following sequence of losses:
$$
\begin{array}{c}
\text{if} \; I_t = 1, \; \text{then} \; \ell_{2,t} = 0 \; \text{and} \; \ell_{i,t}=1 \; \text{for all} \; i \neq 2;
\\
\text{if} \; I_t \neq 1, \; \text{then} \; \ell_{1,t} = 0 \; \text{and} \; \ell_{i,t}=1 \; \text{for all} \; i \neq 1.
\end{array}
$$
The key idea to get around this difficulty is to add randomization to the selection of the action $I_t$ to play. By doing so, the forecaster can ``surprise'' the adversary, and this surprise effect suffices to get a regret essentially as low as the minimax regret for the stochastic model. Since the regret $R_n$ then becomes a random variable, the goal is thus to obtain bounds in high probability or in expectation on $R_n$ (with respect to both eventual randomization of the forecaster and of the adversary). This task is fairly difficult, and a convenient first step is to bound the pseudo-regret
\begin{equation} \label{eq:pseudoregretadv}
\oR_n = \E \sum_{t=1}^n \ell_{I_t,t} - \min_{i=1,\hdots,K} \E \sum_{t=1}^n \ell_{i,t}~.
\end{equation}
Clearly $\oR_n \leq \E\,R_n$, and thus an upper bound on the pseudo-regret does not imply a bound on the expected regret. As argued in the Introduction, the pseudo-regret has no natural interpretation unless the adversary is oblivious. In that case, the pseudo-regret coincides with the standard regret, which is always the ultimate quantity of interest.

\section{Pseudo-regret bounds}
As we pointed out, in order to obtain non-trivial regret guarantees in the adversarial framework it is necessary to consider randomized forecasters. Below we describe the randomized forecaster Exp3, which is based on two fundamental ideas.
%
\begin{center}
\bookbox{
{\em Exp3 (Exponential weights for Exploration and Exploitation)}
\\
{Parameter:} a non-increasing sequence of real numbers $(\eta_t)_{t \in \N}$.

\smallskip\noindent
Let $p_1$ be the uniform distribution over $\{1,\hdots,K\}$.

\smallskip\noindent
For each round $t=1,2,\ldots,n$
\begin{itemize}
\item[(1)]
Draw an arm $I_t$ from the probability distribution $p_t$.
\item[(2)]
For each arm $i=1,\dots,K$ compute the estimated loss $\tilde{\ell}_{i,t} = \frac{\ell_{i,t}}{p_{i,t}} \ds1_{I_t = i}$ and update the estimated cumulative loss
$\tilde{L}_{i,t} = \tilde{L}_{i,t-1} + \tilde{\ell}_{i,s}$.
\item[(3)]
Compute the new probability distribution over arms $p_{t+1}=\bigl(p_{1,t+1},\hdots,p_{K,t+1}\bigr)$, where
$$p_{i,t+1} = \frac{\exp{\left(- \eta_t \tilde{L}_{i,t}\right)}}{\sum_{k=1}^K \exp{\left(- \eta_t \tilde{L}_{k,t}\right)}}~.$$
\end{itemize}
}
\end{center}
First, despite the fact that only the loss of the played arm is observed, with a simple trick it is still possible to build an unbiased estimator for the loss of any other arm. Namely, if the next arm $I_t$ to be played is drawn from a probability distribution $p_t=\bigl(p_{1,t},\hdots,p_{K,t}\bigr)$, then
\[
    \tilde{\ell}_{i,t} = \frac{\ell_{i,t}}{p_{i,t}} \ds1_{I_t=i}
\]
is an unbiased estimator (with respect to the draw of $I_t$) of $\ell_{i,t}$. Indeed, for each $i=1,\dots,K$ we have
\[
    \E_{I_t \sim p_t} \bigl[\tilde{\ell}_{i,t}\bigr] = \sum_{j=1}^K  p_{j,t}\frac{\ell_{i,t}}{p_{i,t}} \ds1_{j=i}  = \ell_{i,t}~.
\]
The second idea is to use an exponential reweighting of the cumulative estimated losses to define the probability distribution $p_t$ from which the forecaster will select the arm $I_t$. Exponential weighting schemes are a standard tool in the study of sequential prediction schemes under adversarial assumptions. The reader is referred to the monograph by \cite{CL06} for a general introduction to prediction of individual sequences, and to the recent survey by \cite{arora2012multiplicative} focussed on computer science applications of exponential weighting.

We provide two different pseudo-regret bounds for this strategy. The bound \eqref{eq:exp3boundanytime} is obtained assuming that the forecaster does not know the number of rounds $n$. This is the anytime version of the algorithm. The bound \eqref{eq:exp3bound}, instead, shows that a better constant can be achieved using the knowledge of the time horizon.
\begin{theorem}[Pseudo-regret of Exp3] \label{th:Exp3}
If Exp3 is run with $\eta_t=\eta = \sqrt{\frac{2 \ln K}{n K}}$, then
\begin{equation} \label{eq:exp3bound}
\oR_n \leq \sqrt{2 n K \ln K}~.
\end{equation}
Moeover, if Exp3 is run with $\eta_t = \sqrt{\frac{\ln K}{t K}}$, then
\begin{equation} \label{eq:exp3boundanytime}
\oR_n \leq 2 \sqrt{n K \ln K}~.
\end{equation}
\end{theorem}
\begin{proof}
We prove that for any non-increasing sequence $(\eta_t)_{t \in \N}$ Exp3 satisfies
\begin{equation} \label{eq:exp3lem}
\oR_n \leq \frac{K}{2} \sum_{t=1}^n \eta_t + \frac{\ln K}{\eta_n}~.
\end{equation}
Inequality \eqref{eq:exp3bound} then trivially follows from \eqref{eq:exp3lem}. For \eqref{eq:exp3boundanytime} we use \eqref{eq:exp3lem} and
$\sum_{t=1}^n \frac{1}{\sqrt{t}} \leq \int_{0}^n \frac{1}{\sqrt{t}} dt = 2 \sqrt{n}$. The proof of \eqref{eq:exp3lem} in divided in five steps.

\subsection*{First step: Useful equalities.}
The following equalities can be easily verified:
\begin{equation} \label{eq:equalities}
\E_{i \sim p_t} \tilde{\ell}_{i,t} = \ell_{I_t,t}, \quad \E_{I_t \sim p_t} \tilde{\ell}_{i,t} = \ell_{i,t}, \quad
\E_{i \sim p_t} {\tilde{\ell}}^2_{i,t} = \frac{\ell_{I_t,t}^2}{p_{I_t,t}}, \quad \E_{I_t \sim p_t} \frac{1}{p_{I_t,t}} = K~.
\end{equation}
In particular, they imply
\begin{equation} \label{eq:expregret}
\sum_{t=1}^n \ell_{I_t,t} - \sum_{t=1}^n \ell_{k,t}  = \sum_{t=1}^n \E_{i \sim p_t} \tilde{\ell}_{i,t} - \sum_{t=1}^n \E_{I_t \sim p_t} \tilde{\ell}_{k,t}~.
\end{equation}
The key idea of the proof is rewrite $\E_{i \sim p_t} \tilde{\ell}_{i,t}$ as follows 
\begin{align} 
 \E_{i \sim p_t} \tilde{\ell}_{i,t} = & \frac{1}{\eta_t} \ln \E_{i \sim p_t} \exp{\left(- \eta_t \bigl(\tilde{\ell}_{i,t}  - \E_{k \sim p_t} \tilde{\ell}_{k,t}\bigr) \right)}  \notag \\
&- \frac{1}{\eta_t} \ln \E_{i \sim p_t} \exp{\left(- \eta_t \tilde{\ell}_{i,t} \right)}~. \label{eq:logmoment}
\end{align}
The reader may recognize $\ln \E_{i \sim p_t} \exp{\bigl(- \eta_t \tilde{\ell}_{i,t} \bigr)}$ as the cumulant-generating function (or the log of the moment-generating function) of the random variable $\tilde{\ell}_{I_t,t}$. This quantity naturally arises in the analysis of forecasters based on exponential weights. In the next two steps we study the two terms in the right-hand side of \eqref{eq:logmoment}.

\subsection*{Second step: Study of the first term in \eqref{eq:logmoment}.}
We use the inequalities $\ln x \leq x-1$ and $\exp(-x) - 1 + x \leq x^2/2$, for all $x \geq 0$, to obtain:
\begin{align}
\ln \E_{i \sim p_t} &\exp{\left(- \eta_t (\tilde{\ell}_{i,t} - \E_{k \sim p_t} \tilde{\ell}_{k,t}) \right)} \notag \\
& =  \ln \E_{i \sim p_t} \exp{\left(- \eta_t \tilde{\ell}_{i,t}\right)} + \eta_t \E_{k \sim p_t} \tilde{\ell}_{k,t} \notag \\
&\leq \E_{i \sim p_t} \left(\exp{\left(- \eta_t \tilde{\ell}_{i,t}\right)} - 1 + \eta_t \tilde{\ell}_{i,t} \right) \notag \\
&\leq \E_{i \sim p_t} \frac{\eta_t^2 \tilde{\ell}_{i,t}^2}{2} \notag \\
&\leq \frac{\eta_t^2}{2 p_{I_t,t}} \label{eq:secondstep}
\end{align}
where the last step comes from the third equality in \eqref{eq:equalities}.

\subsection*{Third step: Study of the second term in \eqref{eq:logmoment}.}
Let $\tilde{L}_{i,0}=0$, $\Phi_0(\eta)=0$ and $\Phi_t(\eta) = \frac{1}{\eta} \ln \frac{1}{K} \sum_{i=1}^K \exp{\left(- \eta \tilde{L}_{i,t}\right)}$. Then, by definition of $p_t$ we have
\begin{align}
- \frac{1}{\eta_t} \ln \E_{i \sim p_t} \exp{\left(- \eta_t \tilde{\ell}_{i,t} \right)} 
& = - \frac{1}{\eta_t} \ln \frac{\sum_{i=1}^K \exp{\left(- \eta_t \tilde{L}_{i,t}\right)}}{\sum_{i=1}^K \exp{\left(- \eta_t \tilde{L}_{i,t-1}\right)}} \notag \\
& = \Phi_{t-1}(\eta_t) - \Phi_{t}(\eta_t)~. \label{eq:thirdstep}
\end{align}

\subsection*{Fourth step: Summing.}
Putting together \eqref{eq:expregret}, \eqref{eq:logmoment}, \eqref{eq:secondstep} and \eqref{eq:thirdstep} we obtain
$$\sum_{t=1}^n g_{k,t} - \sum_{t=1}^n g_{I_t,t} \leq \sum_{t=1}^n \frac{\eta_t}{2 p_{I_t,t}} + \sum_{t=1}^n \Phi_{t-1}(\eta_t) - \Phi_t(\eta_t) - \sum_{t=1}^n \E_{I_t \sim p_t} \tilde{\ell}_{k,t}~.$$
The first term is easy to bound in expectation since, by the rule of conditional expectations and the last equality in \eqref{eq:equalities} we have
$$\E \sum_{t=1}^n \frac{\eta_t}{2 p_{I_t,t}} = \E \sum_{t=1}^n \E_{I_t \sim p_t} \frac{\eta_t}{2 p_{I_t,t}} = \frac{K}{2} \sum_{t=1}^n \eta_t~.$$
For the second term we start with an Abel transformation,
$$\sum_{t=1}^n \bigl(\Phi_{t-1}(\eta_t) - \Phi_t(\eta_t)\bigr) = \sum_{t=1}^{n-1} \bigl(\Phi_t(\eta_{t+1}) - \Phi_t(\eta_t)\bigr) - \Phi_n(\eta_n)$$
since $\Phi_0(\eta_1)=0$. Note that
\begin{align*}
- \Phi_n(\eta_n) & = \frac{\ln K}{\eta_n} - \frac{1}{\eta_n} \ln\left(\sum_{i=1}^K \exp\left(-\eta_n \tilde{L}_{i,n}\right)\right) \\ & \leq \frac{\ln K}{\eta_n} - \frac{1}{\eta_n} \ln\left(\exp\left(-\eta_n \tilde{L}_{k,n}\right)\right)\\
& =  \frac{\ln K}{\eta_n} + \sum_{t=1}^n \tilde{\ell}_{k,t}
\end{align*}
and thus we have
$$\E \left[ \sum_{t=1}^n g_{k,t} - \sum_{t=1}^n g_{I_t,t} \right] \leq \frac{K}{2} \sum_{t=1}^n \eta_t + \frac{\ln K}{\eta_n} + \E \sum_{t=1}^{n-1} \Phi_t(\eta_{t+1}) - \Phi_t(\eta_t)~.$$
To conclude the proof, we show that $\Phi_t'(\eta)\geq 0$. Since $\eta_{t+1}\leq \eta_t$, we then obtain $\Phi_t(\eta_{t+1}) - \Phi_t(\eta_t) \leq 0$. 
Let
\[
    p_{i,t}^{\eta}=\frac{\exp{\left(- \eta \tilde{L}_{i,t}\right)}}{\sum_{k=1}^K \exp{\left(- \eta \tilde{L}_{k,t}\right)}}~.
\]
Then
\begin{align*}
\Phi_t'(\eta) & = - \frac{1}{\eta^2} \ln\left(\frac{1}{K} \sum_{i=1}^K \exp{\left(- \eta \tilde{L}_{i,t}\right)}\right) - \frac{1}{\eta}\frac{\sum_{i=1}^K \tilde{L}_{i,t} \exp{\left(- \eta \tilde{L}_{i,t}\right)}}{\sum_{i=1}^K \exp{\left(- \eta \tilde{L}_{i,t}\right)}} \\*
& = \frac{1}{\eta^2} \frac{1}{\sum_{i=1}^K \exp{\left(- \eta \tilde{L}_{i,t}\right)}} \sum_{i=1}^K \exp{\left(- \eta \tilde{L}_{i,t}\right)} \\*
& \qquad \times \left(-\eta \tilde{L}_{i,t} - \ln\left(\frac{1}{K} \sum_{i=1}^K \exp{\left(- \eta \tilde{L}_{i,t}\right)}\right) \right)~.
\end{align*}
Simplifying, we get (since $p_1$ is the uniform distribution over $\{1,\hdots,K\}$),
\[
\Phi_t'(\eta) = \frac{1}{\eta^2} \sum_{i=1}^K p_{i,t}^{\eta} \ln(K p_{i,t}^{\eta}) = \frac{1}{\eta^2} \K(p_t^{\eta}, p_1) \geq 0~.
\]
\end{proof}

\section{High probability and expected regret bounds} \label{sec:regret}
In this section we prove a high probability bound on the regret. Unfortunately, the Exp3 strategy defined in the previous section is not adequate for this task. Indeed, the variance of the estimate $\tilde{\ell}_{i,t}$ is of order $1/p_{i,t}$, which can be arbitrarily large. In order to ensure that the probabilities $p_{i,t}$ are bounded from below, the original version of Exp3 mixes the exponential weights with a uniform distribution over the arms. In order to avoid increasing the regret, the mixing coefficient $\gamma$ associated with the uniform distribution cannot be larger than $n^{-1/2}$. Since this implies that the variance of the cumulative loss estimate $\tilde{L}_{i,n}$ can be of order $n^{3/2}$, very little can be said about the concentration of the regret also for this variant of Exp3.

This issue can be solved by combining the mixing idea with a different estimate for losses. In fact, the core idea is more transparent when expressed in terms of gains, and so we turn to the gain version of the problem. The trick is to introduce a bias in the gain estimate which allows to derive a high probability statement on this estimate.
\begin{lemma} \label{lem:exp3P}
For $\beta \le 1$, let
\[
    \tilde{g}_{i,t} = \frac{g_{i,t} \ds1_{I_t=i} + \beta}{p_{i,t}}~.
\]
Then, with probability at least $1-\delta$,
$$\sum_{t=1}^n g_{i,t} \leq \sum_{t=1}^n \tilde{g}_{i,t}  + \frac{\ln(\delta^{-1})}{\beta}~.$$
\end{lemma}
\begin{proof}
Let $\E_t$ be the expectation conditioned on $I_1,\dots,I_{t-1}$. Since $\exp(x)\le 1+x+x^2$ for $x\leq 1$, for $\beta\leq 1$ we have
\begin{align*}
    \E_t \exp&\bigg(\beta g_{i,t} - \beta \frac{g_{i,t} \ds1_{I_t=i} + \beta}{p_{i,t}}\bigg)
\\ &\leq
    \Bigg(1 + \E_t \bigg[\beta g_{i,t} - \beta \frac{g_{i,t} \ds1_{I_t=i}}{p_{i,t}}\bigg] + \E_t \bigg[\beta g_{i,t} - \beta \frac{g_{i,t} \ds1_{I_t=i}}{p_{i,t}}\bigg]^2 \Bigg)
\\ &\quad
    \times\exp\bigg(\!- \frac{\beta^2}{p_{i,t}}\bigg)
\\ &\leq
    \Bigg(1 + \beta^2 \frac{g^2_{i,t}}{p_{i,t}} \Bigg) \exp\bigg(\!- \frac{\beta^2}{p_{i,t}}\bigg)
\\ &\leq
     1
\end{align*}
where the last inequality uses $1+u \leq \exp(u)$. As a consequence, we have
  \begin{align*}
  \E \exp\bigg(\beta \sum_{t=1}^n g_{i,t} - \beta \sum_{t=1}^n \frac{g_{i,t} \ds1_{I_t=i} + \beta}{p_{i,t}}\bigg) \leq 1.
  \end{align*}
Moreover, Markov's inequality implies $\P\left(X>\ln(\delta^{-1})\right) \leq \delta \E e^X$ and thus, with probability at least $1-\delta$,
$$\beta \sum_{t=1}^n g_{i,t} - \beta \sum_{t=1}^n \frac{g_{i,t} \ds1_{I_t=i} + \beta}{p_{i,t}} \leq \ln(\delta^{-1})~.$$
\end{proof}
\begin{figure}[t]
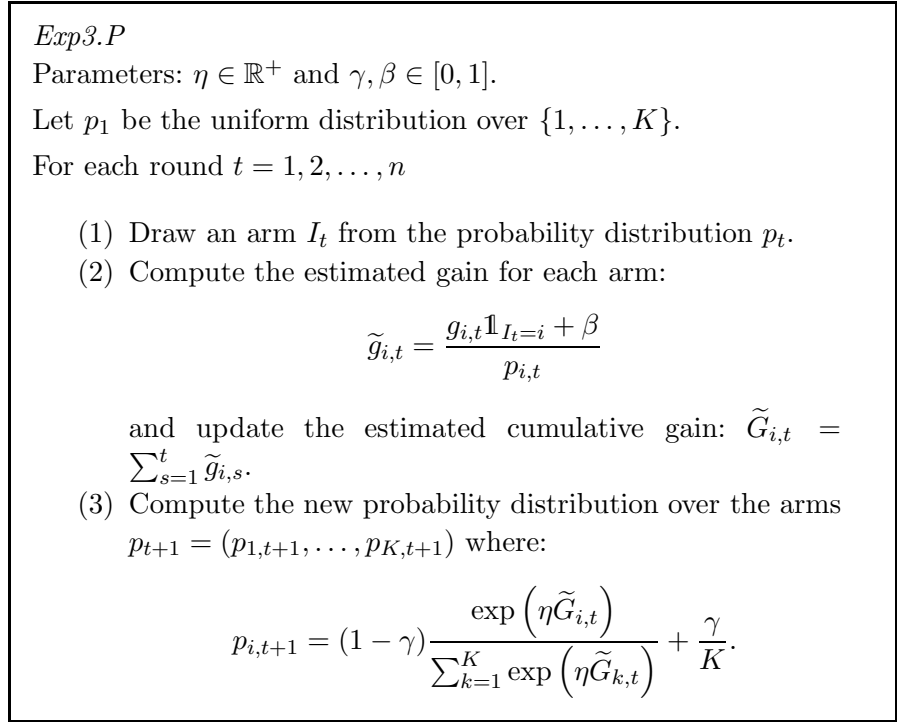

\bookbox{
{\em Exp3.P}
\\
{Parameters:} $\eta \in \R^+$ and $\gamma, \beta \in [0,1]$.

\smallskip\noindent
Let $p_1$ be the uniform distribution over $\{1,\hdots,K\}$.

\smallskip\noindent
For each round $t=1,2,\ldots,n$
\begin{itemize}
\item[(1)]
Draw an arm $I_t$ from the probability distribution $p_t$.
\item[(2)]
Compute the estimated gain for each arm:
\[
    \tilde{g}_{i,t} = \frac{g_{i,t} \ds1_{I_t = i} + \beta}{p_{i,t}}
\]
and update the estimated cumulative gain: $\tilde{G}_{i,t} = \sum_{s=1}^t \tilde{g}_{i,s}$.
\item[(3)]
Compute the new probability distribution over the arms $p_{t+1}=(p_{1,t+1},\hdots,p_{K,t+1})$ where:
$$p_{i,t+1} = (1-\gamma) \frac{\exp{\left(\eta \tilde{G}_{i,t}\right)}}{\sum_{k=1}^K \exp{\left(\eta \tilde{G}_{k,t}\right)}} + \frac{\gamma}{K}.$$
\end{itemize}
}
\caption{Exp3.P forecaster.}
\label{fig:exp3P}
\end{figure}

The strategy associated with these new estimates, called Exp3.P, is described in Figure \ref{fig:exp3P}. Note that, for the sake of simplicity, the strategy is described in the setting with known time horizon ($\eta$ is constant). Anytime results can easily be derived with the same techniques as in the proof of Theorem~\ref{th:Exp3}. 

In the next theorem we propose two different high probability bounds. In \eqref{eq:exp3Pbound1} the algorithm needs the confidence level $\delta$ as an input parameter. In~\eqref{eq:exp3Pbound2} the algorithm satisfies a high probability bound for any confidence level. This latter property is particularly important to derive good bounds on the expected regret.
\begin{theorem}[High probability bound for Exp3.P] \label{th:exp3P}
For any given $\delta \in (0,1)$, if Exp3.P is run with
\[
    \beta = \sqrt{\frac{\ln(K \delta^{-1})}{n K}}, \quad \eta = 0.95 \sqrt{\frac{\ln(K)}{n K}}, \quad \gamma = 1.05 \sqrt{\frac{K \ln(K)}{n}}
\]
then, with probability at least $1-\delta$,
\begin{equation} \label{eq:exp3Pbound1}
R_n \leq 5.15 \sqrt{n K \ln(K \delta^{-1})}~.
\end{equation}
Moreover, if Exp3.P is run with
$
    \beta = \sqrt{\tfrac{\ln(K)}{n K}} 
$
while $\eta$ and $\gamma$ are chosen as before, then, with probability at least $1-\delta$,
\begin{equation} \label{eq:exp3Pbound2}
R_n \leq \sqrt{\frac{n K}{\ln(K)}} \ln(\delta^{-1}) + 5.15 \sqrt{n K \ln(K)}~.
\end{equation}
\end{theorem}
\begin{proof}
We first prove (in three steps) that if $\gamma \leq 1/2$ and $(1+\beta) K \eta \leq \gamma$, then Exp3.P satisfies, with probability at least $1-\delta$,
\begin{equation} \label{eq:exp3Plem}
R_n \leq \beta n K + \gamma n + (1+\beta) \eta K n + \frac{\ln(K \delta^{-1})}{\beta} + \frac{\ln K}{\eta}~.
\end{equation}

\subsection*{First step: Notation and simple equalities.}
One can immediately see that $\E_{i \sim p_t} \tilde{g}_{i,t} = g_{I_t,t} + \beta K$, and thus
\begin{equation} \label{eq:exp3P1}
\sum_{t=1}^n g_{k,t} - \sum_{t=1}^n g_{I_t,t} = \beta n K  + \sum_{t=1}^n g_{k,t} - \sum_{t=1}^n \E_{i \sim p_t} \tilde{g}_{i,t}~.
\end{equation}
The key step is, again, to consider the cumulant-generating function of $\tilde{g}_{i,t}$. However, because of the mixing, we need to introduce a few more notations. Let $u=\bigl(\tfrac{1}{K},\hdots,\tfrac{1}{K}\bigr)$ be the uniform distribution over the arms, let and $w_t = \tfrac{p_t - u}{1-\gamma}$ be the distribution induced by Exp3.P at time $t$ without the mixing. Then we have:
\begin{align} 
- \E_{i \sim p_t} \tilde{g}_{i,t} & = - (1-\gamma) \E_{i \sim w_t} \tilde{g}_{i,t} - \gamma \E_{i \sim u} \tilde{g}_{i,t} \notag \\
& = (1-\gamma) \biggl( \frac{1}{\eta} \ln \E_{i \sim w_t} \exp \bigl( \eta (\tilde{g}_{i,t} - \E_{k \sim w_t} \tilde{g}_{k,t}) \bigr) \notag \\
& \quad - \frac{1}{\eta} \ln \E_{i \sim w_t} \exp\left( \eta \tilde{g}_{i,t} \right) \biggr) - \gamma \E_{i \sim u} \tilde{g}_{i,t}~. \label{eq:exp3P2} 
\end{align}

\subsection*{Second step: Study of the first term in \eqref{eq:exp3P2}.}
We use the inequalities $\ln x \leq x-1$ and $\exp(x) \leq 1 + x + x^2$, for all $x \leq 1$, as well as the fact that $\eta \tilde{g}_{i,t} \leq 1$ since $(1+\beta) \eta K \leq \gamma$:
\begin{align}
\ln \E_{i \sim w_t} \exp{\Bigl( \eta \bigl(\tilde{g}_{i,t} - \E_{k \sim p_t} \tilde{g}_{k,t}\bigr) \Bigr)} \notag
& = \ln \E_{i \sim w_t} \exp{\bigl( \eta \tilde{g}_{i,t}\bigr)} - \eta\,\E_{k \sim p_t} \tilde{g}_{k,t} \notag \\
&\leq \E_{i \sim w_t} \Bigl[\exp{\bigl( \eta \tilde{g}_{i,t}\bigr)} - 1 - \eta \tilde{g}_{i,t} \Bigr] \notag \\
&\leq \E_{i \sim w_t} \eta^2 \tilde{g}_{i,t}^2  \notag \\
&\leq \frac{1+\beta}{1-\gamma} \eta^2 \sum_{i=1}^K \tilde{g}_{i,t} \label{eq:exp3P3}
\end{align}
where we used $\frac{w_{i,t}}{p_{i,t}} \leq \frac{1}{1-\gamma}$ in the last step.

\subsection*{Third step: Summing.}
Set $\tilde{G}_{i,0} = 0$. Recall that $w_t=\bigl(w_{1,t},\hdots,w_{K,t}\bigr)$ with 
\begin{equation} \label{eq:exp3P4}
w_{i,t} = \frac{\exp{\left(- \eta \tilde{G}_{i,t-1}\right)}}{\sum_{k=1}^K \exp{\left(- \eta \tilde{G}_{k,t-1}\right)}}~.
\end{equation}
Then substituting \eqref{eq:exp3P3} in \eqref{eq:exp3P2} and summing using \eqref{eq:exp3P4}, we obtain
\begin{align*}
- \sum_{t=1}^n &\E_{i \sim p_t} \tilde{g}_{i,t} \\
& \leq (1+\beta) \eta \sum_{t=1}^n \sum_{i=1}^K \tilde{g}_{i,t} - \frac{1-\gamma}{\eta} \sum_{t=1}^n \ln \left( \sum_{i=1}^K w_{i,t} \exp{\left( \eta \tilde{g}_{i,t} \right)}  \right) \\
& = (1+\beta) \eta \sum_{t=1}^n \sum_{i=1}^K \tilde{g}_{i,t} - \frac{1-\gamma}{\eta} \ln \left( \prod_{t=1}^n \frac{\sum_{i=1}^K \exp(\eta \tilde{G}_{i,t})}{\sum_{i=1}^K \exp(\eta \tilde{G}_{i,t-1})} \right) \\
& \leq (1+\beta) \eta K \max_j \tilde{G}_{j,n} + \frac{\ln K}{\eta} - \frac{1-\gamma}{\eta} \ln \left( \sum_{t=1}^n \exp(\eta \tilde{G}_{i,n}) \right) \\
& \leq - \bigl(1 - \gamma - (1+\beta) \eta K\bigr) \max_j \tilde{G}_{j,n} + \frac{\ln(K)}{\eta} \\
& \leq - \bigl(1 - \gamma - (1+\beta) \eta K\bigr) \max_j \sum_{t=1}^n g_{j,t} + \frac{\ln(K \delta^{-1})}{\beta} + \frac{\ln(K)}{\eta}~.
\end{align*}
The last inequality comes from Lemma \ref{lem:exp3P}, the union bound, and $\gamma - (1+\beta) \eta K \leq 1$ which is a consequence of $(1+\beta) \eta K \leq \gamma \leq 1/2$. Combining this last inequality with \eqref{eq:exp3P1} we obtain
$$R_n \leq \beta n K + \gamma n + (1+\beta) \eta K n + \frac{\ln\bigl(K \delta^{-1}\bigr)}{\beta} + \frac{\ln(K)}{\eta}$$
which is the desired result.

Inequality~\eqref{eq:exp3Pbound1} is then proved as follows. First, it is trivial if $n \geq 5.15 \sqrt{n K \ln(K \delta^{-1})}$ and thus we can assume that this is not the case. This implies that $\gamma \leq 0.21$ and $\beta \leq 0.1$, and thus we have $(1+\beta) \eta K \leq \gamma \leq 1/2$. Using \eqref{eq:exp3Plem} directly yields the claimed bound. The same argument can be used to derive \eqref{eq:exp3Pbound2}.
\end{proof}
We now discuss expected regret bounds. As the cautious reader may already have observed, if the adversary is oblivious, namely when $\bigl(\ell_{1,t}, \hdots, \ell_{K,t}\bigr)$ is independent of $I_1, \hdots, I_{t-1}$ for each $t$, a pseudo-regret bound implies the same bound on the expected regret. This follows from noting that the expected regret against an oblivious adversary is smaller than the maximal pseudo-regret against deterministic adversaries, see \cite[Proposition 33]{AB10} for a proof of this fact. In the general case of a non-oblivious adversary, the loss vector $\bigl(\ell_{1,t}, \hdots, \ell_{K,t}\bigr)$ at time $t$ depends on the past actions of the forecaster. This makes the analysis of the expected regret more intricate. One way around this difficulty is to first prove high probability bounds, and then integrate the resulting bound. Following this method, we derive a bound on the expected regret of Exp3.P using \eqref{eq:exp3Pbound2}.
\begin{theorem}[Expected regret of Exp3.P]
If Exp3.P is run with
\[
    \beta = \sqrt{\frac{\ln K }{n K}}, \quad \eta = 0.95 \sqrt{\frac{\ln K }{n K}}, \quad \gamma = 1.05 \sqrt{\frac{K \ln K }{n}}
\]
then
\begin{equation} \label{eq:exp3Pbound3}
\E\,R_n \leq 5.15 \sqrt{n K \ln K} + \sqrt{\frac{n K}{\ln K }}~.
\end{equation}
\end{theorem}
\begin{proof}
We integrate the deviations in \eqref{eq:exp3Pbound2} using the formula
\[
    \E\,W \leq \int_{0}^1 \frac{1}{\delta} \P\left(W>\ln\frac{1}{\delta}\right) d\delta
\]
for a real-valued random variable $W$. In particular, taking
\[
    W=\sqrt{\frac{\ln K }{n K}} \left(R_n - 5.15 \sqrt{n K \ln K }\right)
\]
yields $\E\,W \leq 1$, which is equivalent to \eqref{eq:exp3Pbound3}.
\end{proof}

\section{Lower Bound} \label{sec:LB}
The next theorem shows that the results of the previous sections are essentially unimprovable, up to logarithmic factors. The result is proven via the probabilistic method: we show that there exists a distribution of rewards for the arms such that the pseudo-regret of any forecaster must be high when averaged over this distribution. Owing to this probabilistic construction, the lower bound proof is based on the same Kullback-Leibler divergence as the one used in the proof of the lower bound for stochastic bandits ---see Subsection~\ref{s:stoch-lower}. We are not aware of other techniques for proving bandit lower bounds.

We find it more convenient to prove the results for rewards rather than losses. In order to emphasize that our rewards are stochastic (in particular, Bernoulli random variables), we use $Y_{i,t}\in\{0,1\}$ to denote the reward obtained by pulling arm $i$ at time $t$.
\begin{theorem}[Minimax lower bound] \label{th:LBminimax}
Let $\sup$ be the supremum over all distribution of rewards such that, for $i=1,\dots,K$, the rewards $Y_{i,1},Y_{i,2},\ldots\in\{0,1\}$ are i.i.d., and let $\inf$ be the infimum over all forecasters. Then
\begin{equation} \label{eq:LBminimax}
\inf \sup \left(\max_{i=1,\hdots,K} \E \sum_{t=1}^n Y_{i,t} - \E \sum_{t=1}^n Y_{I_t,t} \right) \geq \frac{1}{20} \sqrt{n K}
\end{equation}
where expectations are with respect to both the random generation of rewards and the internal randomization of the forecaster.
\end{theorem}
Since $\max_{i=1,\hdots,K} \E \sum_{t=1}^n Y_{i,t} - \E \sum_{t=1}^n Y_{I_t,t} = \oR_n \le \E\,R_n$, Theorem~\ref{th:LBminimax} immediately entails a lower bound on the regret of any forecaster.

The general idea of the proof goes as follows. Since at least one arm is pulled less than $n/K$ times, for this arm one cannot differentiate between a Bernoulli of parameter $1/2$ and and a Bernoulli of parameter $1/2+\sqrt{K/n}$. Thus, if all arms are Bernoulli of parameter $1/2$ but one, whose parameter is $1/2+\sqrt{K/n}$, then the forecaster should incur a regret of order $n \sqrt{K/n} = \sqrt{n K}$. In order to formalize this idea, we use the Kullback-Leibler divergence, and in particular Pinsker's inequality, to compare the behavior of a given forecaster against: (1) the distribution where all arms are Bernoulli of parameter $1/2$; (2) the same distribution where the parameter of one arm is increased by $\epsilon$. 

We start by proving a more general lemma, which could also be used to derive lower bounds in other contexts. The proof of Theorem \ref{th:LBminimax} then follows by a simple optimization over $\epsilon$.
\begin{lemma} \label{lem:LBminimax}
Let $\epsilon \in [0,1)$. For any $i \in \{1,\hdots, K\}$ let $\E_i$ be the expectation against the joint distribution of rewards where all arms are i.i.d.\ Bernoulli of parameter $\frac{1-\epsilon}{2}$ but arm $i$, which is i.i.d.\ Bernoulli of parameter $\frac{1+\epsilon}{2}$. Then, for any forecaster,
$$\max_{i = 1, \hdots, K}  \E_i \sum_{t=1}^n \bigl(Y_{i,t} - Y_{I_t,t}\bigr) \geq n \epsilon \left(1 - \frac{1}{K} - \sqrt{\epsilon \ln\frac{1+\epsilon}{1-\epsilon}}  \sqrt{\frac{n}{2 K}}\right)~.$$
\end{lemma}
\begin{proof}
We provide a proof in five steps by lower bounding $\frac{1}{K} \sum_{i=1}^K \E_i \sum_{t=1}^n (Y_{i,t} - Y_{I_t,t})$. This implies the statement of the lemma because a max is larger than a mean.
\subsection*{First step: Empirical distribution of plays.}
We start by considering a deterministic forecaster. Let $q_n=\bigl(q_{1,n},\hdots,q_{K,n}\bigr)$ be the empirical distribution of plays over the arms defined by $q_{i,n} = \frac{T_i(n)}{n}$ ---recall from Chapter~\ref{stochastic} that $T_i(n)$ denotes the number of times arm $i$ was selected in the first $n$ rounds. Let $J_n$ be drawn according to $q_n$. We denote by $\P_i$ the law of $J_n$ against the distribution where all arms are i.i.d.\ Bernoulli of parameter $\frac{1-\epsilon}{2}$ but arm $i$, which is i.i.d.\ Bernoulli of parameter $\frac{1+\epsilon}{2}$ (we call this the $i$-th stochastic adversary). Recall that $\P_{i}(J_n = j) = \E_i\tfrac{T_j(n)}{n}$, hence
$$\E_i \sum_{t=1}^n \bigl(Y_{i,t} - Y_{I_t,t}\bigr) = \epsilon n \sum_{j \neq i} \P_{i}(J_n = j) = \epsilon n \bigl(1 - \P_{i}(J_n = i)\bigr)$$
which implies
\begin{equation} \label{eq:firststepmm}
\frac{1}{K} \sum_{i=1}^K \E_i \sum_{t=1}^n \bigl(Y_{i,t} - Y_{I_t,t}\bigr) = \epsilon n \left(1 - \frac{1}{K} \sum_{i=1}^K \P_{i}(J_n = i)\right)~.
\end{equation}

\subsection*{Second step: Pinsker's inequality.}
Let $\P_0$ be the law of $J_n$ for the distribution where all arms are i.i.d.\ Bernoulli of parameter $\frac{1-\epsilon}{2}$. Then Pinsker's inequality immediately gives $\P_{i}(J_n = i) \le \P_{0}(J_n = i) + \sqrt{\tfrac{1}{2} \K(\P_{0},\P_{i})}$, and so
\begin{equation} \label{eq:secondstepmm}
\frac{1}{K} \sum_{i=1}^K \P_{i}(J_n = i) \leq \frac{1}{K} +  \frac{1}{K} \sum_{i=1}^K \sqrt{\frac{1}{2} \K(\P_{0}, \P_{i})}~.
\end{equation}

\subsection*{Third step: Computation of $\K(\P_{0}, \P_{i})$.
}
Since the forecaster is deterministic, the sequence of rewards $Y^n = (Y_1,\hdots,Y_n) \in \{0,1\}^n$ received by the forecaster uniquely determines the empirical distribution of plays $q_n$. In particular, the law of $J_n$ conditionally to $Y^n$ is the same for any $i$-th stochastic adversary. For each $i=0, \hdots, K$, let $\P_i^n$ be the law of $Y^n$ against the $i$-th adversary. Then one can easily show that 
$
\K(\P_{0}, \P_{i}) \leq \K(\P_{0}^n, \P_{i}^n)
$.
Now we use the chain rule for Kullback-Leibler divergence ---see for example \cite[Section A.2]{CL06}--- iteratively to introduce the laws $\P^t_i$ of $Y^t=(Y_1, \hdots, Y_t)$. More precisely, we have
\begin{align}
\lefteqn{\K(\P_{0}^n, \P_{i}^n)} \notag \\
& = \K(\P_0^1, \P_i^1) + \sum_{t=2}^n \sum_{y^{t-1}} \P_0^{t-1}(y^{t-1})\,\K\bigl(\P_{0}^t(\cdot\mid y^{t-1}),\P_{i}^t(\cdot\mid y^{t-1})\bigr) \notag \\
& = \K(\P_0^1, \P_i^1) + \sum_{t=2}^n \left( \sum_{y^{t-1} \,:\, I_t = i} \P_0^{t-1}(y^{t-1})\,\K\left(\tfrac{1-\epsilon}{2},\tfrac{1+\epsilon}{2}\right) \right. \notag \\
&\qquad\qquad\qquad\qquad\qquad
+ \left. \sum_{y^{t-1} \,:\, I_t \neq i} \P_0^{t-1}(y^{t-1})\,\K\left(\tfrac{1+\epsilon}{2}, \tfrac{1+\epsilon}{2}\right)\right) \notag \\
& = \K\left(\tfrac{1-\epsilon}{2}, \tfrac{1+\epsilon}{2} \right) \E_0\,T_{i}(n)~. \label{eq:thirdstep2mm}
\end{align}
\subsection*{Fourth step: conclusion for deterministic forecasters.
}
By using that the square root is concave, and combining $\K(\P_{0}, \P_{i}) \leq \K(\P_{0}^n, \P_{i}^n)$ with \eqref{eq:thirdstep2mm}, we deduce that
\begin{align}
\frac{1}{K} \sum_{i=1}^K \sqrt{\K(\P_{0},\P_{i})}& \leq \sqrt{\frac{1}{K} \sum_{i=1}^K \K(\P_{0},\P_{i})} \notag \\
& \leq \sqrt{\frac{1}{K} \sum_{i=1}^K \K\left(\tfrac{1-\epsilon}{2}, \tfrac{1+\epsilon}{2} \right) \E_0\,T_{i}(n)} \notag \\
& = \sqrt{\frac{n}{K} \K\left(\tfrac{1-\epsilon}{2}, \tfrac{1+\epsilon}{2} \right)}~. \label{eq:fourthstepmm}
\end{align}
We conclude the proof for deterministic forecasters by applying \eqref{eq:secondstepmm} and \eqref{eq:fourthstepmm} to \eqref{eq:firststepmm}, and observing that
$
\K\left(\tfrac{1-\epsilon}{2}, \tfrac{1+\epsilon}{2} \right) = 
\epsilon \ln\tfrac{1+\epsilon}{1-\epsilon}
$.

\subsection*{Fifth step: randomized forecasters via Fubini's Theorem.}
Extending previous results to randomized forecasters is easy. Denote by $\E_r$ the expectation with respect to the forecaster's internal randomization. Then Fubini's Theorem implies
$$
\frac{1}{K} \sum_{i=1}^K \E_i\sum_{t=1}^n \E_r\bigl(Y_{i,t} - Y_{I_t,t}\bigr) = \E_r \frac{1}{K} \sum_{i=1}^K \E_i \sum_{t=1}^n \bigl(Y_{i,t} - Y_{I_t,t}\bigr)~.$$
Now the proof is concluded by applying the lower bound on $\tfrac{1}{K} \sum_{i=1}^K \E_i \sum_{t=1}^n \bigl(Y_{i,t} - Y_{I_t,t}\bigr)$, which we proved in previous steps, to each realization of the forecaster's random bits.
\end{proof}

\section{Refinements and bibliographic remarks}
The adversarial framework studied in this chapter was originally investigated in a {\em full information} setting, where at the end of each round the forecaster observes the complete loss vector $(\ell_{1,t}, \hdots, \ell_{K,t})$. We refer the reader to \cite{CL06} for the history of this problem. The Exp3 and Exp3.P strategies were introduced\footnote{In its original formulation the Exp3 strategy was defined as a mixture of exponential weights with the uniform distribution on the set of arms. It was noted in \cite{Sto05} that this mixing is not necessary, see footnote 2 on p26 in \cite{Bub10} for more details on this.} and analyzed by \cite{ACFS03}, where the lower bound of Theorem \ref{th:LBminimax} is also proven. The proofs presented in this chapter are taken from \cite{Bub10}. We now give an overview of some of the many improvements and refinements that have been proposed since these initial analyses.

\subsection{Log-free upper bounds} \label{sec:logfree}
One can see that there is a logarithmic gap between the pseudo-regret of Exp3, presented in Theorem \ref{th:Exp3}, and the minimax lower bound of Theorem \ref{th:LBminimax}. This gap was closed by \cite{AB09}, who constructed a new class of strategies and showed that some of them have a pseudo-regret of order $\sqrt{n K}$. This new class of strategies, called INF (Implicitily Normalized Forecaster), is based on the following idea. First, note that one can generalize the exponential weighting scheme of Exp3 as follows: given a potential function $\psi$, assign the probability 
\[
    p_{i,t+1}=\frac{\psi(\tilde{L}_{i,t})}{\sum_{j=1}^K \psi(\tilde{L}_{j,t})}~.
\]
This type of strategy is called a weighted average forecaster, see \cite[Chapter 2]{CL06}. In INF the normalization is done implicitily, by a translation of the losses. More precisely, INF with potential $\psi$ assigns the probability $p_{i,t+1} = \psi\bigl(C_t - \tilde{L}_{i,t}\bigr)$, where $C_t$ is the constant such that $p_{t+1}$ sum to $1$. The key to obtain a minimax optimal pseudo-regret is to take $\psi$ of the form $\psi(x) = (- \eta x)^{-q}$ with $q>1$, while Exp3 corresponds to $\psi(x) = \exp(\eta x)$. \cite{ABL11} realized that the INF strategy can be formulated as a Mirror Descent algorithm. This point of view significantly simplifies the proofs. We refer the reader to Chapter~\ref{linear} (and in particular Theorem~\ref{th:osmdzero}) for more details.

While it is possible to get log-free pseudo-regret bounds, the situation becomes significantly more complicated when one considers high probability regret and expected regret. \cite{AB10} proved that one can get a log-free expected regret if the adversary is oblivious, i.e., the sequence of loss vectors is independent of the forecaster's actions. Moreover, it is also possible to get a log-free high probability regret if the adversary is fully oblivious (i.e., the loss vectors are independently drawn, but not identically distributed ---this includes the oblivious adversary). It is conjectured (in \cite{AB10}) that it is not possible to obtain a log-free expected regret bound against a general non-oblivious adversary.

\subsection{Adaptive bounds}
One of the strengths of the bounds proposed in this chapter is also one of its weaknesses: the bounds hold against any adversary. It is clear that in some cases it is possible to obtain a much smaller regret than the worst case regret. For example, when the sequence of losses is an i.i.d.\ sequence, we proved in Chapter \ref{stochastic} that it is is possible to obtain a logarithmic pseudo-regret (provided that the gap $\Delta$ is considered as a constant). Thus it is natural to ask if it possible to have strategies with minimax optimal regret, but also with much smaller regret when the loss sequence is not worst case.

The first bound in this direction was proved by \cite{ACFS03}, who showed that, for the gain version of the problem and against an oblivious adversary, Exp3 has a pseudo-regret of order $\sqrt{K G_n^*}$ (omitting log factors), where $G_n^* \leq n$ is the maximal cumulative reward of the optimal arm after $n$ rounds. This result was improved by \cite{AB10}, who showed that using the gain estimate
\[
    \tilde{g}_{i,t} = - \frac{\ds1_{I_t = i}}{\beta} \ln \left(1 - \frac{\beta g_{i,t}}{p_{i,t}} \right)
\]
one can bound the regret with high probability by essentially the same quantity as before, and against any adversary.

Another direction was explored by \cite{HK09} building on previous works in the full information setting ---see \cite{CMS07}. In this work the authors proved that one can attain a regret of order $\sqrt{\sum_{i=1}^K V_{i,n}}$ excluding log factors, where
\[
    V_{i,n} = \sum_{t=1}^n \left(\ell_{i,t} - \frac{1}{n} \sum_{s=1}^n \ell_{i,s} \right)^2
\]
is the total variation of the loss for arm $i$. In fact their result is more general, as it applies to the linear bandit framework ---see Chapter~\ref{linear}. The main new ingredient in their analysis is a ``reservoir sampling'' procedure. We refer the reader to \cite{HK09} for details. See also the works of \cite{slivkins2008adapting,slivkins2009contextual} for related results on slowly changing bandits.

In Section \ref{sec:best-of-both} below we describe another type of adaptive bound, where one combines minimax optimal regret for the adversarial model with logarithmic pseudo-regret for the stochastic model.

\subsection{Competing with the best switching strategy} \label{sec:exp3S}
While competing against the policy consistently playing the best fixed arm is a natural way of defining regret, in some applications it might be interesting to consider regret with respect to a bigger class of policies. Though this problem is the focus of Chapter~\ref{contextual}, there is a class of natural policies that can be directly dealt with by the methods of this chapter. Namely, consider the problem of competing against any policy constrained to make at most $S \leq n$ switches (a switch is when the arm played at time $t$ is different from the arm played at time $t+1$). This problem was studied by \cite{Aue02}, where it was first shown that a simple variant of Exp3 attains a low switching regret against oblivious adversaries. Later, \cite{AB10} proved that Exp3.P attains an expected regret (and a high probability regret) of order $\sqrt{n K S \ln (n K /S)}$ for this problem.

\subsection{Stochastic versus adversarial bandits} \label{sec:best-of-both}
From a practical viewpoint, Exp3 should be a safe choice when we have reasons to believe that the sequence of rewards is not well matched by any i.i.d.\ process. Indeed, it is easy to prove that UCB can have linear regret, i.e. $\oR_n = \Omega(n)$, on certain deterministic sequences. In \cite{BS12} a new strategy was described, called SAO (Stochastic and Adversarial Optimal), which enjoys (up to logarithmic factors) both the guarantee of Exp3 for the adversarial model and the guarantee of UCB for the stochastic model. More precisely SAO satisfies $\oR_n = \mathcal{O} \left(\frac{K}{\Delta} \log^2(n) \log(K) \right)$ in the stochastic model and $\oR_n = \mathcal{O} \left( \sqrt{n K} \log^{3/2}(n) \log(K) \right)$ in the adversarial model. Note that while this result is a step towards more flexible strategies, the very notion of regret $\oR_n$ can become vacuous with nonstationarities in the reward sequence, since the total reward of the best fixed action might be very small. In that case the notion of switching regret ---see Subsection \ref{sec:exp3S}--- is more relevant, and it would be interesting to derive a strategy with logarithmic regret in the stochastic model, and a switching regret of order $\sqrt{n K S}$ in the adversarial model.

\subsection{Alternative feedback structures} \label{sec:feedback}
As mentioned at the beginning of this section, the adversarial multi-armed bandit is a variation of the full information setting, with a weaker feedback signal (only the incurred loss versus the full vector of losses is observed). Many other feedback structures can be considered, and we conclude the chapter by describing a few of them.

In the {\em label efficient} setting, originally proposed by \cite{HP97}, at the end of each round the forecaster has to decide whether to ask for the losses of the current round, knowing that this can be done for at most $m \leq n$ times. In this setting, \cite{CLS05} proved that the minimax pseudo-regret is of order $n \sqrt{\frac{\ln K}{m}}$. A bandit label efficient version was proposed by \cite{AAGO06}. \cite{AB10} proved that the minimax pseudo-regret for the bandit label efficient version is of order $n \sqrt{\frac{K}{m}}$. These results do not require any fundamentally new algorithmic idea, besides the fact the forecaster has to randomize to select the rounds in which the losses are revealed. Roughly speaking, a simple coin toss with parameter $\epsilon = m/n$ is sufficient to obtain an optimal regret.

\cite{MS11} study a model that interpolates between the full information and the bandit setting. The basic idea is that there is an undirected graph $G$ with $K$ vertices (one vertex for each arm) that encodes the feedback structure. When one pulls arm $i$ the losses of all neighboring arms $j \in N(i)$ in the graph are observed. Thus, a graph with no edges is equivalent to the bandit problem, while the complete graph is equivalent to the full information setting. Given the feedback structure $G$, it is natural to consider the following unbiased loss estimate
\[
    \tilde{\ell}_{i,t} = \frac{\ell_{i,t} \ds1_{i \in N(I_t)}}{\sum_{j \in N(i)} p_{j,t}}~.
\]
Using Exp3 with this loss estimate, the authors show that the minimax pseudo-regret (up to logarithmic factors) is of order of $\sqrt{\alpha(G) n}$, where $\alpha(G)$ is the independence number of graph $G$. Note that this interpolated setting naturally arises in applications like ad placement on websites. Indeed, if a user clicks on an advertisement, it is plausible to assume that the same user would have clicked on similar advertisements, had they been displayed.

The above problems are all specific examples of the more general partial monitoring setting. In this model, at the end of each round the player does not observe the incurred loss $\ell_{I_t,t}$ but rather a stochastic ``signal'' $S_{I_t,t}$. A prototypical example of this scenario is the following: a website is repeatedly selling the same item to a sequence of customers. The selling price is dynamically adjusted, and each customer buys the item only if the current price is smaller or equal than his own hidden value for the item. The pricing algorithm (i.e., the player in our terminology) does not see each user's value, but only whether the user bought the item or not.

The relationship between the signals and the incurred losses defines the instance of a partial monitoring problem. We refer the interested reader to \cite{CL06} for more details, including an historical account. Recent progress on this problem has been made by \cite{BPS10} and~\cite{FR11}.

%% file: contextual.tex
A natural extension of the multi-armed armed problem is obtained by associating side information with each arm. Based on this side information, or context, a notion of ``contextual regret'' is introduced where optimality is defined with respect to the best policy (i.e., mapping from contexts to arms) rather than the best arm. The space of policies, within which the optimum is sought, is typically chosen in order to have some desired structure. A different viewpoint is obtained when contexts are privately accessed by the policies (which are then appropriately called ``experts''). In this case the contextual information is hidden from the forecaster, and arms must be chosen based only on the past estimated performance of the experts.

Contextual bandits naturally arise in many applications. For example, in personalized news article recommendation the task is to select, from a pool of candidates, a news article to display whenever a new user visits a website. The articles correspond to arms, and a reward is obtained whenever the user clicks on the selected article. Side information, in the form of features, can be extracted from both user and articles. For the user this may include historical activities, demographic information, and geolocation; for the articles, we may have content information and categories. See \cite{li2010contextual} for more details on this application of contextual bandits.

In general, the presence of contexts creates a wide spectrum of possible variations obtained by combining assumptions on the rewards with assumptions on the nature of contexts and policies. In this chapter we describe just a few of the results available in the literature, and use the bibliographic remarks to mention all those that we are aware of.

\section{Bandits with side information}
\label{s:context-intro}
The most basic example of contextual bandits is obtained when game rounds $t=1,2,\dots$ are marked by contexts $s_1,s_2,\dots$ from a given context set $\cS$. The forecaster must learn the best mapping $g : \cS \to \{1,\dots,K\}$ of contexts to arms. We analyze this simple side information setting in the case of adversarial rewards, and we further assume that the sequence of contexts $s_t$ is arbitrary but fixed. The approach we take is the simplest: run a separate instance of Exp3 on each distinct context.

We introduce the following notion of pseudoregret
\[
	\oRS_n = \max_{g \,:\, \cS\to\{1,\dots,K\}} \E\left[ \sum_{t=1}^n \ell_{I_t,t} - \sum_{t=1}^n \ell_{g(s_t),t} \right]~.
\]
Here $s_t\in\cS$ denotes the context marking the $t$-th game round. A bound on this pseudoregret is almost immediately obtained using the adversarial bandit results from Section~\ref{adversarial}. 
\begin{theorem}
\label{th:context-simple}
There exists a randomized forecaster for bandits with side information (the $\cS$-Exp3 forecaster, defined in the proof) that satisfies
\[
	\oRS_n \le \sqrt{2n |\cS| K\ln K}
\]
for any set $\cS$ of contexts.
\end{theorem}
\begin{proof}
Let $S = |\cS|$.
The $\cS$-Exp3 forecaster runs an instance of Exp3 on each context $s\in\cS$. Let $n_s$ the number of times when $s_t=s$ within the first $n$ time steps. Using the bound~(\ref{eq:exp3bound}) established in Theorem~\ref{th:Exp3} we get
\begin{align*}
	\max_{g \,:\, \cS\to\{1,\dots,K\}} \E\left[ \sum_{t=1}^n \bigl(\ell_{I_t,t} - \ell_{g(s_t),t}\bigr) \right]
&=
	\sum_{s\in\cS} \max_{k=1,\dots,K} \E\left[ \sum_{t \,:\, s_t=s} \bigl(\ell_{I_t,t} - \ell_{k,t} \bigr)\right]
\\ &\le
	\sum_{s\in\cS} \sqrt{2 n_s K\ln K}
\\ &\le
	\sqrt{2 n S K\ln K}
\end{align*}
where in the last step we used Jensen's inequality and the identity $\sum_s n_s = n$.
\end{proof}
In subsection~\ref{ss:best-context}, we extend this construction by considering several context sets simultaneously.

A lower bound $\Omega\bigl(\sqrt{nSK}\bigl)$ is an immediate consequence of the adversarial bandit lower bound (Theorem~\ref{th:LBminimax}) under the assumption that a constant fraction of the contexts in $\cS$ marks at least constant fraction of the $n$ game rounds.

\section{The expert case}
\label{s:context-adv}
We now consider the contextual variant of the basic adversarial bandit model of Chapter~\ref{adversarial}. In this variant there is a finite set of $N$ randomized policies. Following the setting of prediction with expert advice, no assumptions are made on the way policies compute their randomized predictions, and the forecaster experiences the contexts only through the advice provided by the policies. For this reason, in what follows we use the word expert to denote a policy. Calling this a model of contextual bandits may sound a little strange, as the structure of contexts does not seem to play a role here. However, we have decided to include this setting in this chapter because bandit with experts have been used in practical contextual bandit problems -see, e.g., the news recommendation experiment in \cite{beygelzimer2010contextual}.

Formally, at each step $t=1,2,\dots$ the forecaster obtains the expert advice $\bigl(\xi_t^1,\dots,\xi_t^N\bigl)$, where each $\xi_t^j$ is a probability distribution over arms representing the randomized play of expert $j$ at time $t$. If $\ell_t = \bigl(\ell_{1,t},\dots,\ell_{K,t}\bigr) \in [0,1]^K$ is the vector of losses incurred by the $K$ arms at time $t$, then $\E_{i \sim \xi_t^j} \ell_{i,t}$ denotes the expected loss of expert $j$ at time $t$. We allow the expert advice to depend on the realization of the forecaster's past random plays. This fact is explicitely used in the proof of Theorem~\ref{th:theta-regret}.

Similarly to the pseudo-regret~\eqref{eq:pseudoregretadv} for adversarial bandits, we now introduce the pseudo-regret $\oRx_n$ for the adversarial contextual bandit problem,
\[
	\oRx_n = \max_{i=1,\dots,N} \E\left[ \sum_{t=1}^n \ell_{I_t,t} - \sum_{t=1}^n \E_{k \sim \xi_t^i} \ell_{k,t} \right]~.
\]

\begin{figure}[t]
\bookbox{
{\em Exp4 (Exponential weights algorithm for Exploration and Exploitation with Experts) without mixing:}

\medskip\noindent
{Parameter:} a non-increasing sequence of real numbers $(\eta_t)_{t \in \N}$.

\medskip\noindent
Let $q_1$ be the uniform distribution over $\{1,\hdots,N\}$.

\medskip\noindent
For each round $t=1,2,\ldots,n$
\begin{itemize}
\item[(1)]
Get expert advice $\xi_t^1,\dots,\xi_t^N$, where each $\xi_t^j$ is a probability distribution over arms.
\item[(2)]
Draw an arm $I_t$ from the probability distribution $p_t = \bigl(p_{1,t},\dots,p_{K,t}\bigr)$, where $p_{i,t} = \E_{j \sim q_t} \xi_{i,t}^j$.
\item[(3)]
Compute the estimated loss for each arm
\[
	\tilde{\ell}_{i,t} = \frac{\ell_{i,t}}{p_{i,t}} \ds1_{I_t = i} \qquad i=1,\dots,K~.
\]
\item[(4)]
Compute the estimated loss for each expert
\[
	\tilde{y}_{j,t} = \E_{i \sim \xi_t^j}\tilde{\ell}_{i,t} \qquad j=1,\dots,N~.
\]
\item[(5)]
Update the estimated cumulative loss for each expert
$\tilde{Y}_{j,t} = \sum_{s=1}^t \tilde{y}_{j,s}$ for $j=1,\dots,N$.
\item[(6)]
Compute the new probability distribution over the experts $q_{t+1}=\bigl(q_{1,t+1},\dots,q_{N,t+1}\bigr)$, where
\[
	q_{j,t+1} = \frac{\exp{\left(- \eta_t \tilde{Y}_{j,t}\right)}}{\sum_{k=1}^N \exp{\left(- \eta_t \tilde{Y}_{k,t}\right)}}~.
\]
\end{itemize}
}
\caption{Exp4 forecaster.}
\label{fig:exp4}
\end{figure}

In order to bound the contextual pseudo-regret $\oRx_n$, one could naively use the Exp3 strategy of Chapter~\ref{adversarial} on the set of experts. This would give a bound of order $\sqrt{n N \log N}$. In Figure~\ref{fig:exp4} we introduce the contextual forecaster Exp4 for which we show a bound of order $\sqrt{n K \ln N}$. Thus, in this framework we can be competitive even with an exponentially large (with respect to $n$) number of experts.

Exp4 is a simple adaptation of Exp3 to the contextual setting. Exp4 runs Exp3 over the $N$ experts using estimates of the experts' losses $\E_{i \sim \xi_t^j} \ell_{i,t}$. In order to draw arms, Exp4 mixes the expert advice with the probability distribution over experts maintained by Exp3. The resulting bound on the pseudo-regret is of order $\sqrt{nK\ln N}$, where the term $\sqrt{\ln N}$ comes from running Exp3 over the $N$ experts, while $\sqrt{K}$ is a bound on the second moment of the estimated expert losses under the distribution $q_t$ computed by Exp3. Inequality~\eqref{eq:exp4-ineq3} shows that $\E_{j \sim q_t} \tilde{y}^2_{j,t} \le \E_{i \sim p_t} \tilde{\ell}_{i,t}^2$. That is, this second moment is at most that of the estimated arm losses under the distribution $p_t$ computed by Exp4, which in turn is bounded by $\sqrt{K}$ using techniques from Chapter~\ref{adversarial}.
\begin{theorem}[Pseudo-regret of Exp4]
\label{th:Exp4}
Exp4 without mixing and with $\eta_t=\eta = \sqrt{\frac{2 \ln N}{n K}}$ satisfies
\begin{equation} \label{eq:exp4bound}
	\oRx_n \leq \sqrt{2 n N \ln K}~.
\end{equation}
On the other hand, with $\eta_t = \sqrt{\frac{\ln N}{t K}}$ it satisfies
\begin{equation} \label{eq:exp4boundanytime}
	\oRx_n \leq 2 \sqrt{n N \ln K}~.
\end{equation}
\end{theorem}

\begin{proof}
We apply the analysis of Exp3 (Theorem~\ref{th:Exp3}) to a forecaster using distributions $q_t$ over $N$ experts, whose pseudo-losses are $\tilde{y}_{j,t}$ for $j=1,\dots,N$. This immediately gives the inequality
\begin{equation}
\label{eq:exp4-exp3}
	\sum_{t=1}^n \E_{j \sim q_t}\tilde{y}_{j,t}
\le
	\tilde{Y}_{k,n} + \frac{\log N}{\eta_n} + \frac{1}{2}\sum_{t=1}^n \eta_t\,\E_{j \sim q_t}\,\tilde{y}_{j,t}^2~.
\end{equation}
Now, similarly to \eqref{eq:equalities} in the proof of Theorem~\ref{th:Exp3}, we establish the following inequalities
\begin{align}
\label{eq:exp4-ineq1}
	&\E_{I_t \sim p_t} \tilde{y}_{k,t} = \E_{I_t \sim p_t} \E_{i \sim \xi_t^k} \tilde{\ell}_{i,t} = \E_{i \sim \xi_t^k} \ell_{i,t} = y_{k,t}
\\
\label{eq:exp4-ineq2}
	&\E_{j \sim q_t} \tilde{y}_{j,t} = \E_{j \sim q_t} \E_{i \sim \xi_t^j} \tilde{\ell}_{i,t} = \E_{i \sim p_t} \tilde{\ell}_{i,t} = \ell_{I_t,t}
\\
\label{eq:exp4-ineq3}
	&\E_{j \sim q_t} \tilde{y}^2_{j,t} = \E_{j \sim q_t} \left(\E_{i \sim \xi_t^j} \tilde{\ell}_{i,t}\right)^2 \le \E_{j \sim q_t} \E_{i \sim \xi_t^j} \tilde{\ell}_{i,t}^2 = \E_{i \sim p_t} \tilde{\ell}_{i,t}^2 = \frac{\ell_{I_t,t}^2}{p_{I_t,t}}
\end{align}
where we used Jensen's inequality to prove~\eqref{eq:exp4-ineq3}.
By applying~\eqref{eq:exp4-ineq2} and~\eqref{eq:exp4-ineq3} to~\eqref{eq:exp4-exp3} we get
\[
	\sum_{t=1}^n \ell_{I_t,t}
=
	\sum_{t=1}^n \E_{j \sim q_t} \tilde{y}_{j,t}
\le
	\tilde{Y}_{k,n} + \frac{\log N}{\eta_n} + \frac{1}{2}\sum_{t=1}^n \eta_t\,\frac{\ell_{I_t,t}^2}{p_{I_t,t}}~.
\]
Now note that, if we take expectation over the draw of $I_1,\dots,I_n$, using~\eqref{eq:exp4-ineq1} we obtain
\[
	\E\,\tilde{Y}_{k,n} = \E\left[\sum_{t=1}^n \E\bigl[\tilde{y}_{j,n} \,\big|\, I_1,\dots,I_{t-1} \bigr] \right] =  \E\left[\sum_{t=1}^n \E_{i \sim \xi_t^k} \ell_{i,t} \right] = \E\,Y_{k,n}~.
\]
Hence,
\[
	\oRx_n = \max_{k=1,\dots,N} \E\left[ \sum_{t=1}^n \ell_{I_t,t} - Y_{k,n} \right]
\le
	 \frac{\log N}{\eta_n} + \frac{K}{2}\sum_{t=1}^n \eta_t~.
\]
Choosing $\eta_t$ as in the statement of the Theorem, and using the inequality $\sum_{t=1}^n t^{-1/2} \le 2\sqrt{n}$, concludes the proof.
\end{proof}
Besides pseudo-regret, the contextual regret
\[
	\Rx_n = \max_{k=1,\dots,N} \left( \sum_{t=1}^n \ell_{I_t,t} - \sum_{t=1}^n \E_{i \sim \xi_t^J} \ell_{i,t} \right)
\]
can be also bounded, at least with high probability. Indeed, similarly to the variant Exp3.P of Exp3 (see Section~\ref{sec:regret}), an analogous modification of Exp4, called Exp4.P, satisfies
\[
    \Rx_n \le c\sqrt{n K \ln(N \delta^{-1})}
\]
for some constant $c > 0$ and with probability at least $1-\delta$, where $\delta \in (0,1)$ is a parameter of the algorithm.

\subsection{Competing against the best context set}
\label{ss:best-context}
We revisit the basic contextual scenario introduced in Section~\ref{s:context-intro}, where the goal is to compete against the best mapping from contexts to arms. Consider now a class $\theset{\cS_{\theta}}{\theta\in\Theta}$ of context sets. In this new game, each time step $t=1,2,\dots$ is marked by the vector $\bigl(s_{\theta,t}\bigr)_{\theta\in\Theta}$ of contexts, one for each set in $\Theta$. Introduce the pseudoregret
\[
	\oR^{\Theta}_n = \max_{\theta\in\Theta} \max_{g \,:\, \cS_{\theta}\to\{1,\dots,K\}} \E\left[ \sum_{t=1}^n \ell_{I_t,t} - \sum_{t=1}^n \ell_{g(s_{\theta,t}),t} \right]~.
\]
When $|\Theta|=1$ we recover the contextual pseudoregret $\oRS_n$. In general, when $\Theta$ contains more than one set, the forecaster must learn both the best set $\cS_{\theta}$ and the best function $g : \cS_{\theta}\to\{1,\dots,K\}$ from that set to the set of arms.

We find this variant of contextual bandits interesting because its solution involves a nontrivial combination of two of the main algorithms examined in this chapter: Exp4 and $\cS$-Exp3. In particular, we consider a scenario in which Exp4 uses instances of $\cS$-Exp3 as experts. The interesting aspect is that these experts are learning themselves, and thus the analysis of the combined algorithm requires taking into account the learning process at both levels.

Note that in order to solve this problem we could simply lump all contexts in a big set and use the proof of Theorem~\ref{th:context-simple}. However, this would give a regret bound that depends exponentially in $|\Theta|$. On the other hand, by using Exp4 directly on the set of all policies $g$ (which is of cardinality exponential in $|\Theta|\times|S|$), we could improve this to a bound that scales with $\sqrt{|\Theta|}$. The idea we explore here is to use Exp4 over the class $\Theta$ of ``experts'', and combine this with the $\cS$-Exp3 algorithm of Theorem~\ref{th:context-simple}. This gets us down to a logarithmic dependency on $|\Theta|$, albeit at the price of a worse dependency on $n$.

Intuitively, Exp4 provides competitiveness against the best context set $\cS_{\theta}$, while the instances of the $\cS$-Exp3 algorithm, acting as experts for Exp4, ensure that we are competitive against the best function $g : \cS_{\theta} \to \{1,\dots,K\}$ for each $\theta\in\Theta$. However, by doing so we immediately run into a problem: the $p_t$ used by Exp4 is not the same as the $p_t$'s used by each expert. In order to address this issue, we now show that the analysis of Exp3 holds even when the sequence of plays $I_1,I_2,\dots$ is drawn from a sequence of distributions $q_1,q_2,\dots$ possibly different from the one chosen by the forecaster. The only requirement we need is that each probability in $q_t$ be bounded away from zero.
\begin{theorem}
\label{th:exp3-rebel}
Consider a $K$-armed bandit game in which at each step $t=1,2,\dots$ the played arm $I_t$ is drawn from an arbitrary distribution $q_t$ over arms. Each $q_t$ may depend in an arbitrary way on the pairs $(I_1,\ell_{I_1,1}),\dots,(I_{t-1},\ell_{I_{t-1},t-1})$. Moreover, $q_{t,i} \ge \ve > 0$ for all $i=1,\dots,K$ and $t \ge 1$.

If Exp3 without mixing is run with $\tilde{\ell}_{i,t} = \frac{\ell_{i,t}}{q_{i,t}} \ds1_{I_t = i}$ and $\eta_t = \eta = \sqrt{\frac{2\ln K}{n K}}$ then
\begin{equation}
    \max_{k=1,\dots,K} \E_{I^n \sim q^n} \left[ \sum_{t=1}^n \E_{i \sim p_t} \ell_{i,t} - \sum_{t=1}^n \ell_{k,t} \right] \le
    \sqrt{\frac{2 n}{\ve}\ln K}
\end{equation}
where $I^n \sim q^n$ means that each $I_t$ is drawn from $q_t$ for $t=1,\dots,n$, and $p_t$ is the distribution used by Exp3 at time $t$.
\end{theorem}
\begin{proof}
The proof is an easy adaptation of Exp3 analysis (Theorem~\ref{th:Exp3} in Section~\ref{adversarial}) and we just highlight the differences.
The key step is the analysis of the log-moment of $\tilde{\ell}_{i,t}$:
\begin{align*}
    \E_{i \sim p_t} \tilde{\ell}_{i,t} =&\, \frac{1}{\eta} \log \E_{i \sim p_t} \exp{\left(- \eta (\tilde{\ell}_{i,t} - \E_{k \sim p_t} \tilde{\ell}_{k,t}) \right)}
\\ &-
    \frac{1}{\eta} \log \E_{i \sim p_t} \exp{\left(- \eta \tilde{\ell}_{i,t} \right)}~.
\end{align*}
The first term is bounded in a manner slightly different from the proof of Theorem~\ref{th:Exp3},
\[
    \log \E_{i \sim p_t} \exp{\left(- \eta (\tilde{\ell}_{i,t} - \E_{k \sim p_t} \tilde{\ell}_{k,t}) \right)} 
\le
    \frac{\eta^2}{2}\,\E_{i \sim p_t}\,\tilde{\ell}_{i,t}^2
\le
    \frac{\eta^2}{2}\,\frac{p_{I_t,t}}{q_{I_t,t}^2}~.
\]
The analysis of the second term is unchanged:
Let $\tilde{L}_{i,0}=0$, $\Phi_0(\eta)=0$ and $\Phi_t(\eta) = \frac{1}{\eta} \log \frac{1}{K} \sum_{i=1}^K \exp{\left(- \eta \tilde{L}_{i,t}\right)}$. Then by definition of $p_t$ we have:
\[
    - \frac{1}{\eta} \log \E_{i \sim p_t} \exp{\left(- \eta \tilde{\ell}_{i,t} \right)} 
=
    \Phi_{t-1}(\eta) - \Phi_{t}(\eta)~.
\]
Proceeding again as in the proof of Theorem~\ref{th:Exp3} we obtain
\begin{align*}
    \E_{I^n \sim q^n}\left[\sum_{t=1}^n \E_{i \sim p_t} \tilde{\ell}_{i,t}\right]
\le
    \E_{I^n \sim q^n}\left[\sum_{t=1}^n \tilde{\ell}_{k,t} + \frac{\eta}{2}\,\frac{p_{I_t,t}}{q_{I_t,t}^2} \right] + \frac{\ln K}{\eta}~.
\end{align*}
Now observe that
\begin{align*}
    \E_{I_t \sim q_t} \tilde{\ell}_{k,t} = \ell_{k,t}
\qquad\text{and}\qquad
    \E_{I_t \sim q_t} \frac{p_{I_t,t}}{q_{I_t,t}^2} = \sum_{i=1}^K \frac{p_{i,t}}{q_{i,t}} \le \frac{1}{\ve}~.
\end{align*}
Therefore
\begin{align*}
    \E_{I^n \sim q^n}\left[\sum_{t=1}^n \Bigl( \E_{i \sim p_t} \ell_{i,t} - \ell_{k,t} \Bigr) \right]
&=
    \E_{I^n \sim q^n}\left[\sum_{t=1}^n \Bigl( \E_{i \sim p_t} \tilde{\ell}_{i,t} - \tilde{\ell}_{k,t} \Bigr) \right]
\\ &\le
    \frac{\eta n}{2\ve} + \frac{\ln K}{\eta}~.
\end{align*}
Choosing $\eta$ as in the statement of the theorem concludes the proof.
\end{proof}
It is left to the reader to verify that the analysis of $\cS$-Exp in Theorem~\ref{th:context-simple} can be combined with the above analysis to give the bound
\begin{equation}
\label{eq:exp3-rebel-cont}
    \max_{g \,:\, \cS\to\{1,\dots,K\}} \E_{I^n \sim q^n} \left[ \sum_{t=1}^n \E_{i \sim p_t} \ell_{i,t} - \sum_{t=1}^n \ell_{g(s_t),t} \right] \le
    \sqrt{\frac{2 n}{\ve}|\cS|\ln K}~.
\end{equation}
Next, we state a bound on the contextual pseudoregret of a variant of Exp4 whose probabilities $p_{i,t}$ satisfy the property $p_{i,t} \ge \tfrac{\gamma}{K}$ for all $i=1,\dots,K$ and $t \ge 1$, where $\gamma > 0$ is a parameter. This is obtained by replacing in Exp4 the assignment $p_{i,t} = \E_{j \sim q_t} \xi_{i,t}^j$ (line 2 in Figure~\ref{fig:exp4}) with the assignment
\[
	p_{i,t} = (1-\gamma)\E_{j \sim q_t} \xi_{i,t}^j + \frac{\gamma}{K}
\]
where $\gamma > 0$ is the mixing coefficient. This mixing clearly achieves the desired property for each $p_{i,t}$.
\begin{theorem}[Pseudo-regret of Exp4 with mixing]
\label{th:exp4-mixing}
Exp4 with mixing coefficient $\gamma$ and with $\eta_t = \eta = \gamma/K$ satisfies
\begin{equation}
\label{eq:exp4bound-mix}
	\oRx_n \le \frac{\gamma\,n}{2} + \frac{K\ln N}{\gamma}~.
\end{equation}
\end{theorem}
\begin{proof}
The proof goes along the same lines of Exp4 original proof~\cite[Theorem~7.1]{ACFS03} with the following modifications: since the weights are negative exponentials, we can use the bound $\exp(-x) \le 1 - x + \tfrac{x^2}{2}$ for all $x \ge 0$ rather than $\exp(x) \le 1 + x + (e-2)x^2$ for all $0 \le x \le 1$; the term $(1-\gamma)\sum_t\ell_{k,t}$ is upper bounded directly by $\sum_t\ell_{k,t}$; the term $\tfrac{\gamma}{K}\sum_t\sum_i \ell_{i,t}$ is upper bounded by $\gamma\,n$ without requiring the assumption that the expert set contains the ``uniform expert''. Finally, the fact that experts' distributions $\xi_t^j$ depend on the realization of past forecaster's random arms is dealt with in the same way as in the proof of Theorem~\ref{th:Exp4}.
\end{proof}
\begin{theorem}
\label{th:theta-regret}
There exists a randomized forecaster achieving
\[
	\oR^{\Theta}_n = \mathcal{O}\left(n^{2/3}\left(\max_{\theta\in\Theta}|\cS_{\theta}|K\ln K\right)^{1/3}\sqrt{\ln|\Theta|}\right)
\]
for any class $\theset{\cS_{\theta}}{\theta\in\Theta}$ of context sets.
\end{theorem}
\begin{proof}
We run the Exp4 forecaster with mixing coefficient $\gamma$ using instances of the $\cS$-Exp3 algorithm (defined in the proof of Theorem~\ref{th:context-simple}) as experts. Each $\cS$-Exp3 instance is run on a different context set $\cS_{\theta}$ for $\theta\in\Theta$. Let $\xi_t^{\theta}$ be the distribution used at time $t$ by the $\cS$-Exp3 instance running on context set $\cS_{\theta}$ and let $p^n$ be the joint distribution of $I^n = (I_1,\dots,I_n)$ used by Exp4. Since $p_{i,t} \ge \tfrac{\gamma}{K}$ for all $i=1,\dots,K$ and $t\ge 1$, we can use~(\ref{eq:exp3-rebel-cont}) with $\ve = \gamma/K$. Thus, Theorem~\ref{th:exp4-mixing} implies
\begin{align*}
    \E_{I^n \sim p^n}\left[ \sum_{t=1}^n \ell_{I_t,t} \right]
\le &\,
    \min_{\theta\in\Theta} \E_{I^n \sim p^n}\left[\sum_{t=1}^n \E_{k \sim \xi_t^\theta} \ell_{k,t} \right]
+
    \frac{\gamma\,n}{2} + \frac{K\ln|\Theta|}{\gamma}
\\ \le &\,
    \min_{\theta\in\Theta} \min_{g \,:\, \cS_{\theta}\to\{1,\dots,K\}} \E\left[ \sum_{t=1}^n \ell_{g(s_t),t} \right]
\\ &
    + \sqrt{\frac{2 n}{\ve}\max_{\theta\in\Theta}|\cS_{\theta}| \ln K} + \frac{\gamma\,n}{2} + \frac{K\ln |\Theta|}{\gamma}~.
\end{align*}
Substituting $\ve = \gamma/K$ in the above expression and choosing $\gamma$ of the order of
$
	n^{-1/3}\left(\max_{\theta\in\Theta}|\cS_{\theta}|K\ln K\right)^{1/3}\sqrt{\ln|\Theta|}
$
gives the desired result.
\end{proof}
Note that in Theorem~\ref{th:theta-regret} the rate is $n^{2/3}$, in contrast to the more usual $n^{1/2}$ bandit rate. This worsening is inherent in the Exp4-over-Exp3 construction. It is not known whether the rate could be improved while keeping the same logarithmic dependence on $|\Theta|$ guaranteed by this construction.

\section{Stochastic contextual bandits}
\label{s:context-stochastic}
We now move on to consider the case in which policies have a known structure. More specifically, each policy is a function $f$ mapping the context space to the arm space $\{1,\dots,K\}$ and the set $\cF$ of policies is given as an input parameter to the forecaster.

Under this assumption on the policies, the problem can be viewed as a bandit variant of supervised learning. For this reason, here and in the next section we follow the standard notation of supervised learning and use $x$ rather than $s$ to denote contexts.

In supervised learning, we observe data of the form $(x_t,\ell_t)$. In the contextual bandit setting, the observed data are $(x_t,\ell_{I_t,t})$ where $I_t$ is the arm chosen by the forecaster at time $t$ given context $x_t\in\cX$. This connection to supervised learning has steered the focus of research towards stochastic data generation models, which are widespread in the analysis of supervised learning. In the stochastic variant of contextual bandits, contexts $x_t$ and arm losses $\ell_t = (\ell_{1,t},\dots,\ell_{K,t})$ are realizations of i.i.d.\ draws from a fixed and unknown distribution $D$ over $\cX \times [0,1]^K$. In tight analogy with statistical learning theory, a policy $f$ is evaluated in terms of its statistical risk
$
    \ell_D(f) = \E_{(x,\ell) \sim D} \ell_{f(x)}
$.
Let
\[
    f^* = \arginf_{f\in\cF} \ell_D(f)
\]
the risk-minimizing policy in the class.
The regret with respect to the class $\cF$ of a forecaster choosing arms $I_1,I_2,\dots$ is then defined by
\[
    \sum_{t=1}^n \ell_{I_t,t} - n\,\ell_D(f^*)~.
\]
This can be viewed as the stochastic counterpart of the adversarial contextual regret $\oRx_n$ introduced in Section~\ref{s:context-adv}. The main question is now to characterize the ``price of bandit information'' using the sample complexity of supervised learning as yardstick.

In the rest of this section we focus on the case of $K=2$ arms and parametrize classes $\cF$ of policies $f : \cX\to\{1,2\}$ by their VC-dimension $d$ ---see~\cite{BBL05} for a modern introduction to VC theory. For this setting, we consider the following forecaster.
%
\begin{center}
\bookbox{
{\em VE (VC dimension by Exponentiation):}

\medskip\noindent
{Parameters:} number $n$ of rounds, $n'$ satisfying $1 \le n' \le n$.

\begin{itemize}
\item[(1)]
For the first $n'$ rounds, choose arms uniformly at random.
\item[(2)]
Build $\cF'\subseteq\cF$ such that for any $f\in\cF$ there is exactly one $f'\in\cF'$ satisfying $f(x_t)=f'(x_t)$ for all $t=1,\dots,n'$.
\item[(3)]
For $t=n'+1,\dots,n$ play by simulating Exp4.P using the policies of $\cF'$ as experts.
\end{itemize}
}
\end{center}
%
We now show that the per round regret of VE is of order $\sqrt{d/n}$, excluding logarithmic factors. This rate is equal to the optimal rate for supervised learning of VC-classes, showing that ---in this case--- the price of bandit information is essentially zero.
\begin{theorem}
For any class $\cF$ of binary policies $f : \cX\to\{0,1\}$ of VC-dimension $d$ and for all $n > d$,
the forecaster VE run with $n' = \sqrt{n\left(2d\ln\frac{en}{d}+ \ln\frac{3}{\delta}\right)}$ satisfies
\begin{equation} \label{eq:vebound}
	\sum_{t=1}^n \ell_{I_t,t} - n\inf_{f\in\cF}\ell_D(f) \leq c\sqrt{n\left(d\ln\frac{en}{d} + \ln\frac{3}{\delta}\right)}
\end{equation}
for some constant $c > 0$ and with probability at least $1-\delta$ with respect to both the random data generation and VE's internal randomization.
\end{theorem}
\begin{proof}
Given a sample realization $(x_1,\ell_1),\dots,(x_n,\ell_n)$, let $f'$ the unique element of $\cF'$ such that $f'(x_t) = f^*(x_t)$ for all $t=1,\dots,n'$, where $f^*$ is the risk-minimizing function in $\cF$. Given a sample, we may assume without loss of generality that $\cF$ contains functions restricted on the finite domain $\{x_1,\dots,x_n\}$. Recall Sauer-Shelah lemma ---see, e.g.~\cite{BBL05}, stating that any class $\cF$ of binary functions defined on a finite domain of size $n$ satisfies $|\cF| \le \left(\frac{en}{d}\right)^d$, where $d$ is the VC-dimension of $\cF$. 
Then, with probability at least $1- \tfrac{\delta}{3}$ with respect to VE's internal randomization,
\begin{align*}
    \sum_{t=1}^n &\ell_{I_t,t}
\le
    n' + \sum_{t=n'+1}^n \ell_{f'(x_t),t} + c\sqrt{2(n-n')\ln\frac{3|\cF'|}{\delta}}
\\ &\le
    n' + \sum_{t=n'+1}^n \bigl( \ell_{f^*(x_t),t} + \ell_{f'(x_t),t} - \ell_{f^*(x_t),t}\bigr) + c\sqrt{2(n-n')\ln\frac{3|\cF'|}{\delta}}
\\ &\le
    n' + \sum_{t=n'+1}^n \bigl( \ell_{f^*(x_t),t} + \ds1_{f'(x_t) \neq f^*(x_t)} \bigr) + c\sqrt{2(n-n')\ln\frac{3|\cF'|}{\delta}}
\\ &\le
    n' + \sum_{t=n'+1}^n \bigl( \ell_{f^*(x_t),t} + \ds1_{f'(x_t) \neq f^*(x_t)} \bigr) + c\sqrt{2n\left(d\ln\frac{en}{d}+ \ln\frac{3}{\delta}\right)}
\end{align*}
where we used $\ell_{i,t}\in [0,1]$ in the penultimate step and the Sauer-Shelah lemma in the last step.
Now, the term $\sum_t \ell_{f^*(x_t),t}$ is controlled in probability w.r.t.\ the random draw of the sample via Chernoff bounds,
\[
    \P\left(\sum_{t=n'+1}^n \ell_{f^*(x_t),t} > (n-n')\,\ell_D(f^*) + \sqrt{\frac{n-n'}{2}\ln\frac{3}{\delta}} \right) \le \delta~.
\]
Hence,
\begin{align*}
    \sum_{t=1}^n \ell_{I_t,t}
&\le
    n' + n\,\ell_D(f^*)
\\ &+ \sum_{t=n'+1}^n \ds1_{f'(x_t) \neq f^*(x_t)} + c\sqrt{2n\left(d\ln\frac{en}{d}+ \ln\frac{3}{\delta}\right)}
\end{align*}
with probability at least $\tfrac{2\delta}{3}$ with respect to both the random sample draw and VE's internal randomization.

The term $\sum_t \ds1_{f'(x_t) \neq f^*(x_t)}$ quantifies the fact that the unique function $f'\in\cF'$ that agrees with $f^*$ on the first $n'$ data points is generally different from $f^*$ on the remaining $n-n'$ points. Since each data point $(x_t,\ell_t)$ is drawn i.i.d., the distribution of a sequence of $n$ pairs remains the same if we randomly permute their positions after drawing them. Hence we can bound
$
    \sum_t \ds1_{f'(x_t) \neq f^*(x_t)}
$
in probability w.r.t.\ a random permutation $\sigma$ of $\{1,\dots,n\}$. Let $\norm{f-g} = \sum_{t=1}^n \ds1_{f'(x_t) \neq f^*(x_t)}$. Then
\begin{align*}
    \P_\sigma&\left(\sum_{t=n'+1}^n \ds1_{f'(x_{\sigma(t)}) \neq f^*(x_{\sigma(t)})} > k\right)
\\ &\le
    \P_\sigma\left( \exists f,g \in \cF,\, \norm{f-g} > k \,:\,  f(x_{\sigma(t)})=g(x_{\sigma(t)}),\, t=1,\dots,n' \right)
\\ &\le
    |\cF|^2 \left(1-\frac{k}{n}\right)^{n'}
\\ &\le
    \left(\frac{en}{d}\right)^{2d} \exp\left(-\frac{kn'}{n}\right)
\\ &\le
    \frac{\delta}{3}
\end{align*}
for
\[
    k \ge \frac{n}{n'}\left(2d\ln\frac{en}{d} + \ln\frac{3}{\delta}\right)~.
\]
Now, since we just proved that
\[
    \sum_{t=n'+1}^n \ds1_{f'(x_{\sigma(t)}) \neq f^*(x_{\sigma(t)})}
\le
    \frac{n}{n'}\left(2d\ln\frac{en}{d} + \ln\frac{3}{\delta}\right)
\]
holds with probability at least $\tfrac{\delta}{3}$ for any sample realization, it holds with the same probability for a random sample. Hence, by choosing $n'$ as in the statement of the theorem and overapproximating, we get the desired result.
\end{proof}

\section{The multiclass case}
\label{s:multiclass}
A different viewpoint on contextual bandits is provided by the so-called bandit multiclass problem. This is a bandit variant of the online protocol for multiclass classification, where the goal is to sequentially learn a mapping from the context space $\R^d$ to the label space $\{1,\dots,K\}$, with $K \ge 2$. In this protocol the learner keeps a classifier parameterized by a $K \times d$ matrix $W$. At each time step $t=1,2,\dots$ the side information $x_t\in\R^d$ is observed (following standard notations in online classification, here we use $x$ instead of $s$ to denote contexts), and the learner predicts the label $\yhat_t$ maximizing $\bigl(Wx_t\bigr)_i$ over all labels $i=1,\dots,K$. In the standard online protocol, the learner observes the true label $y_t$ associated with $x_t$ after each prediction, and uses this information to adjust $W$. In the bandit version, the learner only observes $\ds1_{\yhat_t \neq y_t}$; that is, whether the prediction at time $t$ was correct or not.

A simple but effective learning strategy for (non-bandit) online classification is the multiclass Perceptron algorithm. This algorithm updates $W$ at time $t$ using the rule $W \leftarrow W + X_t$, where
$X_t$ is a $K \times d$ matrix with components $\bigl(X_t\bigr)_{i,j} = x_{t,j}\bigl(\ds1_{y_t=i} - \ds1_{\yhat_t=i}\bigr)$. Therefore, the update rule can be rewritten as
\begin{align*}
    w_{y_t} &\leftarrow w_{y_t} + x_t
\\
    w_{\yhat_t} &\leftarrow w_{\yhat_t} - x_t
\\
    w_i &\leftarrow w_i \qquad\qquad \text{for all $i \neq y_t$ and $i \neq \yhat_t$}
\end{align*}
where $w_i$ denotes the $i$-th row of matrix $W$.
Note, in particular, that no update takes place (i.e., $X_t$ is the all zero matrix) when $\yhat_t = y_t$, which means that $y_t$ is predicted correctly.

A straightforward generalization of the Perceptron analysis gives that, on any sequence of $(x_1,y_1),(x_2,y_2),\ldots\in\R^d\times\{1,\dots,K\}$ such that $\norm{x_t} =1$, the number of classification mistakes satisfies the following notion of regret,
\[
    \sum_{t=1}^n \ds1_{\yhat_t \neq y_t}
\le
    \inf_{U} \left( L_n(U) + 2\norm{U}^2 + \norm{U}\sqrt{2n\bar{L}_n(U)} \right)
\]
uniformly over $n \ge 1$, where the infimum is over all $K \times d$ matrices $U$ and $\norm{\,\cdot\,}$ denotes the Frobenius norm. Here $L_n(U)$ denotes the cumulative \textsl{hinge loss} of policy $U$,
\[
    L_n(U) = \sum_{t=1}^n \ell_t(U) = \sum_{t=1}^n \Bigl[1 - \bigl(Ux_t\bigr)_{y_t} + \max_{i \neq y_t}\bigl(Ux_t\bigr)_i \Bigr]_+
\]
where $[\,\cdot\,]_+ = \max\{0,\,\cdot\,\}$ is the \textsl{hinge} function. Finally, $\bar{L}_n(U) = \tfrac{1}{n}L_n(U)$ is the average hinge loss of $U$.

Note that $\ell_t(U) = 0$ if and only if $\bigl(Ux_t\bigr)_{y_t} \ge 1 + \max_{i \neq y_t}\bigl(Ux_t\bigr)_i$, which can only happen when $y_t = \yhat_t = \argmax_{i}\bigl(Ux_t\bigr)_i$. Moreover, $\ell_t(U) \ge 1$ if and only if $\yhat_t \neq y_t$. This means that $\ell_t$ is a convex upper bound on the mistake indicator function $\ds1_{\yhat_t \neq y_t}$ for the multiclass classifier represented by $U$.

We now introduce a bandit variant of the multiclass Perceptron called Banditron.
\begin{center}
\bookbox{
{\em Banditron}

\medskip\noindent
{Parameter:} number $\gamma\in \bigl(0,\tfrac{1}{2}\bigr)$.\\
Initialize: Set $W_1$ to the zero $K\times d$ matrix.

\smallskip\noindent
For each round $t=1,2,\dots,n$
\begin{itemize}
\item[(1)]
Observe $x_t\in\R^d$.
\item[(2)]
Set ${\dt \yhat_t = \argmax_{i=1,\dots,K} \bigl(W_t\,x_t\bigr)_i }$.
\item[(3)]
Predict $Y_t\in\{1,\dots,K\}$ drawn from distribution $p_t = \bigl(p_{1,t},\dots,p_{K,t}\bigr)$ such that $p_{i,t} = (1-\gamma)\ds1_{\yhat_t = i} + \tfrac{\gamma}{K}$.
\item[(4)]
Observe $\ds1_{Y_t = y_t}$.
\item[(5)]
Update $W_{t+1} = W_t + \tX_t$ where
\[
    (\tX_t)_{i,j} = x_{t,j}\left(\frac{\ds1_{Y_t = y_t}\ds1_{Y_t = i}}{p_{i,t}} - \ds1_{\yhat_t = i}\right)~.
\]
\end{itemize}
}
\end{center}
The Banditron operates in the bandit variant of the online protocol for multiclass classification. As $X_t$ depends on the true label $y_t$, which is only available when the classification is correct, the Banditron computes an estimate $\tX_t$ of $X_t$ via a randomized technique inspired by the nonstochastic multiarmed bandit problem. The label $\yhat_t = \argmax_i \bigl(W\,x_t\bigr)_i$ is used to make the prediction at time $t$ only with probability $1-\gamma$, whereas with probability $\gamma$ a random label is predicted at each time $t$.

We now analyze the expected number of prediction mistakes made by the Banditron algorithm on any sequence of examples $(x_t,y_t)$. Unlike the non-bandit case, where the number of mistakes $M_n$
after $n$ steps of the multiclass Perceptron provides a ``multiclass regret'' bound $M_n - L_n(U) = \mathcal{O}\bigl(\sqrt{n}\bigr)$, in the bandit case the regret achieved by the variant of the Perceptron is only bounded by $\mathcal{O}\bigl(n^{2/3}\bigr)$.
\begin{theorem}
If the Banditron algorithm is run with parameter $\gamma = (K/n)^{1/3}$ on any sequence $(x_1,y_1),\dots,(x_n,y_n) \in \R^d\times\{-1,+1\}$ of examples such that $n \ge 8K$ and $\norm{x_t} = 1$, then the number $M_n$ of prediction mistakes satisfies
\begin{align*}
    \E\,M_n
\le
    \inf_{U}&\biggl( L_n(U) + \left(1 + \norm{U}\sqrt{2\bar{L}_n(U)}\right)K^{1/3}n^{2/3}
\\ &+ 2\norm{U}^2 K^{2/3}n^{1/3} + \sqrt{2}\norm{U} K^{1/6}n^{1/3}\biggr)
\end{align*}
where the infimum is over all $K \times d$ matrices $U$ and $\bar{L}_n(U) = \tfrac{1}{n}L_n(U)$ is the average hinge loss of $U$.
\end{theorem}
\begin{proof}
We need to bound $M = \sum_t \ds1_{Y_t \neq y_t}$. Let $\E_t$ be the expectation conditioned on the first $t-1$ predictions. We start by bounding the first and second moments of $\tX_t$,
\begin{align*}
    \E_t\bigl[(\tX_t)_{i,j}\bigr]
&=
    x_{t,j}\sum_{k=1}^K p_{k,t}\left(\frac{\ds1_{k = y_t}\ds1_{k = i}}{p_{k,t}} - \ds1_{\yhat_t = i}\right)
\\ &=
    x_{t,j}\bigl(\ds1_{y_t=i} - \ds1_{\yhat_t=i}\bigr) = (X_t)_{i,j}~.
\end{align*}
For the second moment, note that
\begin{align*}
    \bigl\|\tX_t\bigr\|^2
&=
    \sum_{i=1}^K\sum_{j=1}^d x_{t,j}^2\left(\frac{\ds1_{Y_t = y_t}\ds1_{Y_t = i}}{p_{i,t}} - \ds1_{\yhat_t = i}\right)^2
\\ &=
    \sum_{i=1}^K\left(\frac{\ds1_{Y_t = y_t}\ds1_{Y_t = i}}{p_{i,t}} - \ds1_{\yhat_t = i}\right)^2
\end{align*}
where
\[
    \sum_{i=1}^K\left(\frac{\ds1_{Y_t = y_t}\ds1_{Y_t = i}}{p_{i,t}} - \ds1_{\yhat_t = i}\right)^2
=
    \left\{ \begin{array}{cl}
    {\dt \left(\frac{1}{p_{y_t,t}^2} + 1\right) } & \text{if $Y_t = y_t \neq \yhat_t$}
\\[2mm]
    {\dt \left(\frac{1}{p_{y_t,t}} - 1\right)^2 } & \text{if $Y_t = y_t = \yhat_t$}
\\[2mm]
    1 & \text{otherwise.}
    \end{array} \right.
\]
Therefore, if $y_t \neq \yhat_t$,
\begin{align*}
    \E_t\bigl\|\tX_t\bigr\|^2
&=
    p_{y_t,t}\left(\frac{1}{p_{y_t,t}^2} + 1\right) + \bigl(1-p_{y_t,t}\bigr)
\\&=
    1 + \frac{1}{p_{y,t}}
=
    1 + \frac{K}{\gamma}
\le
    \frac{2K}{\gamma}
\end{align*}
because $p_{i,t} = \gamma$ when $y_t \neq \yhat_t$.
Otherwise, if $y_t = \yhat_t$
\begin{align*}
    \E_t\bigl\|\tX_t\bigr\|^2
&=
    p_{y_t,t}\left(\frac{1}{p_{y_t,t}} - 1\right)^2 + \bigl(1-p_{y_t,t}\bigr)
\\&=
    \frac{1}{p_{y_t,t}} - 1
=
    \frac{1}{1-\gamma} - 1
\le
    2\gamma
\end{align*}
because $p_{i,t} = 1-\gamma$ when $y_t = \yhat_t$ and $\gamma \le \tfrac{1}{2}$.
Hence,
\[
    \E_t\bigl\|\tX_t\bigr\|^2
\le
    2\left(\frac{K}{\gamma}\ds1_{y_t \neq \yhat_t} + \gamma\ds1_{y_t = \yhat_t}\right)~.
\]
We are now ready to prove a bound on the expected number of mistakes. Following the standard analysis for the Perceptron algorithm, we derive upper and lower bounds on the expectation of the quantity $\inner{U}{W_{n+1}} = \mathrm{tr}\bigl(U\,W_{n+1}^{\top}\bigr)$, for an arbitrary $K \times d$ matrix $U$.
First, using Cauchy-Schwartz and Jensen inequalities we obtain
\begin{align*}
    \E\inner{U}{W_{n+1}}
\le
    \sqrt{\norm{U}^2\E\norm{W_{n+1}}^2}~.
\end{align*}
Now
\begin{align*}
    \E_n\bigl[\norm{W_{n+1}}^2\bigr]
&=
    \E_n\left[\norm{W_n}^2 + 2\inner{W_n}{\tX_n} + \bigl\|\tX_n\bigr\|^2\right]
\\ &\le
    \norm{W_n}^2 + \E_n\bigl\|\tX_n\bigr\|^2~.
\end{align*}
In order to see why the inequality holds, note that
\begin{align*}
    \E_n\inner{W_n}{\tX_n}
=
    \inner{W_n}{X_n}
&=
    \sum_{i=1}^K \bigl(W_n\,x_t\bigr)_i \bigl(\ds1_{y_n=i} - \ds1_{\yhat_n=i}\bigr)
\\ &=
    \bigl(W_n\,x_n\bigr)_{y_n} - \bigl(W_n\,x_n\bigr)_{\yhat_n} \le 0
\end{align*}
because $\yhat_n = \argmax_{i=1,\dots,K} \bigl(W_n\,x_n\bigr)_i$ by definition.
Therefore, since $W_1$ is the zero matrix,
\begin{align*}
    \E\norm{W_{n+1}}^2
&\le
    \sum_{t=1}^n \E\bigl\|\tX_n\bigr\|^2
\\ &\le
    2\sum_{t=1}^n \left(\frac{K}{\gamma}\P\bigl(y_t \neq \yhat_t\bigr) + \gamma\,\P\bigl(y_t = \yhat_t\bigr)\right)
\\ &\le
\frac{2K}{\gamma}\sum_{t=1}^n \P\bigl(y_t \neq \yhat_t\bigr) + 2\gamma\,n~.
\end{align*}
Thus we have
\[
    \E\inner{U}{W_{n+1}} \le \norm{U}\sqrt{\frac{2K}{\gamma}\sum_{t=1}^n \P\bigl(y_t \neq \yhat_t\bigr) + 2\gamma\,n}~.
\]
Now we lower bound $\inner{U}{W_{n+1}}$ as follows,
\begin{align*}
    \E_n\inner{U}{W_{n+1}}
&=
    \E_n\inner{U}{W_n + \tX_n}
\\ &=
    \inner{U}{W_n} + \inner{U}{X_n}
\\ &\ge
    \inner{U}{W_n} + \ds1_{y_t \neq \yhat_t} - \ell_t(U)
\end{align*}
because, by definition of $\ell_t$,
\begin{align*}
    \ell_t(U)
&=
    \Bigl[1 - \bigl(Ux_t\bigr)_{y_t} + \max_{i \neq y_t}\bigl(Ux_t\bigr)_i \Bigr]_+
\\ &\ge
    1 - \bigl(Ux_t\bigr)_{y_t} + \bigl(Ux_t\bigr)_{\yhat_t}
\\ &\ge
    \ds1_{y_t \neq \yhat_t} - \bigl(Ux_t\bigr)_{y_t} + \bigl(Ux_t\bigr)_{\yhat_t}
\\ &=
    \ds1_{y_t \neq \yhat_t} - \inner{U}{X_t}~.
\end{align*}
Therefore, using again the fact that $W_1$ is the zero matrix,
\[
    \E\inner{U}{W_{n+1}} \ge \sum_{t=1}^n \P\bigl(y_t \neq \yhat_t\bigr) - \sum_{t=1}^n \ell_t(U)~.
\]
Combining the upper and lower bounds on $\inner{U}{W_{n+1}}$ we get
\[
    \sum_{t=1}^n \P\bigl(y_t \neq \yhat_t\bigr) - L_n(U)
\le
    \norm{U}\sqrt{\frac{2K}{\gamma}\sum_{t=1}^n \P\bigl(y_t \neq \yhat_t\bigr) + 2\gamma\,n}~.
\]
Solving for $\sum_t \P\bigl(y_t \neq \yhat_t\bigr)$ and overapproximating yields
\begin{align*}
    \sum_{t=1}^n \P\bigl(y_t \neq \yhat_t\bigr)
&\le
    L_n(U) + \frac{2K}{\gamma}\norm{U}^2 + \norm{U}\sqrt{\frac{2K}{\gamma} L_n(U) + 2\gamma\,n}
\\ &=
    L_n(U) + \frac{2K}{\gamma}\norm{U}^2 + \norm{U}\sqrt{\left(\frac{2K}{\gamma}\,\bar{L}_n(U) + 2\gamma\right)n}~.
\end{align*}
Now, since $\P\bigl(y_t \neq Y_t\bigr) = (1-\gamma)\P\bigl(y_t \neq \yhat_t\bigr) + \gamma$,
\[
    \sum_{t=1}^n \P\bigl(y_t \neq Y_t\bigr)
\le
    L_n(U) + \gamma\,n + \frac{2K}{\gamma}\norm{U}^2 + \norm{U}\sqrt{\left(\frac{2K}{\gamma}\,\bar{L}_n(U) + 2\gamma\right)n}~.
\]
Choosing $\gamma$ as in the statement of the theorem yields the desired result.
Note that $\gamma \le \tfrac{1}{2}$ because we assume $n \ge 8K$.
\end{proof}

\section{Bibliographic remarks} \label{sec:contextual}
A model of contextual stochastic bandits close to those studied here is introduced by~\cite{wang2005bandit}. The context is provided by a i.i.d.\ sequence of random variables and the rewards for each arm depend on the context through an unknown parametric model beloging to a known class. This result has been generalized to the non i.i.d.\ case by~\cite{wang2005arbitrary}, to the multivariate linear case by~\cite{RT10}, and to the multivariate and nonparametric case by~\cite{PR11}. In~\cite{slivkins2009contextual}, contexts and rewards belong to arbitrary metric spaces, and the unknown function mapping contexts to rewards satisfies a Lipschitz assumption (remarkably, the same algorithm also applies to slowly changing expected rewards and sleeping bandit settings). The case of deterministic covariates (fixed design), finitely many arms, and a linear stochastic dependence between covariates and rewards has been studied in~\cite{Aue02,chu2011contextual} ---see also~\cite{abe1999associative}. The work of~\cite{filippi2010parametric} extends the analysis of fixed design by assuming a generalized linear model to capture the dependence of rewards on covariates.

The simple contextual model analyzed in Section~\ref{s:context-intro}, as well as its extension described in Subsection~\ref{ss:best-context}, are due to~\cite{maillard2011adaptive}. The Exp4 algorithm for the adversarial case was introduced in~\cite{ACFS03}. Subsequent improvements were proposed in the two papers~\cite{beygelzimer2011contextual} (Exp4.P with high-probability bounds) and~\cite{mcmahan2009tighter} (exploitation of correlations in expert advice). The VE algorithm and its analysis in Section~\ref{s:context-stochastic} are also taken from~\cite{beygelzimer2011contextual}.

A stochastic model for contextual bandits with finitely many arms and finitely many states has been investigated by~\cite{NIPS2011_0948} using new sophisticated tools of PAC-Bayesian analysis.

The general stochastic model of Section~\ref{s:context-stochastic} for contextual bandits with finitely many arms is due to~\cite{langford2007epoch}. An efficient algorithm for this model has been recently proposed in~\cite{dudik2011efficient}.

The bandit multiclass model of Section~\ref{s:multiclass} is due to~\cite{langford2007epoch}. The Banditron algorithm and its analysis are from~\cite{kakade2008efficient}. See also~\cite{crammer2011multiclass} and~\cite{hazan2011newtron} for recent variants and improvements.

%% file: linear.tex
We now consider the important generalization of adversarial bandits where the set of arms $\{1,\hdots,K\}$ is replaced by a compact set $\cK \subset \R^d$. In this case, the loss at each step is some function defined on $\cK$, and the task is to pick an arm as close as possible to the minimum of the loss function at hand. In order to allow sublinear regret bounds, even in the presence of infinitely many arms, we must assume some structure for the loss function. In particular, in this chapter we assume that the loss at each time step is a linear function of arms. Linearity is a standard assumption (think, for instance, of linear regression) and naturally occurs in many bandit applications. The source routing problem mentioned in the introduction is a good example, since the cost of choosing a routing path is linear in the cost of the edges that make up the path. This defines the so-called online linear optimization setting: at each time step $t=1,2, \hdots$ the forecaster chooses $x_t \in \cK$ while, simultaneously, the adversary chooses $\ell_t$ from some fixed and known subset $\cL$ of $\R^d$. The loss incurred by the forecaster is the inner product $x_t^{\top}\ell_t$. In this chapter we focus on the analysis of the pseudo-regret
$$\oR_n = \E \sum_{t=1}^n x_t^{\top} \ell_t - \min_{x \in \cK} \E \sum_{t=1}^n x^{\top} \ell_t$$
where the expectation is with respect to the forecaster's internal randomization.
The adversarial bandit setting of Chapter~\ref{adversarial} is obtained by choosing $\cK=\{e_1, \hdots, e_d\}$, where $e_1, \hdots, e_d$ is the canonical basis of $\R^d$, and $\cL=[0,1]^d$. Similarly to Chapter \ref{adversarial}, we focus on the bandit feedback where the forecaster only observes the incurred loss $x_t^{\top} \ell_t$ at the end round $t$. However, we also discuss the full information setting, where the complete loss vector $\ell_t$ is revealed at the end of each round $t$, as well as other feedback models.

It is important to note that, without any loss of generality (as far as regret bounds are concerned), one can always assume that $\cK$ has size $\scO(n^d)$. Indeed, since $\cK$ is a compact set and the loss is linear (and therefore Lipschitz), one can cover $\cK$ with $\scO(n^d)$ points such that one incurs a vanishing extra cumulative regret by playing on the discretization rather than on the original set. Of course, the downside of this approach is that a strategy resulting from such a cover is often not computationally efficient. On the other hand, this assumption allows us to apply ideas from Chapter~\ref{adversarial} to this more general setting. We analyze the pseudo-regret for finite classes in Section \ref{sec:Exp2}. Without loss of generality, it is also possible to assume that $\cK$ is a convex set. Indeed, the pseudo-regret is the same if one plays $x_t$, or if one plays a point at random in $\cK$ such that the expectation of the played point is $x_t$. This remark is critical, and allows us to develop a powerful technology known as the Mirror Descent algorithm. We describe this approach in Section \ref{sec:OMD}, and explore it further in subsequent sections.

\section{Exp2 (Expanded Exp) with John's exploration} \label{sec:Exp2}
In this section we work under the {\em bounded scalar loss} assumption. That is, we assume that $\cK$ and $\cL$ are such that $|x^{\top} \ell| \leq 1$ for any $(x, \ell) \in \cK \times \cL$. Moreover, we restrict our attention to finite sets $\cK$, with $|\cK| = N$. Without loss of generality we assume that $\cK$ spans $\R^d$. If it is not the case, then one can rewrite the elements of $\cK$ in some lower dimensional vector space, and work there. Note that a trivial application of Exp3 to the set $\cK$ of arms gives a bound that scales as $\sqrt{n N\ln N}$. If $\cK$ is a discretized set (in the sense described earlier), then $N$ is exponential in $d$. We show here that, by appropriately modifying Exp3, one can obtain a polynomial regret of order $\sqrt{n d \ln N}$.

To describe the strategy, we first need a useful result from convex geometry: John's theorem. This result concerns the ellipsoid $\cE$ of minimal volume enclosing a given convex set $\cK$ (which we call the John's ellipsoid of $\cK$). Basically, the theorem states that $\cE$ has many contact points with $\cK$, and these contact points are ``nicely'' distributed, in the sense that they almost form an orthonormal basis ---see \cite{Bal97} for a proof.
\begin{theorem}[John's theorem] \label{th:john}
Let $\cK \subset \R^d$ be a convex set. If the ellipsoid $\cE$ of minimal volume enclosing $\cK$ is the unit ball in some norm derived from a scalar product $\langle \cdot, \cdot \rangle$, then there exist $M \leq \tfrac{1}{2}d(d+1) + 1$ contact points $u_1, \hdots, u_M$ between $\cE$ and $\cK$, and a probability distribution $(\mu_1,\dots,\mu_M)$ over these contact points, such that
$$x = d \sum_{i=1}^M \mu_i \langle x, u_i \rangle u_i \qquad \forall x \in \R^d .$$
\end{theorem}
In fact John's theorem is a {\em if and only if}, but here we only need the direction stated in the theorem. 
We are now in position to describe the strategy. Let $\conv(S)$ be the convex hull of a set $S\in\R^d$. First, we perform a preprocessing step in which the set $\cK$ is rewritten so that John's ellipsoid of $\conv(\cK)$ is the unit ball for some inner product $\langle \cdot, \cdot \rangle$. The loss of playing $x \in \cK$ against $\ell \in \cL$ is then given by $\langle x, \ell \rangle$. See \cite{BCK12} for the details of this transformation. Let $u_1, \hdots, u_M \in \cK$ and $(\mu_1,\dots,\mu_M)$ satisfy Theorem \ref{th:john} for the convex set $\conv(\cK)$.

Recall from Chapter \ref{adversarial} that the key idea to play against an adversary is to select $x_t$ at random from some probability distribution $p_t$ over $\cK$. We first describe how to build an unbiased estimate of $\ell_t$, given such a point $x_t$ played at random from $p_t$ (such that $p_t(x) > 0$  for any $x \in \cK$). Recall that the outer product $x \otimes x$ is defined as the linear mapping from $\R^d$ to $\R^d$ such that $x \otimes x\, (y) = \langle x, y \rangle\, x$. Note that one can also view $x \otimes x$ as a $d \times d$ matrix, so that the evaluation of $x \otimes x$ is equivalent to a multiplication by the corresponding matrix. Now let
$$P_t = \sum_{x \in \cK} {p}_t(x)\, x \otimes x~.$$
Note that this matrix is invertible, since $\cK$ is full rank and $p_t(x) > 0$ for all $x \in \cK$. The estimate for $\ell_t$ is given by
$
\tilde{\ell}_t = P_t^{-1} \left(x_t \otimes x_t\right) \ell_t
$.
Note that this is a valid estimate since $\left(x_t \otimes x_t\right) \ell_t = \langle x_t, \ell_t \rangle x_t$ and $P_t^{-1}$ are observed quantities. Also, it is clearly an unbiased estimate.
 
Now the Exp2 algorithm with John's exploration corresponds to playing according to the following probability distribution
\begin{equation}
\label{eq:exp2def}
p_t(x) = (1- \gamma) \frac{\exp\left(- \eta \sum_{s=1}^{t-1} \langle x, \tilde{\ell}_t \rangle  \right)}{\sum_{y \in \cK} \exp\left(- \eta \sum_{s=1}^{t-1} \langle y, \tilde{\ell}_t \rangle  \right)} + \gamma \sum_{i=1}^M \mu_i\,\ds1_{x = u_i}
\end{equation}
where $\eta,\gamma > 0$ are input parameters. Note that this corresponds to a variant of Exp3 using $\langle x, \tilde{\ell}_t \rangle$ as loss estimate for $x \in \cK$, and an exploration distribution supported by the contact points.
\begin{theorem}[Pseudo-regret of Exp2 with John's exploration] \label{th:exp2john}
For any $\eta,\gamma > 0$ such that ${\eta d} \le {\gamma}$, Exp2 with John's exploration satisfies
$$\oR_n \leq 2 \gamma n + \frac{\ln N}{\eta} + \eta n d~.$$
In particular, with $\eta = \sqrt{\frac{\ln N}{3 n d}}$ and $\gamma = \eta d$,
$$\oR_n \leq 2 \sqrt{3 n d \ln N}~.$$
\end{theorem}
\begin{proof}
Since $\cK$ is finite, we can easily adapt the analysis of Exp3 in Theorem~\ref{th:Exp3} to take into account the exploration term.
This gives
$$\oR_n \leq 2 \gamma n + \frac{\ln N}{\eta} + \eta\,\E \sum_{t=1}^n \sum_{x \in \cK} {p}_t(x) \langle x, \tilde{\ell}_t \rangle^2
$$
whenever
$\eta\,\langle x, \tilde{\ell}_t \rangle \leq 1$ for all $x \in \cK$. We now bound the last term in the right-hand side of the above inequality. Using the definition of the estimated loss $\tilde{\ell}_t = P_t^{-1} \left(x_t \otimes x_t\right) \ell_t$, we can write
\begin{align*}
    \sum_{x \in \cK} {p}_t(x) \langle x, \tilde{\ell}_t \rangle^2
&=
    \sum_{x \in \cK} {p}_t(x) \langle \tilde{\ell}_t, (x \otimes x)\,\tilde{\ell}_t \rangle
\\& =
    \langle \tilde{\ell}_t, P_t\,\tilde{\ell}_t \rangle 
\\&=
    \langle x_t, \ell_t \rangle^2 \langle P_t^{-1} x_t, P_t\, P_t^{-1} x_t\rangle
\\ &\le
    \langle P_t^{-1} x_t, x_t \rangle~.
\end{align*}
Now we use a spectral decomposition of $P_t$ in an orthonormal basis for $\langle \cdot, \cdot \rangle$ and write
$P_t = \sum_{i=1}^d \lambda_i v_i \otimes v_i $. In particular, we have $P_t^{-1} = \sum_{i=1}^d \frac{1}{\lambda_i} v_i \otimes v_i$ and thus
\begin{align*}
\E \langle P_t^{-1} x_t, x_t \rangle & = \sum_{i=1}^d \frac{1}{\lambda_i} \E \langle (v_i \otimes v_i)\, x_t, x_t \rangle \\
& = \sum_{i=1}^d \frac{1}{\lambda_i} \E \langle (x_t \otimes x_t)\, v_i, v_i \rangle \\
& = \sum_{i=1}^d \frac{1}{\lambda_i} \langle P_t\, v_i, v_i \rangle \\
& =  \sum_{i=1}^d \frac{1}{\lambda_i} \langle \lambda_i v_i, v_i \rangle
 =  d.
\end{align*}
Finally, to show $\eta\,\langle x, \tilde{\ell}_t \rangle \le 1$ observe that
$$
\langle x, \tilde{\ell}_t \rangle  =  \langle x_t, \ell_t \rangle \langle x, P_t^{-1} x_t \rangle
 \leq  \langle x, P_t^{-1} x_t \rangle
 \leq  \frac{1}{\min_{1 \leq i \leq d} \lambda_i} ,$$
where the last inequality follows from the fact that $\langle x, x \rangle \leq 1$ for any $x \in \cK$, since $\cK$ is included in the unit ball. To conclude the proof, we need to lower bound the smallest eigenvalue of $P_t$. Using Theorem~\ref{th:john}, one can see that $P_t \succeq \frac{\gamma}{d} I_d$, and thus $\lambda_i \geq \frac{\gamma}{d}$. Since ${\eta d} \le {\gamma}$, the proof is concluded.
\end{proof}

\section{Online Mirror Descent (OMD)} \label{sec:OMD}
We now introduce the Online Mirror Descent (OMD) algorithm, a powerful generalization of gradient descent for sequential decision problems. We start by describing OMD for convex losses in the full information setting. That is, $\cL$ is a set of convex functions, and at the end of round $t$ the forecaster observes $\ell_t\in\cL$ rather than only $\ell_t(x_t)$. 

The rest of this chapter is organized as follows. Next, we briefly recall a few key concepts from convex analysis. Then we describe the OMD strategy and prove a general regret bound. In Section \ref{sec:OSMD} we introduce Online Stochastic Mirror Descent (OMSD), which is a variant of OMD based on a stochastic estimate of the gradient. We apply this strategy to linear losses in two different bandit settings. Finally, in Section~\ref{sec:ball} we show how OMSD obtains improved bounds in certain special cases. The case of convex losses with bandit feedback is treated in Chapter \ref{nonlinear}.

We introduce the following definitions.
\begin{definition}
Let $\cX\subseteq\R^d$.
A function $f : \cX \to \R$ is subdifferentiable if for all $x \in \cX$ there exists $g \in \R^d$ such that
$$f(x) - f(y) \leq g^{\top} (x-y) \qquad\forall y \in \cX~.$$
Such a $g$ is called a subgradient of $f$ at $x$.
\end{definition}
Abusing notation, we use $\nabla f(x)$ to denote both the gradient of $f$ at $x$ when $f$ is differentiable, and any subgradient of $f$ at $x$ when $f$ is subdifferentiable (a sufficient condition for subdifferentiability of $f$ is that $f$ is convex and $\cX$ is open).
\begin{definition}
Let $f : \cX \rightarrow \R$ be a convex function defined on a convex set $\cX\subseteq\R^d$. The Legendre-Fenchel transform of $f$ is defined by:
$$f^*(u) = \sup_{x \in \cX} \left(x^{\top} u - f(x)\right)~.$$
\end{definition}
This definition directly implies the Fenchel-Young inequality for convex functions,
$u^{\top} x \leq f(x) + f^*(u)$.

Let $\cD \subset \R^d$ be an open convex set, and let $\oD$ be the closure of $\cD$. 
\begin{definition}
A continuous function $F:\oD\to\R$ is Legendre if 
\begin{itemize}
\item[(i)] $F$ is strictly convex and admits continuous first partial
derivatives on $\cD$;
\item[(ii)] ${\displaystyle \lim_{x \to \oD \setminus \cD} \norm{\nabla F(x)} = +\infty.}$\footnote{By the equivalence of norms in $\R^d$, this definition does not depend on the choice of the norm.}
\end{itemize}
\end{definition}
The Bregman divergence $D_F: \oD\times\cD \to \R$ associated with a Legendre function $F$ is defined by
  $
  D_F(x,y) = F(x) - F(y) - (x-y)^{\top} \nabla F(y)
  $. 
Moreover, we say that $\cD^* = \nabla F (\cD)$ is the dual space of $\cD$ under $F$.
Note that, by definition, $D_F(x,y) > 0$ if $x \neq y$, and $D_F(x,x)=0$. The following lemma is the key to understand how a Legendre function act on the space. See~\cite[Proposition~11.1]{CL06} for a proof.
\begin{lemma}
Let $F$ be a Legendre function. Then $F^{**} = F$, and $\nabla F^* = (\nabla F)^{-1}$ restricted on the set $\cD^*$. Moreover, for all $x, y \in \cD$,
\begin{equation} \label{eq:transrelation}
D_F(x,y) = D_{F^*}\bigl(\nabla F(y), \nabla F(x)\bigr)~.
\end{equation}
\end{lemma}
The gradient $\nabla F$ maps $\cD$ to the dual space $\cD^*$, and $\nabla F^*$ is the inverse mapping from the dual space to the original (primal) space. Note also that~\eqref{eq:transrelation} shows that the Bregman divergence in the primal exactly corresponds to the Bregman divergence of the Legendre-transform in the dual.

The next lemma shows that the geometry induced by a Bregman divergence resembles to the geometry of the squared Euclidean distance. See~\cite[Lemma~11.3]{CL06} for a proof.
\begin{lemma}[Generalized Pythagorean inequality] \label{lem:proj}
Let $\cK \subseteq \oD$ be a closed convex set such that $\cK \cap \cD \neq \emptyset$. Then, for all $x \in \cD$, the Bregman projection
$$z = \argmin_{y \in \cK} D_F(y,x)$$
exists and is unique. Moreover, for all $z \in \cK \cap \cD$ and $y \in \cK$,
$$D_F(y,x) \geq D_F(y,z) + D_F(z,x)~.$$
\end{lemma}
%
The idea of OMD is very simple: first, select a Legendre function $F$ on $\oD \supset \cK$ (such that $\cK \cap \cD \neq \emptyset$); second, perform a gradient descent step, where the update with the gradient is performed in the dual space $\cD^*$ rather than in the primal $\cD$; third, project back to $\cK$ according to the Bregman divergence defined by $F$.
\bookbox{
{\em OMD (Online Mirror Descent):}

\medskip\noindent
Parameters: compact and convex set $\cK\subseteq\R^d$, learning rate $\eta > 0$, Legendre function $F$ on $\oD \supset \cK$.\\
Initialize: ${\displaystyle x_1 \in \argmin_{x \in \cK} F(x)}$ (note that $x_1 \in \cK \cap \cD$).

\smallskip\noindent
For each round $t=1,2,\dots,n$
\begin{itemize}
\item[(1)] Play $x_t$ and observe loss vector $\ell_t$.
\item[(2)] $w_{t+1} = \nabla F^*\Bigl(\nabla F(x_t) - \eta \nabla \ell_t(x_t) \Bigr)$.
\item[(3)] ${\displaystyle x_{t+1} = \argmin_{y \in \cK} D_F(y, w_{t+1}) }$.
\end{itemize}
}

Note that step~(2) is well defined if the following consistency condition is satisfied:
\begin{equation} \label{eq:consistency}
\nabla F(x) - \eta \nabla \ell(x) \in \cD^* \qquad \forall (x,\ell) \in \bigl(\cK \cap \cD\bigr) \times \cL~. 
\end{equation}
Note also that step~(2) can be rewritten as
\begin{equation} \label{eq:MGD3}
\nabla F(w_{t+1}) = \nabla F(x_t) - \eta \nabla \ell_t(x_t)~.
\end{equation}
Finally, note that the Bregman projection in step~(3) is a convex program, in the sense that $x \mapsto D_F(x,y)$ is always a convex function. This does not necessarily imply that step~(3) can be performed efficiently, since in some cases the feasible set $\cK$ might only be described with an exponential number of constraints (we encounter examples like this in Section~\ref{sec:semibandit}).

In the description above we emphasized that $F$ has to be a Legendre function. In fact, as we see in Theorem~\ref{th:Fnotlegendre}, if $F$ has effective domain $\cK$ (that is, $F$ takes value $+\infty$ outside of $\cK$), then it suffices that the Legendre-Fenchel dual $F^*$ is differentiable on $\R^d$ to obtain a good regret bound. See the end of this section for more details on this.

When $\cK$ is the simplex and $F(x) = \sum_{i=1}^d x_i \ln x_i - \sum_{i=1}^d x_i$, OMD reduces to 
an exponentially weighted average forecaster, similar to those studied in Chapter~\ref{adversarial}. The well-known online gradient descent strategy corresponds to taking $F(x) = \frac12 \norm{x}^2_2$. In the following we shall see other possibilities for the Legendre function $F$.

We prove now a very general and powerful theorem concerning the regret of OMD.
\begin{theorem}[Regret of OMD with a Legendre function] \label{th:MGD}
Let $\cK$ be a compact and convex set of arms, $\cL$ be a set of subdifferentiable functions, and $F$ a Legendre function defined on $\oD \supset \cK$, such that \eqref{eq:consistency} is satisfied. Then OMD satisfies for any $x \in \cK$,
\begin{align*}
\sum_{t=1}^n \ell_t(x_t) - \sum_{t=1}^n \ell_t(x)
& \leq \frac{F(x) - F(x_1)}{\eta} \\
&  \quad + \frac{1}{\eta} \sum_{t=1}^n D_{F^*}\bigg(\nabla F(x_t) - \eta \nabla \ell_t(x_t), \nabla F(x_t)\bigg)~.
\end{align*}
\end{theorem}
\begin{proof}
Let $x \in \cK$. Since $\cL$ is a set of subdifferentiable functions, we have:
$$\sum_{t=1}^n \bigl( \ell_t(x_t) - \ell_t(x) \bigr) \leq \sum_{t=1}^n \nabla \ell_t(x_t)^{\top} (x_t -x)~.$$
Using \eqref{eq:MGD3}, and applying the definition of $D_F$, one obtains
  \begin{align*}
  \eta \nabla \ell_t(x_t)^{\top} (x_t - x) & = (x-x_t)^{\top} \big(\nabla F(w_{t+1}) - \nabla F(x_t) \big)\\
  & = D_F(x,x_t)+D_F(x_t,w_{t+1})-D_F(x,w_{t+1})~.
  \end{align*}
By Lemma \ref{lem:proj}, one gets 
  $D_F(x,w_{t+1}) \ge D_F(x,x_{t+1}) + D_F(x_{t+1},w_{t+1}),$ hence
  \begin{align*}
  \eta \nabla \ell_t(x_t)^{\top} (x_t - x) & \le D_F(x,x_t)+D_F(x_t,w_{t+1}) \\
& \quad -D_F(x,x_{t+1})-D_F(x_{t+1},w_{t+1})~.
  \end{align*}
Summing over $t$ then gives  
  \begin{align*}
  \sum_{t=1}^n \eta \nabla \ell_t(x_t)^{\top} (x_t - x) & \le D_F(x,x_1)-D_F(x,x_{n+1}) \\
& \quad +\sum_{t=1}^n \big(D_F(x_t,w_{t+1}) -D_F(x_{t+1},w_{t+1})\big)~.
  \end{align*}
By the nonnegativity of the Bregman divergences, we get 
$$\sum_{t=1}^n \eta \nabla \ell_t(x_t)^{\top} (x_t - x)
\le D_F(x,x_1)+\sum_{t=1}^n D_F(x_t,w_{t+1})~.$$
From \eqref{eq:transrelation}, one has
  $
  D_F(x_t,w_{t+1})=D_{F^*}\big(\nabla F(x_t)- \eta \nabla \ell_t(x_t),\nabla F(a_t)\big)
  $
and, moreover, by first-order optimality, one has $D_F(x,x_1) \leq F(x) - F(x_1)$, which concludes the proof.
\end{proof}
We show now how to prove a regret bound if $F$ has effective domain $\cK$ and $F^*$ is differentiable on $\R^d$, but not necessarily Legendre. In this case, it is easy to see that step~(1) and step~(2) in OMD reduce to
\[
x_{t+1} = \nabla F^* \left( - \eta \sum_{s=1}^{t-1} \nabla \ell_s(x_s) \right)~.
\]
\begin{theorem}[Regret of OMD with non-Legendre function] \label{th:Fnotlegendre}
Let $\cK$ be a compact set of actions, $\cL$ be a set of subdifferentiable functions, and $F$ a function with effective domain $\cK$ such that $F^*$ is differentiable on $\R^d$. Then for any $x \in \cK$ OMD satisfies
\begin{align*}
\sum_{t=1}^n \ell_t(x_t) - \sum_{t=1}^n \ell_t(x)
& \leq \frac{F(x) - F(x_1)}{\eta} \\
& + \frac{1}{\eta} \sum_{t=1}^n D_{F^*}\bigg(- \eta \sum_{s=1}^{t} \nabla \ell_s(x_s), - \eta \sum_{s=1}^{t-1} \nabla \ell_s(x_s) \bigg).
\end{align*}
\end{theorem}
\begin{proof}
Let $x \in \cK$. Since $\cL$ is a set of subdifferentiable functions, we have
$$\sum_{t=1}^n \bigl( \ell_t(x_t) - \ell_t(x) \bigr) \leq \sum_{t=1}^n \nabla \ell_t(x_t)^{\top} (x_t -x)~.$$ Using the Fenchel-Young inequality, one obtains
\begin{align*}
- \eta \sum_{t=1}^n &\nabla \ell_t(x_t)^{\top} x
\leq
    F(x) + F^*\left( - \eta \sum_{t=1}^{n} \nabla \ell_t(x_t) \right)
\\&=
    F(a) + F^*(0) \\ & \quad + \sum_{t=1}^n \left( F^*\left( - \eta \sum_{s=1}^{t} \nabla \ell_s(x_s)\right)  - F^*\left( - \eta \sum_{s=1}^{t-1} \nabla \ell_s(x_s)\right) \right)~.
\end{align*}
Observe that $F^*(0)= - F(x_1)$ and, for each term in the above sum,
\begin{align*}
    \nabla F^*&\left( - \eta \sum_{s=1}^{t-1} \nabla \ell_s(x_s) \right)^{\top} \Bigl(- \eta \nabla \ell_t(x_t)\Bigr)
\\& \quad
    + D_{F^*}\left(- \eta \sum_{s=1}^{t} \nabla \ell_s(x_s), - \eta \sum_{s=1}^{t-1} \nabla \ell_s(x_s)\right)
\\ &=
    - \eta x_t^{\top} \nabla \ell_t(x_t) + D_{F^*}\left(- \eta \sum_{s=1}^{t} \nabla \ell_s(x_s), - \eta \sum_{s=1}^{t-1} \nabla \ell_s(x_s) \right)~.
\end{align*}
\end{proof}

\section{Online Stochastic Mirror Descent (OSMD)} \label{sec:OSMD}
We now turn to the analysis of the bandit setting, where the gradient information $\nabla \ell_t(x_t)$ is not available, and thus one cannot run OMD. This scenario has been extensively in gradient-free optimization, and the basic idea is to play a perturbed version $\tilde{x}_t$ of the current point $x_t$. This should be done in such a way that, upon observing $\ell_t(\tilde{x}_t)$, one can build an unbiased estimate $\tilde{g}_t$ of $\nabla \ell_t(x_t)$. Whereas such estimators are presented in Chapter~\ref{nonlinear}, here we restrict our attention to linear losses. For this simpler case we consider specialized estimators with optimal performances.
Given a perturbation scheme, one can run OMD with the gradient estimates instead of the true gradients. We call Online Stochastic Mirror Descent (OSMD) the resulting algorithm.
\bookbox{
{\em OSMD (Online Stochastic Mirror Descent):}

\smallskip\noindent
Parameters: compact and convex set $\cK\subseteq\R^d$, learning rate $\eta > 0$, Legendre function $F$ on $\oD \supset \cK$.\\
Initialize: ${\displaystyle x_1 \in \argmin_{x \in \cK} F(x)}$ (note that $x_1 \in \cK \cap \cD$).

\smallskip\noindent
For each round $t=1,2,\dots,n$
\begin{itemize}
\vspace{-2mm}
\item[(1)] Play a random perturbation $\tilde{x}_t$ of $x_t$ and observe $\ell_t\bigl(\tilde{x}_t\bigr)$
\item[(2)] Compute random estimate $\tilde{g}_t$ of $\nabla \ell_t(x_t)$
\item[(3)] $w_{t+1} = \nabla F^*\Bigl(\nabla F(x_t) - \eta \tilde{g}_t \Bigr)$
\item[(4)] ${\displaystyle x_{t+1} = \argmin_{y \in \cK} D_F(y, w_{t+1}) }$
\end{itemize}
}

In order to relate this linear bandit strategy to the Exp2 forecaster~(\ref{eq:exp2def}), it is important to observe that running the Exp2 forecaster over a finite set $\cK$ of arms, with exploration distribution $\mu$ and mixing coefficient $\gamma > 0$, is equivalent to running OSMD over the $|\cK|$-dimensional simplex with $F(x) = \tfrac{1}{\eta}\sum_{x\in\cK} x\ln x$ (the negative entropy), $\tilde{x}_t$ drawn from $(1-\gamma)x_t + \gamma\,\mu$, and estimated linear loss $\tilde{g}_t = \bigl(\langle x, \tilde{\ell}_t \rangle\bigr)_{x\in\cK}$. Indeed, the projection step (4), when $F$ is the negative entropy, corresponds to the standard normalization of a probability distribution.

The following theorem establishes a general regret bound for OSMD. Note that here the pseudo-regret is defined as
$$\oR_n = \E \sum_{t=1}^n \ell_t(\tilde{x}_t) - \min_{x \in \cK} \E \sum_{t=1}^n \ell_t(x)~.$$
Note also that we state the theorem for a Legendre function $F$, but a similar result can be obtained under the same assumptions as those of Theorem~\ref{th:Fnotlegendre}.
\begin{theorem}[Pseudo-regret of OSMD] \label{th:OSMD2}
Let $\cK$ be a compact and convex set, $\cL$ a set of subdifferentiable functions, and $F$ a Legendre function defined on $\oD \supset \cK$. If OSMD is run with a loss estimate $\tilde{g}_t$ such that~\eqref{eq:consistency} is satisfied (with $\nabla \ell(x)$ replaced by $\tilde{g}_t$), and with $\E\bigl[\tilde{g}_t \mid x_t\bigr] =\nabla \ell_t(x_t)$, then
\begin{align*}
\oR_n & \leq \frac{\sup_{x \in \cK} F(x) - F(x_1)}{\eta} + \frac{1}{\eta} \sum_{t=1}^n \E\, D_{F^*}\Bigl(\nabla F(x_t) - \eta \tilde{g}_t, \nabla F(x_t)\Bigr) \\
&\quad + \sum_{t=1}^n \E\Bigl[\norm{x_t - \tilde{x}_t}  \norm{\tilde{g}_t}_*\Bigr]
\end{align*}
for any norm $\norm{\,\cdot\,}$.
Moreover if the loss is linear, that is $\ell(x) = \ell^{\top} x$, then
\begin{align*}
\oR_n & \leq \frac{\sup_{x \in \cK} F(x) - F(x_1)}{\eta} + \frac{1}{\eta} \sum_{t=1}^n \E\, D_{F^*}\Bigl(\nabla F(x_t) - \eta \tilde{g}_t, \nabla F(x_t)\Bigr) \\
&\quad + \sum_{t=1}^n \E\Bigl[\norm{x_t - \E\bigl[\tilde{x}_t \mid x_t\bigr]} \norm{\tilde{g}_t}_*\Bigr]~.
\end{align*}
\end{theorem}
\begin{proof}
Using Theorem \ref{th:MGD} one directly obtains:
$$\sum_{t=1}^n \tilde{g}_t^{\top} (x_t - x) \leq \frac{F(x) - F(x_1)}{\eta} + \frac{1}{\eta} \sum_{t=1}^n \E\, D_{F^*} \Bigl(\nabla F(x_t) - \eta \tilde{g}_t, \nabla F(x_t)\Bigr)~.$$
Moreover since $\E\bigl[\tilde{g}_t \mid x_t\bigr] =\nabla \ell_t(x_t)$, one has:
\begin{align*}
\E \sum_{t=1}^n \bigl( \ell_t(\tilde{x}_t) - \ell_t(x) \bigr) & = \E \sum_{t=1}^n \Bigl( \ell_t(\tilde{x}_t) - \ell_t(x_t) + \ell_t(x_t) - \ell_t(x) \Bigr) \\
& \leq \E \sum_{t=1}^n \norm{x_t - \tilde{x}_t} \norm{\tilde{g}_t}_* + \E \sum_{t=1}^n \nabla \ell_t(x_t)^{\top} (x_t - x) \\
& = \E \sum_{t=1}^n \norm{x_t - \tilde{x}_t} \norm{\tilde{g}_t}_* + \E \sum_{t=1}^n \tilde{g}_t^{\top} (x_t - x)
\end{align*}
which concludes the proof of the first regret bound. The case of a linear loss follows very easily from the same computations.
\end{proof}

\section{Online combinatorial optimization} \label{sec:semibandit}
In this section we consider an interesting special case of online linear optimization. In the online combinatorial optimization setting the set of arms is $\cC \subseteq \{0,1\}^d$ and the set of linear loss functions is $\cL = [0,1]^d$. We assume $\norm{v}_1 = m$ for all $v\in\cC$ and for some integer $m \le d$. Many interesting problems fall into this framework, including ranking/selection of $m$ items, or path planning.

Here we focus on the version of the problem with {\em semi-bandit} feedback, which is defined as follows: after playing $v_t \in \cC$, one observes $\bigl(\ell_t(1) v_t(1), \hdots, \ell_t(d) v_t(d)\bigr)$. Namely, one only observes the coordinates of the loss that were {\em active} in the arm $v_t$ that we chose. This setting has thus a much weaker feedback than the full information case, but still stronger than the bandit case. Note that the semi-bandit setting includes the basic multi-armed bandit problem of Chapter~\ref{adversarial}, which simply corresponds to $\cC = \{e_1, \hdots, e_d\}$ where $e_1, \hdots, e_d$ is the canonical basis of $\R^d$.

Again, the key to tackle this kind of problem is to select $v_t$ at random from some probability distribution $p_t$ over $\cC$. Note that such a probability corresponds to an average point $x_t \in \conv(\cC)$. Turning the tables, one can view $v_t$ as a random perturbation of $x_t$ such that $\E\bigl[v_t \mid x_t\bigr] = x_t$. This suggests a strategy: play OSMD on $\cK=\conv(\cC)$, with $\tilde{x}_t = v_t$. Surprisingly, we show that this randomization is enough to obtain a good unbiased estimate of the loss, and that it is not necessary to add further perturbations to $x_t$. Note that $\E\bigl[\tilde{x}_t \mid x_t\bigr] = x_t$ by definition. We now need to describe how to obtain an unbiased estimate of the gradient (which is the loss vector itself, since losses are linear). The following simple formula gives an unbiased estimate of the loss:
\begin{equation} \label{eq:semibanditestimate}
\tilde{\ell}_t(i) = \frac{\ell_t(i)\,v_t(i)}{x_t(i)} \qquad \forall i \in \{1,\hdots, d\} .
\end{equation}
Note that this is a valid estimate since it only makes use of $\bigl(\ell_t(1) v_t(1), \hdots, \ell_t(d) v_t(d)\bigr)$. Moreover, it is unbiased with respect to the random drawing of $v_t$ from $p_t$. Indeed,
\[
    \E\bigl[\tilde{\ell}_t(i) \mid x_t\bigr] = \frac{\ell_t(i)}{x_t(i)}\E\bigl[v_t(i) \mid x_t\bigr] = \ell_t(i)~.
\]
Using Theorem \ref{th:OSMD2} one directly obtains:
\begin{align}
\nonumber
    \oR_n
&\le
    \frac{\sup_{x \in \cK} F(x) - F(x_1)}{\eta}
\\ &\quad + \frac{1}{\eta} \sum_{t=1}^n \E\, D_{F^*}\Bigl(\nabla F(x_t) - \eta \tilde{\ell}_t, \label{eq:semibanditregret}
\nabla F(x_t)\Bigr)~.
\end{align}
We show now how to use this bound to obtain concrete performances for OSMD using the negative entropy as Legendre function. Later, we show that one can improve the results by logarithmic factors, using a more subtle Legendre function.
%

We start with OSMD and the negative entropy.
\begin{theorem}[OSMD with negative entropy]
\label{th:osmdnegent}
For any set $\cC\subseteq\{0,1\}^d$, if OSMD is run on $\cK=\conv(\cC)$ with $F(x) = \sum_{i=1}^d x_i \ln x_i - \sum_{i=1}^d x_i$, perturbed points $\tilde{x}_t$ such that $\E\bigl[\tilde{x}_t \mid x_t\bigr] = x_t$, and loss estimates $\tilde{\ell}_t$, then
$$\oR_n \leq \frac{m}{\eta}\ln \frac{d}{m} + \frac{\eta}{2} \sum_{t=1}^n \sum_{i=1}^d \E\bigl[x_t(i)\,\tilde{\ell}_t(i)^2\bigr]~.$$
In particular, with the estimate \eqref{eq:semibanditestimate} and $\eta= \sqrt{\frac{2m}{nd}\ln\frac{d}{m}}$,
$$\oR_n \leq  \sqrt{2 m d n \ln \frac{d}{m}}~.$$
\end{theorem}
\begin{proof}
First note that:
$$F(x) - F(x_1) \leq \sum_{i=1}^d x_1(i) \ln \frac{1}{x_1(i)} \leq m \ln \left( \sum_{i=1}^d \frac{x_1(i)}{m} \frac{1}{x_1(i)}\right) = m \ln \frac{d}{m}~.$$
Moreover, straightforward computations give
$$D_{F^*}\bigg(\nabla F(x_t) - \eta \tilde{\ell}_t, \nabla F(x_t)\bigg) =\sum_{i=1}^d x_{t}(i)\, \Theta\bigl(- \eta \tilde{\ell}_t(i)\bigr)$$
where $\Theta : x \in \R \mapsto \exp(x) - 1 - x$. Using that $\Theta(x) \leq \frac{x^2}{2}$ for all $x \leq 0$, concludes the proof of the first inequality (since $\tilde{\ell}_t(i) \geq 0$). The second inequality follows from
\[
    x_t(i)\,\E\bigl[\tilde{\ell}_t(i)^2 \mid x_t\bigr] = x_t(i)\frac{\ell_t(i)^2}{x_t(i)^2}\E\bigl[v_t(i) \mid x_t\bigr] \le 1
\]
where we used $\ell_t(i) \in [0,1]$ and $v_t(i) \in \{0,1\}$.
\end{proof}
%
%
We now greatly generalize the negative entropy with the following definition. When used with OSMD, this more general entropy allows us to obtain a bound tighter than that of Theorem~\ref{th:osmdnegent}.
\begin{definition}
Let $\omega \ge 0$.
A function 
$\psi: (-\infty,a) \rightarrow \R^*_+$ for some 
$a\in\R\cup\{+\infty\}$ is called an 
$\omega$-potential if it is convex,
 continuously differentiable, and satisfies
\begin{align*} 
& \lim_{x\rightarrow -\infty} \psi(x)=\omega &&
\lim_{x\rightarrow a} \psi(x)= +\infty \notag \\
& \psi' > 0 && \int_{\omega}^{\omega+1} |\psi^{-1}(s)|ds <+\infty~.
\end{align*}
With a potential $\psi$ we associate the function $F_{\psi}$ defined on $\cD=(\omega, +\infty)^d$ by
$$F_{\psi}(x) = \sum_{i=1}^d \int_{\omega}^{x_i} \psi^{-1}(s) ds~.$$
\end{definition}
We restrict our attention to $0$-potentials. A non-zero $\omega$ might be used to derive high probability regret bounds (instead of pseudo-regret bounds). Note that with $\psi(x)=e^x$ we have that $F_{\psi}$ reduces to the negative entropy.
\begin{lemma} \label{lem:psi}
Let $\psi$ be a $0$-potential. Then $F_{\psi}$ is Legendre
and for all $u,v \in \cD^*=(-\infty,a)^d$ such that $u_i \leq v_i$ for $i=1,\hdots,d$,
$$D_{F^*}(u,v) \leq \frac{1}{2} \sum_{i=1}^d \psi'(v_i) (u_i - v_i)^2~.$$
\end{lemma}
\begin{proof}
It is easy to check that $F$ is a Legendre function. Moreover, since
	$
  \nabla F^*(u)= (\nabla F)^{-1}(u)= \big(\psi(u_1),\dots,\psi(u_d)\big)
	$
we obtain
	$$
  D_{F^*}(u,v)=\sum_{i=1}^d \bigg(\int_{v_i}^{u_i} \psi(s)ds-(u_i-v_i)\psi(v_i)\bigg)~.
  $$
From a Taylor expansion, we have
	$$
	D_{F^*}(u,v)\le\sum_{i=1}^d \max_{s\in[u_i,v_i]} \frac12 \psi'(s) (u_i-v_i)^2~.
	$$
Since the function $\psi$ is convex, and $u_i \leq v_i$, we have
	$$
	\max_{s\in[u_i,v_i]} \psi'(s) \le \psi'\big(\max\{u_i,v_i\}\big) \leq \psi'(v_i)
	$$
which gives the desired result.
\end{proof}
We are now ready to bound the pseudo-regret of OSMD run with an arbitrary $0$-potential. For a specific choice of the potential we obtain an improvement of Theorem~\ref{th:osmdnegent}. In particular for $m=1$ this result gives the log-free bound for the adversarial multi-armed bandit that was discussed in Section \ref{sec:logfree}.
\begin{theorem}[OSMD with a $0$-potential] \label{th:osmdzero}
For any set subset $\cC$ of $\{0,1\}^d$, if OSMD is run on $\cK=\conv(\cC)$ with $F_{\psi}$ defined by a $0$-potential $\psi$, and non-negative loss estimates $\tilde{\ell}_t$, then
$$\oR_n \leq \frac{\sup_{x \in \cK} F_{\psi}(x) - F_{\psi}(x_1)}{\eta} + \frac{\eta}{2} \sum_{t=1}^n \sum_{i=1}^d \E\left[\frac{\tilde{\ell}_t(i)^2}{(\psi^{-1})'\bigl(x_t(i)\bigr)}\right]~.$$
In particular, choosing the $0$-potential $\psi(x) = (- x)^{-q}$, the estimate~\eqref{eq:semibanditestimate}, and $\eta= \sqrt{\frac{2}{q-1} \frac{m^{1 - 2/q}}{d^{1 - 2/q}}}$,
$$\oR_n \leq q \sqrt{\frac{2}{q-1} m d n}~.$$
With $q=2$ this gives
$$\oR_n \leq 2 \sqrt{2 m d n}~.$$ 
\end{theorem}
\begin{proof}
First note that since $\cD^*=(-\infty,a)^d$ and $\tilde{\ell}_t$ has non-negative coordinates, then~\eqref{eq:consistency} is satisfied and thus OSMD is well defined.

The first inequality trivially follows from~\eqref{eq:semibanditregret}, Lemma~\ref{lem:psi}, and the fact that $\psi'\bigl(\psi^{-1}(s)\bigr) = \frac{1}{(\psi^{-1})'(s)}$.

Let $\psi(x) = (- x)^{-q}$. Then we have that $\psi^{-1}(x) = - x^{-1/q}$ and $F(x)= - \frac{q}{q-1} \sum_{i=1}^d x_i^{1 - 1/q}$. In particular, by H{\"o}lder's inequality, since $\sum_{i=1}^d x_1(i) =m$,
$$F(x) - F(x_1) \leq \frac{q}{q-1} \sum_{i=1}^d x_1(i)^{1 - 1/q} \leq \frac{q}{q-1} m^{(q-1)/q} d^{1/q}~.$$
Moreover, note that $(\psi^{-1})'(x) = \frac1q x^{- 1 - 1/q}$, and
$$\sum_{i=1}^d \E\left[\frac{\tilde{\ell}_t(i)^2}{(\psi^{-1})'(x_t(i))} \,\bigg|\, x_t \right] \leq q \sum_{i=1}^d x_t(i)^{1/q} \leq q m^{1/q} d^{1 - 1/q}$$
which ends the proof.
\end{proof}

\section{Improved regret bounds for bandit feedback} \label{sec:ball}
We go back to the setting of linear losses with bandit feedback considered in Section~\ref{sec:Exp2}. Namely, actions belong to a compact and convex set $\cK\subseteq\R^d$, losses belong to a subset $\cL\subseteq\R^d$, and the loss of playing $x_t\in\cK$ at time $t$ is $x_t^{\top}\ell_t$, which is also the feedback received by the player. As we proved in Section~\ref{sec:Exp2}, under the bounded scalar loss assumption, $|x^{\top}\ell| \le 1$ for all $(x,\ell)\in\cK\times\cL$, one can obtain a regret bound of order $d \sqrt{n}$ (up to logarithmic factors) for any compact and convex set $\cK$. It can be shown that this rate is not improvable in general. However, results from Section~\ref{sec:semibandit} (or from Chapter~\ref{adversarial}) show that for the simplex, one can obtain a regret bound of order $\sqrt{d n}$, and we showed in Chapter~\ref{adversarial} that this rate is also unimprovable. The problem of obtaining a charaterization of the sets for which such improved regret bounds are possible is an open problem. Improved rates can be obtained for another convex body: the Euclidean ball. We now describe a strategy that attains a pseudo-regret of order $\sqrt{d n}$ (up to a logarithmic factor) in this case. The strategy is based on OSMD with a carefully chosen Legendre function.

In the following, let $\norm{\,\cdot\,}$ be the Euclidean norm. We consider the online linear optimization problem with bandit feedback on the Euclidean unit ball $\cK = \{x \in \R^d \,:\, \norm{x} \leq 1 \}$. We perform the following perturbation of a point $x_t$ in the interior of $\cK$,
\[
    \tilde{x}_t = \left\{ \begin{array}{cl}
        x_t / \norm{x_t} & \text{if $\xi_t = 1$,}
    \\
        \epsilon_t\,e_{I_t} & \text{otherwise}
    \end{array} \right.
\]
where $\xi_t$ is a Bernoulli random variable of parameter $\norm{x_t}$, $I_t$ is drawn uniformly at random in $\{1, \hdots, d\}$, and $\epsilon_t$ is a Rademacher random variable with parameter $\frac12$.

It is easy to check that this perturbation is unbiased, in the sense that
$\E\bigl[\tilde{x}_t \mid x_t\bigr] = x_t$.
An unbiased estimate of the loss vector is given by
\begin{equation} \label{eq:ballestimate}
\tilde{\ell}_t = d(1 - \xi_t) \frac{\tilde{x}_t^{\top}\ell_t}{1- \norm{x_t}}\,\tilde{x}_t~.
\end{equation}
Again, it is easy to check that
$\E\bigl[\tilde{\ell}_t \mid x_t\bigr] = x_t$.
We are now ready to prove the following result, showing that OSMD with a suitable $F$ achieves a pseudo-regret of order $\sqrt{dn\ln n}$ on the Euclidean ball.
\begin{theorem}[OSMD for the Euclidean ball] \label{th:ball}
Let $\cK = \cL = \{x \in \R^d \,:\, \norm{x} \leq 1 \}$ define an online linear optimization problem with bandit feedback. If OSMD is run on $\cK' = (1-\gamma)\cK$ with $F(x) = - \ln(1 - \norm{x}) - \norm{x}$ and the estimate~\eqref{eq:ballestimate}, then for any $\eta > 0$ such that $\eta d \leq \frac12$,
\begin{equation} \label{eq:ballregret1}
\oR_n \leq \gamma n + \frac{\ln \gamma^{-1}}{\eta} + \eta \sum_{t=1}^n \E\Bigl[(1 - \norm{x_t})\big\|\tilde{\ell}_t\big\|^2\Bigr]~.
\end{equation}
In particular, with $\gamma = \frac{1}{\sqrt{n}}$ and $\eta= \sqrt{\frac{\ln n}{2 n d}}$,
\begin{equation} \label{eq:ballregret2}
\oR_n \leq  3 \sqrt{d n \ln n}~.
\end{equation}
\end{theorem}
\begin{proof}
First, it is clear that by playing on $\cK' = (1-\gamma)\cK$ instead of $\cK$, OSMD incurs an extra $\gamma n$ regret. Second, note that $F$ is stricly convex (it is the composition of a convex and nondecreasing function with the Euclidean norm) and
\begin{equation} \label{eq:ballgradient}
\nabla F (x) = \frac{x}{1 - \norm{x}}~.
\end{equation}
In particular, $F$ is Legendre on the open unit ball $\cD = \{x \in \R^d \,:\, \norm{x} < 1 \}$, and one has $\cD^* = \R^d$. Hence~\eqref{eq:consistency} is always satisfied, and OSMD is well defined. Now the regret with respect to $\cK'$ can be bounded as follows: using Theorem~\ref{th:OSMD2} and the unbiasedness of $\tilde{x}_t$ and $\tilde{\ell}_t$ we get
$$\frac{\sup_{x \in \cK} F(x) - F(x_1)}{\eta} + \frac{1}{\eta} \sum_{t=1}^n \E\,D_{F^*}\Bigl(\nabla F(x_t) - \eta \tilde{\ell}_t, \nabla F(x_t)\Bigr)~.$$
The first term is clearly bounded by $\frac{1}{\eta}\ln\frac{1}{\gamma}$ (since $x_1 = 0$). For the second term, we need to do a few computations. The first one follows from~\eqref{eq:ballgradient}), the others follow from simple algebra
\begin{align*}
\nabla F^* (u) & = \frac{u}{1 + \norm{u}}\\
F^*(u) & = - \ln (1 + \norm{u}) + \norm{u}\\
D_{F^*}(u, v) & = \frac{1}{1+\norm{v}} \bigg(\norm{u} - \norm{v} + \norm{u} \cdot \norm{v} - v^T u   \\
& \qquad \left. - (1+\norm{v}) \ln \left( 1 + \frac{\norm{u} - \norm{v}}{1 + \norm{v}}\right) \right)~.
\end{align*}
Let $\Theta(u,v)$ such that $D_{F^*}(u, v) = \frac{1}{1+\norm{v}} \Theta(u,v)$. First note that 
\begin{equation} \label{eq:superball}
\frac{1}{1+\norm{\nabla F (x_t)}} = 1 - \norm{x_t}~.
\end{equation}
Thus, in order to prove \eqref{eq:ballregret1} it remains to show that $\Theta(u,v) \leq \norm{u - v}^2$, for $ u = \nabla F(x_t) - \eta \tilde{\ell}_t$ and $v = \nabla F(x_t)$. In fact, we prove that this inequality holds as soon as $\frac{\norm{u} - \norm{v}}{1 + \norm{v}} \geq - \frac{1}{2} .$ This is the case for the pair $(u,v)$ under consideration, since by the triangle inequality, equations \eqref{eq:ballestimate} and \eqref{eq:superball}, and the assumption on $\eta$,
$$\frac{\norm{u} - \norm{v}}{1 + \norm{v}}  \geq - \frac{\eta \big\|\tilde{\ell}_t\big\|}{1 + \norm{v}} \geq - \eta d \geq - \frac{1}{2}~.$$
Now using that $\ln(1+x) \geq x -x^2$ for all $x \geq - \frac12$, we obtain that for $u,v$ such that $\frac{\norm{u} - \norm{v}}{1 + \norm{v}} \geq - \frac{1}{2}$,
\begin{align*}
\Theta(u,v) & \leq \frac{(\norm{u} - \norm{v})^2}{1 + \norm{v}} + \norm{u} \cdot \norm{v} - v^{\top} u \\ 
& \leq (\norm{u} - \norm{v})^2 + \norm{u} \cdot \norm{v} - v^{\top} u \\
& = \norm{u}^2 + \norm{v}^2 - \norm{u} \cdot \norm{v} - v^{\top} u \\
& = \norm{u - v}^2 + 2 v^{\top} u - \norm{u} \cdot \norm{v} - v^{\top} u \\
& \leq \norm{u - v}^2
\end{align*}
which concludes the proof of \eqref{eq:ballregret1}. For the proof of~\eqref{eq:ballregret2} it suffices to note that
\begin{align*}
\E\Bigl[1 - \norm{x_t}\big\|\tilde{\ell}_t\big\|^2\Bigr]  & = (1- \norm{x_t}) \sum_{i=1}^d  \frac{1 - \norm{x_t}}{d} \frac{d^2}{(1 - \norm{x_t})^2} (\ell_t^{\top} e_i)^2 \\
&  = d \norm{\ell_t}^2 \\
& \le d
\end{align*}
and perform with straightforward computations.
\end{proof}

\section{Refinements and bibliographic remarks}
Online convex optimization in the full information setting was introduced by \cite{Zin03}. Online linear optimization with bandit feedback was pioneered in \cite{AK04, MB04}. For this problem, \cite{DHK08} were the first to obtain optimal $\mathcal{O}\bigl(\sqrt{n}\bigr)$ bounds in terms of the number $n$ of rounds. This was done using the Exp2 strategy with an exploration uniform over a barycentric spanner for $\cK$. The exploration part was first improved by \cite{CL11} for combinatorial sets $\cK$. Finally, the optimal exploration based on John's theorem was introduced by \cite{BCK12}. Theorem~\ref{th:exp2john} is extracted from this last paper.

Simultaneously with the line of research on Exp2, algorithms based on Online Mirror Descent were also investigated. Mirror Descent was originally introduced in the seminal work of \cite{Nem79, NY83} as a standard (offline) convex optimization method. A somewhat similar class of algorithms was rediscovered in the online learning community, see \cite{HW98, GLS01, KW01, Sha07}. The connection between existing online learning algorithms (such as Exponential weights or Online Gradient Descent) and Mirror Descent was first made explicit in \cite{CL06} ---see also \cite{Rak09} and \cite{Haz11}. Earlier applications of Mirror Descent in the learning community can be found in \cite{JNTV05}. The first application of Mirror Descent to online linear optimization with bandit feedback was given by \cite{AHR08}. In this pioneering paper, the authors describe the first computationally efficient strategy (i.e., with complexity polynomial in $d$) with $\mathcal{O}(\sqrt{n})$ regret. The main idea is to use Mirror Descent with a self-concordant barrier $F$ for the set $\cK$. Unfortunately, the drawback is a suboptimal dependency on $d$ in the regret. More precisely. they obtain a $\mathcal{O}(d^2 \sqrt{n})$ regret under the bounded scalar loss assumption, while Exp2 with John's exploration attains $\mathcal{O}(d \sqrt{n})$. However, Mirror Descent can also deliver optimal regret bounds in the bandit case, as we showed in Section \ref{sec:ball}, which is extracted from \cite{BCK12}.

The presentation of the Online Mirror Descent algorithm in Section~\ref{sec:OMD} is inspired by~\cite{Bub11}. The definition of Legendre functions comes from \cite[Chapter 11]{CL06} ---further developments on convex analysis can be found in \cite{HL01, BV04}. Theorem~\ref{th:MGD} is taken from \cite{ABL11}, but the proof technique goes back at least to \cite{BN99}. The proof of Theorem~\ref{th:Fnotlegendre} is adapted from \cite{KST10}. Section~\ref{sec:OSMD} is inspired by gradient-free optimization, a topic extensively studied since \cite{RM51, KW52} ---see \cite{NJLS09, CSV09, Nes11, BM11} for recent accounts on this theory.
Alternative views have been proposed on the Online Mirror Descent strategy. In particular, it is equivalent to a Follow The Regularized Leader, and to proximal algorithms, see \cite{Rak09}. This viewpoint was pioneered by \cite{BT03} ---see also \cite{BPSS11} for more details. Finally, a notion of universality of Online Mirror Descent in the online prediction setting was proposed by \cite{SST11}.

The online combinatorial optimization problem studied in Section~\ref{sec:semibandit} was introduced by \cite{KV05} for the full information setting. Several works have studied this problem for specific sets $\cC$, see in particular \cite{TW03, WK08, HW09, HKW10, KWK10, WKH11, CL11}. The semi-bandit feedback was studied in the series of papers \cite{GLLO07, KRS10, UNK10, ABL11}. The presentation adopted in this section is based on the last paper. OSMD with negative entropy was first studied by \cite{HW09} for the full information setting and for a specific set $\cC$. It was then studied more generally in \cite{KWK10} for any set $\cC$. The generalization to semi-bandit feedback was done by \cite{ABL11}. OSMD with a Legendre derived from a potential was introduced by \cite{ABL11}. In the case of the simplex, this strategy corresponds to the INF strategy of \cite{AB09} discussed in Section~\ref{sec:logfree}.

Online linear optimization is still far from being completely understood. For instance, see \cite[Chapter 9]{Bub11} for a list of open problems. In this section we also omitted a few important topics related to online linear optimization. We briefly review some of them below.

\subsection{Lower bounds}
Under the bounded scalar loss assumption, it was proved by \cite{DHK08} that for $\cK=[-1,1]^d$ the minimax regret in the full information setting is at least of order $\sqrt{d n}$, while under bandit feedback it is of order $d \sqrt{n}$. In both cases Exp2 is matching these lower bounds (using John's exploration in the bandit case). 

In the combinatorial setting, where $\cK \subset \{0,1\}^d$ and $\cL = [0,1]^d$, \cite{ABL11} show that the minimax regret in the full information and semi-bandit cases is at least of order $d \sqrt{n}$, while in the bandit case it is of order $d^{3/2} \sqrt{n}$. OSMD with the negative entropy matches the bounds in the full information and semi-bandit cases. However, in the bandit case the best known bound is obtained by Exp2 (with John's exploration) and gives a regret of order $d^2 \sqrt{n}$. It is important to remark that \cite{ABL11} show that Exp2 is a provably suboptimal strategy in the combinatorial setting.

Finally, lower bounds for the full information case, and for a few specific sets $\cK$ of interest, were derived by \cite{KWK10}.

\subsection{High probability bounds} \label{sec:hplinear}
In this chapter we focused on the pseudo-regret $\oR_n$. However, just like in Chapter \ref{adversarial}, a much more important and interesting statement concerns high probability bounds for the regret $R_n$. Partial results in this direction can be found in \cite{BDHKRT08} for the Exp2 strategy, and in \cite{AR09} for the OSMD algorithm.

\subsection{Stochastic online linear optimization} \label{sec:stochlinear}
Similarly to the stochastic bandit case (see Chapter \ref{stochastic}), a natural restriction to consider for the adversary is that the sequence of losses $\ell_1,\ell_2,\ldots$ is an i.i.d.\ sequence. This stochastic setting was introduced by \cite{Aue02}, and further studied by \cite{DHK08b}. In particular, in the latter paper it was proved that regrets logarithmic in $n$ and polynomial in $d$ are possible, as long as $\cK$ is a polytope. Recent progress on this problem can be found in \cite{RT10, filippi2010parametric, APS11}.

%% file: nonlinear.tex
We now extend the analysis of the previous chapter to the following scenario: arms are still points in a convex set $\cK \subset \R^d$, but now losses are not necessarily linear functions of the arms. More precisely the adversary selects loss functions from some set $\cL$ of real-valued functions defined on $\cK$. The pseudo-regret is then defined as:
$$\oR_n = \E \sum_{t=1}^n  \ell_t(x_t) - \min_{x \in \cK} \E \sum_{t=1}^n \ell_t(x) .$$
This modification has important consequences.
For instance with strictly convex losses one has to do local perturbations in order to estimate the loss gradient, this is in contrast to the global perturbations studied in the previous chapter.
In agreement with the setting of Chapter~\ref{linear}, we initially focus on the nonstochastic setting, where the forecaster faces an unknown sequence of convex Lipschitz and differentiable losses (in the nonlinear case the regret scales with the Lipschitz constant of losses). Problems of this kind can be viewed as dynamic variants of convex optimization problems, in which the convex function to optimize evolves over time. The bandit constraint can be simply interpreted as the impossibility of computing gradients (because, for instance, we do not have a explicit representation of the function, but it can only be accessed by querying for values at desired points).

We look at two feedback models. In the first one, at each step the forecaster evaluates the loss function at two points: the played point plus an additional point of its choice. In the second one, only the value of the loss evaluated at the played point is made available to the forecaster. We show that while the two-points model allows for a $\scO\bigl(\sqrt{n}\bigr)$ bound on pseudo-regret, in the one-point model a pseudo-regret bound of only $\scO\bigl(n^{3/4}\bigr)$ is achieved. The stochastic setting is investigated in Section~\ref{s:nonlinear-stochastic} where, similarly to Chapter~\ref{stochastic}, we assume that each play of an arm returns a stochastic loss with fixed but unknown mean. Unlike the nonstochastic case, the mean loss function is assumed to be Lipschitz and unimodal, but not necessarily convex.
For keeping things simple, the stochastic setting is studied in $1$-dimensional case, when arms are points in the unit interval. For this case we show a bound on the pseudo-regret of $\scO\bigl(\sqrt{n}(\log n)\bigr)$.

\section{Two-points bandit feedback}
We start by analyzing the nonstochastic case in the two-point feedback model: at each time step $t$, the forecaster observes the value of a convex and differentiable loss function $\ell_t$ at the played point $x_t$ and at an extra point $x_t'$ of its choice. If the second point is chosen at random in a neighborhood of the first one, one can use it to compute an estimate of the gradient of $\ell_t$ at $x_t$. Hence, running OSMD on the estimated gradients we obtain a regret bound controlled by the second moments of these estimates. The algorithm we present in this section follows this intuition, although ---for technical reasons--- the gradient is estimated at a point which is close but distinct from the point actually played.

We focus our analysis on OSMD with Legendre function $F = \tfrac{1}{2}\norm{\cdot}^2$, where $\norm{\cdot}$ is the Euclidean norm. The resulting strategy, Online Stochastic Gradient Descent (OSGD) is sketched below here.
\bookbox{
{\em OSGD (Online Stochastic Gradient Descent):}

\smallskip\noindent
Parameters: Closed and convex set $\cK\subseteq\R^d$, learning rate $\eta > 0$.\\
Initialize: $x_1 = (0,\dots,0)$.

\smallskip\noindent
For each round $t=1,2,\dots,n$
\begin{itemize}
\vspace{-2mm}
\item[(1)] Observe stochastic estimate $\gtilde_t(x_t)$ of $\nabla\ell_t(x_t)$;
\item[(2)] $x_{t+1}' = x_t - \eta\,\gtilde_t(x_t)$;
\item[(3)] ${\dt x_{t+1} = \argmin_{y \in \cK} \norm{y - x_{t+1}'} }$;
\end{itemize}
}

We now introduce our main technical tool: the two-point gradient estimate. The two points on which the loss value is queried at time $t$ are denoted by $X_t^+$ and $X_t^-$. OSGD always plays one of these two points at random.

Let $\Ball = \theset{x \in \R^d}{\norm{x} \le 1}$ be the unit ball in $\R^d$ and $\Sphere = \theset{x \in \R^d}{\norm{x} = 1}$ be the unit sphere.
Fix $\delta > 0$ and introduce the notations $X_t^+ = x_t + \delta S$ and $X_t^- = x_t - \delta S$, where $x_t \in \cK$ and $S$ is a random variable with uniform distribution in $\Sphere$. Then, for any convex loss $\ell_t$, the two-point gradient estimate $\gtilde_t$ is defined by
\begin{equation}
\label{eq:two-point-est}
    \gtilde_t(x_t) = \frac{d}{2\delta}\bigl(\ell_t(X_t^+) - \ell_t(X_t^-) \bigr)S~.
\end{equation}
In order to compute the expectation of $\gtilde_t$, first note that by symmetry
\[
    \E\,\gtilde_t(x) = \frac{d}{\delta}\,\E\bigl[\ell_t(x + \delta S)S\bigr]~.
\]
In order to compute the expectation in the right-hand side we need the following preliminary lemma.
\begin{lemma}
\label{l:divergence}
For any differentiable function $\ell : \R^d \to \R$
\[
    \nabla\int_{\Ball} \ell(x + \delta b)\,db = \int_{\Sphere} \ell(x + \delta s)s\,d\sigma(s)
\]
where $\sigma$ is the unnormalized spherical measure.
\end{lemma}
\begin{proof}
The proof of this result is an easy consequence of the Divergence Theorem,
\begin{align*}
    \nabla\int_{\Ball} \ell(x + \delta b)\,db
&=
    \int_{\Ball} \nabla \ell(x + \delta b)\,db
\\ &=
    \int_{\Sphere} \frac{1}{\delta} \ell(x + \delta s)s\,d \sigma(s)
\\ &=
    \frac{1}{\delta} \int_{\Sphere} \ell(x + \delta s)s\,d \sigma(s)~.
\end{align*}
\end{proof}
We are now fully equipped to compute the expectation of $\gtilde_t$.
\begin{lemma}
\label{l:stokes}
If $B$ is a random variable with uniform distribution in $\Ball$ and $S$ is a random variable with uniform distribution in $\Sphere$, then for all differentiable functions $\ell_t : \R^d\to\R$,
\[
    \frac{d}{\delta}\E\bigl[\ell(x + \delta S)S\bigr] = \nabla\E\,\ell(x + \delta B)~.
\]
\end{lemma}
\begin{proof}
First consider the easy one-dimensional case. Namely, $\cK = [a,b]$ for some reals $a < b$. Note that, in this case, $S$ is uniform in $\{-1,+1\}$ whereas $B$ is uniform in $[-1,+1]$. Then
\begin{align*}
    \E\,\ell(x + \delta B)
=
    \frac{1}{2\delta}\int_{-\delta}^{\delta} \ell(x + \ve)\,d\ve
=
    \frac{L(x+\delta) - L(x-\delta)}{2\delta}
\end{align*}
by the fundamental theorem of calculus, where $L$ is the antiderivative of $\ell$ satisfying $L' = \ell$.
This gives
\[
    \frac{d}{dx} \E\,\ell(x + \delta B) = \frac{\ell(x+\delta) - \ell(x-\delta)}{2\delta}~.
\]
On the other hand,
\[
    \frac{1}{\delta}\,\E\bigl[\ell(x + \delta S)S\bigr]
=
    \frac{\ell(x+\delta) - \ell(x-\delta)}{2\delta}~.
\]
Hence $\tfrac{1}{\delta}\E\bigl[\ell(x + \delta S)S\bigr] = \tfrac{d}{dx}\E\,\ell(x + \delta B)$ and the $1$-dimensional case is established. Note that the equivalence we just proved relates an integral over the unit sphere $\Sphere$ to an integral over the unit ball $\Ball$. In $d$ dimensions, Lemma~\ref{l:divergence} delivers the corresponding generalized identity
\[
    \frac{1}{\delta}\int_{\Sphere} \ell(x + \delta s)s\, d\sigma(s)
=
    \nabla\int_{\Ball} \ell(x + \delta b)\, db~.
\]
Now, since $\mathrm{Vol}(\Sphere) = d\,\mathrm{Vol}(\Ball)$ we immediately obtain
\[
    \frac{d}{\delta}\E\bigl[\ell(x + \delta S)S\bigr]
=
    \nabla\E\,\ell(x + \delta B)
\]
concluding the proof.
\end{proof}
We have thus established $\E\,\gtilde_t(x) = \nabla\E\,\ell_t(x + \delta B)$, showing that $\gtilde_t$ provides an unbiased estimate of a smoothed version $\tloss_t(x) = \E\,\ell_t(x + \delta B)$ of the loss function $\ell_t$.

We can measure how well $\tloss_t$ approximates $\ell_t$ by exploiting the Lipschitz assumption,
\begin{align}
\nonumber
    \bigl|\ell_t(x) - \tloss_t(x)\bigr|
&=
    \bigl|\ell_t(x) - \E\,\ell_t(x + \delta B)\bigr|
\\ &\le
\nonumber
    \E\bigl|\ell_t(x) - \ell_t(x + \delta B)\bigr|
\\ &\le
\nonumber
    \delta G\,\E\norm{B}
\\ &\le
\label{eq:lip-tloss}
    \delta G~.
\end{align}
The next lemma relates the regret under the losses $\ell_t$ to the regret under their smoothed versions $\tloss_t$. An additional issue taken into account by the lemma is that OSGD might play a point close to the boundary of the set $\cK$. In this case the perturbed point on which the gradient is estimated could potentially be outside of $\cK$. In order to prevent this from happening we need to run OSGD on a shrunken set $(1-\xi)\cK$.
\begin{lemma}
\label{l:two-point}
Let $\cK\subseteq\R^d$ be a convex set such that $\cK\subseteq R\,\Ball$ for some $R \ge 0$, and fix $0 \le \xi \le 1$.
For any sequence $\ell_1,\ell_2,\dots$ of $G$-Lipschitz differentiable and convex losses, and for any sequence $x_1,x_2,\ldots\in (1-\xi)\cK \subseteq\R^d$, the following holds
\begin{align*}
    \frac{1}{2}\sum_{t=1}^n &\bigl(\ell_t(X_t^+) + \ell_t(X_t^-)\bigr) - \sum_{t=1}^n \ell_t(x)
\\& \le
    \sum_{t=1}^n \tloss_t(x_t) - \sum_{t=1}^n \tloss_t\bigl((1-\xi)x\bigr) + 3\delta G n + \xi GR n
\end{align*}
for all realizations of the random process $\bigl(X_t^+,X_t^-\bigr)_{t \ge 1}$.
\end{lemma}
\begin{proof}
Using the Lipschitzness of $\ell_t$ and~(\ref{eq:lip-tloss}) we obtain
\begin{align*}
    \frac{1}{2} &\bigl(\ell_t(X_t^+) + \ell_t(X_t^-)\bigr) + \tloss_t\bigl((1-\xi)x\bigr)
\\ &\le
    \frac{1}{2} \Bigl(\ell_t(x_t) + \delta G \norm{S} + \ell_t(x_t) + \delta G \norm{S} \Bigr) + \ell_t\bigl((1-\xi)x\bigr) + \delta G
\\ &\le
    \ell_t(x_t) + \ell_t\bigl(x\bigr) + 2\delta G + \xi GR
\\ &\le
    \tloss_t(x_t) + \ell_t\bigl(x\bigr) + 3\delta G + \xi GR~.
\end{align*}
In the second step we used $\tloss_t\bigl((1-\xi)x\bigr) \le \xi G\norm{x} \le \xi GR$ which results from the Lipschitzness of $\ell_t$ and the assumption $\cK\subseteq R\,\Ball$.
\end{proof}
Next, we show that the second moment of $\gtilde_t$ can be controlled by exploiting the Lipschitzness of $\ell_t$. In particular,
\begin{align*}
    \norm{\gtilde_t(x)}
=
    \frac{d}{2\delta}\bigl|\ell_t(x + \delta S) - \ell_t(x - \delta S)\bigr|\norm{S}
\le
    \frac{Gd}{2\delta}\norm{2\delta S}
=
    Gd~.
\end{align*}
We are now ready to prove the main result of this section. Namely, that the pseudo-regret of OSGD run using the gradient estimate~(\ref{eq:two-point-est}) is of order $\sqrt{n}$. We assume that the point $\tX_t$ played by OSGD at each time $t$ is randomly drawn between the two points $X_t^+$ and $X_t^-$ where the loss function is queried.
\begin{theorem}[Regret of OSGD with two-points feedback]
\label{th:two-point}
Let $\cK\subseteq\R^d$ be a closed convex set such that $r\,\Ball\subseteq\cK\subseteq R\,\Ball$ for some $r,R > 0$. 
Let $\cL$ be a set of $G$-Lipschitz differentiable and convex losses.
Fix $\delta > 0$ and assume OSGD is run on $\bigl(1-\tfrac{\delta}{r}\bigr)\cK$ with learning rate $\eta > 0$ and gradient estimates~(\ref{eq:two-point-est}),
\[
    \gtilde_t(x_t) = \frac{d}{2\delta}\bigl( \ell_t(X_t^+) - \ell_t(X_t^-) \bigr)S_t
\]
where $S_1,S_2,\dots\in\Sphere$ are independent. For each $t=1,2,\dots$ let $\tX_t$ be drawn at random between $X_t^+$ and $X_t^-$. Then
the following holds
\[
\oR_n
\le
    \frac{R^2}{\eta} + \eta (Gd)^2 n + \delta\left(3 + \frac{R}{r} \right) G n~.
\]
Moreover, if $\eta = \tfrac{R}{GD\sqrt{n}}$ then for $\delta\to 0$ we have that
\[
\oR_n
\le
    2RGd\sqrt{n}~.
\]
\end{theorem}
\begin{proof}
First of all, we must check that the points $X_t^+ = x_t + \delta S$ and $X_t^- = x_t - \delta S$ on which $\ell_t$ is queried belong to $\cK$. To see this, recall that $x_t\in \bigl(1-\tfrac{\delta}{r}\bigr)\cK$. Now, setting $\alpha = \tfrac{\delta}{r}$ we have that $X_t^+,X_t^- \in (1-\alpha)\cK + \alpha r\,\Sphere$. Since $r\,\Sphere \subseteq \cK$ and $\cK$ is convex, we obtain $(1-\alpha)\cK + \alpha r\,\Sphere \subseteq (1-\alpha)\cK + \alpha\cK \subseteq \cK$. Hence, using Lemma~\ref{l:two-point} with the choice $\xi = \tfrac{\delta}{r}$ we immediately get that for all $x\in\cK$,
\begin{align*}
    \sum_{t=1}^n & \E( \ell_t(\tX_t) | X_t^+, X_t^-) - \sum_{t=1}^n \ell_t(x)
\\& \le
    \frac{1}{2}\sum_{t=1}^n \bigl(\ell_t(X_t^+) + \ell_t(X_t^-)\bigr) - \sum_{t=1}^n \ell_t(x)
\\& \le
    \sum_{t=1}^n \tloss_t(x_t) - \sum_{t=1}^n \tloss_t\Bigl(\bigl(1-\tfrac{\delta}{r}\bigr)x\Bigr) + \delta\left(3 + \frac{R}{r} \right) G n~.
\end{align*}
Since we already related the loss of $\tX_t$ to the loss of $x_t$, we can now apply Theorem~\ref{th:OSMD2} in the special case of $\tilde{x}_t = x_t$ and with the sequence of losses $(\tloss_t)$. This gives
\begin{align*}
    \E \sum_{t=1}^n \tloss_t(x_t) - \E \sum_{t=1}^n \tloss_t\Bigl(\bigl(1-\tfrac{\delta}{r}\bigr)x\Bigr)
& \le
    \frac{R^2}{\eta} + \eta\sum_{t=1}^n \E \norm{\gtilde_t(x_t)}^2
\\& \le
    \frac{R^2}{\eta} + \eta (Gd)^2 n
\end{align*}
where we overapproximated $\norm{\bigl(1-\tfrac{\delta}{r}\bigr)\cK} \le \norm{\cK} = R$. This concludes the proof.
\end{proof}

\section{One-point bandit feedback}
Building on the analysis of the previous section, it is not hard to show that the pseudo-regret can be bounded even when the loss function at each time $t$ is queried in only \textsl{one} point. However, we pay this reduced bandit feedback with a worse rate of $n^{3/4}$ in the pseudo-regret bound. It is not known if this rate is optimal, or if it is possible to get a $\sqrt{n}$ regret as in the two-points setting.

The one-point estimate at time $t$ is defined by
\begin{equation}
\label{eq:one-point-est}
    \gtilde_t(x) = \frac{d}{\delta}\ell_t(x + \delta S)S
\end{equation}
where $S$ is drawn at random from $\Sphere$. Obviously, Lemma~\ref{l:stokes} can be applied to get
$
    \E\,\gtilde_t(x) = \nabla\tloss_t(x)
$
where, we recall, $\tloss_t(x) = \E\,\ell_t(x + \delta B)$. Differences with the two-point case arise when we bound the second moment of this new $\gtilde_t$. Indeed, if $x + \delta S \in \cK$ and the maximum value of each $\ell_t$ in $\cK$ is bounded by $L$, then
\[
    \norm{\gtilde_t(x)}
=
    \frac{d}{\delta}\bigl|\ell_t(x + \delta S)\bigr|\norm{S}
\le
    \frac{dL}{\delta}~.
\]
Note the inverse dependence on $\delta$. This dependence plays a key role in the final bound, as the next result shows.
\begin{theorem}[Regret of OSGD with one-point feedback]
Let $\cK\subseteq\R^d$ be a closed convex set such that $r\,\Ball\subseteq\cK\subseteq R\,\Ball$ for some $r,R > 0$. 
Let $\cL$ be a set of $G$-Lipschitz differentiable and convex losses, uniformly bounded by $L$ (that is $||\ell||_{\infty} \leq L, \forall \ell \in \cL$). 
Fix $\delta > 0$ and assume OSGD is run on $\bigl(1-\tfrac{\delta}{r}\bigr)\cK$ with learning rate $\eta > 0$ and gradient estimates~(\ref{eq:one-point-est}),
\[
    \gtilde_t(x_t) = \frac{d}{\delta}\ell_t(\tX_t)S_t
\]
where $\tX_t = x_t + \delta S_t$ and $S_1,S_2,\dots\in\Sphere$ are independent. 
Then
the following holds
\[
\oR_n
\le
    \frac{R^2}{\eta} + \frac{(dL)^2}{\delta^2}\eta n + \delta\left(3 + \frac{R}{r} \right) G n~.
\]
Moreover, if
\[
    \delta = \frac{1}{(2n)^{1/4}}\sqrt{\frac{RdL}{\bigl(3 + \tfrac{R}{r} \bigr)G}}
\qquad\text{and}\qquad
    \eta = \frac{1}{(2n)^{3/4}}\sqrt{\frac{R^3}{dL\bigl(3 + \tfrac{R}{r} \bigr)G}}
\]
then
\[
\oR_n
\le
    4 n^{3/4}\sqrt{RdL\bigl(3 + \tfrac{R}{r} \bigr)G}~.
\]
\end{theorem}
\begin{proof}
The proof follows along the same lines as the proof of Theorem~\ref{th:two-point}. Indeed, we can show that the points $\tX_t = x_t + \delta S$ on which $\ell_t$ is queried belong to $\cK$. Then, using an easy modification of Lemma~\ref{l:two-point} we get that for all $x\in\cK$,
\begin{align*}
    \sum_{t=1}^n & \E(\ell_t(\tX_t)| X_t^+, X_t^-) - \sum_{t=1}^n \ell_t(x)
\\& \le
    \sum_{t=1}^n \tloss_t(x_t) - \sum_{t=1}^n \tloss_t\Bigl(\bigl(1-\tfrac{\delta}{r}\bigr)x\Bigr) + \delta\left(3 + \frac{R}{r} \right) G n~.
\end{align*}
Applying Theorem~\ref{th:OSMD2} as in the proof of Theorem~\ref{th:two-point} gives
\begin{align*}
    \E \sum_{t=1}^n \tloss_t(x_t) - \E \sum_{t=1}^n \tloss_t\Bigl(\bigl(1-\tfrac{\delta}{r}\bigr)x\Bigr)
& \le
    \frac{R^2}{\eta} + \eta\sum_{t=1}^n \E \norm{\gtilde_t(x_t)}^2
\\& \le
    \frac{R^2}{\eta} + \frac{(dL)^2}{\delta^2}\eta n~.
\end{align*}
\end{proof}

\section{Nonlinear stochastic bandits}
\label{s:nonlinear-stochastic}
We conclude with a simple example of nonlinear bandits in the stochastic setting. Unlike the gain-based analysis of stochastic bandits of Chapter~\ref{stochastic}, here we keep in with the convention used throughout this chapter and work exclusively with losses.

We consider a simple unidimensional setting where arms are points in the unit interval $[0,1]$. If at time $t$ a point $x_t \in [0,1]$ is played, the loss is the realization of an independent random variable $Y_t\in [0,1]$ with expected value $\E[Y_t | x_t] = \mu(x_t)$, where $\mu : [0,1]\to [0,1]$ is a fixed but unknown mean loss function. Similarly to Chapter~\ref{stochastic}, here the pseudo-regret after $n$ plays of a given strategy can be rewritten as
\[
    \oR_n = \sum_{t=1}^n \mu(x_t) - n\max_{x\in [0,1]}\mu(x)
\]
where $x_1,\dots,x_n \in [0,1]$ denote the points played by the strategy.

Throughout this section, we assume that $\mu : [0,1] \to [0,1]$ is unimodal, but not necessarily convex. This means there exist a unique $x^* = \argmin_{x \in [0,1]} \mu(x)$ such that $\mu(x)$ is monotone decreasing for $x \in [0,x^*]$ and monotone increasing for $x \in [x^*,1]$. For example, if $\mu$ can be written as $\mu(x) = x\,f(x)$ where $f : [0,1] \to [0,1]$ is differentiable, monotone decreasing, and such that $x\,f'(x)$ is strictly decreasing with $f(0) > 0$, then $\mu$ is unimodal.

The bandit strategy we analyze in this section is based on the golden section search due to \cite{kiefer1953sequential}, which is a general algorithm for finding the extremum of a unimodal function. Similarly to binary search, each step of golden section search narrows the interval in which the extremum is found by querying the function value at certain points that are chosen depending on the outcome of previous queries. Each query shrinks the interval by a factor of $\tfrac{1}{\varphi} = 0.618\dots$, where $\varphi =  \tfrac{1}{2}\bigl(1 + \sqrt{5}\bigr)$ is the golden ratio.

In our case, queries (i.e., plays) at $x$ return a perturbed version of $\mu(x)$. Since $\mu$ is bounded, Hoeffding bounds ensure that we can find the minimum of $\mu$ by repeatedly querying each point $x$ requested by the golden search algorithm. However, in order to have a lower bound on the accuracy with which each $\mu$ needs to be estimated, we must assume the following condition: there exists $C_L > 0$ such that
\begin{equation}
\label{eq:strongmax}
    \bigl|\mu(x)-\mu(x')\bigr| \ge C_L|x-x'|
\end{equation}
for each $x,x'$ that belong either to $[0,x^*-1/C_L]$ or to $[x^*+1/C_L,1]$.

Finally, irrespective to the uncertainty in the evaluation of $\mu$, in order to bound the regret incurred by golden section search we need a Lipschitz condition on $\mu$. Namely, there exists $C_H > 0$ such that $\bigl|\mu(x)-\mu(x')\bigr| \le C_H|x-x'|$ for all $x,x'\in [0,1]$.

We are now ready to introduce our stochastic version of the golden section search algorithm.
\bookbox{
{\em SGS (Stochastic Golden Search):}

\smallskip\noindent
Parameters: $\ve_1,\ve_2,\dots > 0$.\\
Initialize:
$
    x_A = 0 \quad x_B = \frac{1}{\varphi^2} \quad x_C = 1
$.

\smallskip\noindent
For each stage $s=1,\dots,n$
\begin{itemize}
\item[(1)] Let
${\dt
x_B' = \left\{ \begin{array}{cl}
        x_B - \tfrac{1}{\varphi^2}(x_B - x_A) & x_B - x_A > x_C - x_B
    \\[1mm]
        x_B + \tfrac{1}{\varphi^2}(x_C - x_B) & \text{otherwise}
    \end{array} \right.
}$
and rename points $x_B,x_B'$ so that $x_A < x_B < x_B' < x_C$.
\item[(2)] Play each point in $\{x_A,x_B,x_B',x_C\}$ for $\tfrac{2}{\ve_s^2}\ln(6n)$ times and let $\xhat$ be the point with lowest total loss in this stage.
\item[(3)] If $\xhat\in\{x_A,x_B\}$ then eliminate interval $(x_B',x_C]$ and let $x_C = x_B'$,
\item[(4)] else eliminate interval $[x_A,x_B)$ and let $x_A = x_B$.
\end{itemize}
}

Recall that golden section search proceeds as follows: given three queried points $x_A < x_B < x_C$ where the distance of $x_B$ to the other two points is in the golden ratio ($x_B$ might be closer to $x_A$ or to $x_C$ depending on past queries), the next point $x_B'$ is queried in the largest interval between $x_B-x_A$ and $x_C-x_B$ so that the distance of $x_B'$ to the extrema of that largest interval is in the golden ratio. Assume the resulting ordering is $x_A < x_B < x_B' < x_C$. Then we drop either $[x_A,x_B)$ or $(x_B',x_C]$ according to whether the smallest value of $\mu$ is found in, respectively, $\{x_B',x_C\}$ or $\{x_B',x_C\}$. The remaining triplet is such that the distance of the middle point to the other two is again in the golden ratio.

Using elementary algebraic identities for $\varphi$, one can show that setting $x_C-x_A = 1$ the following equalities hold at any step of SGS:
\begin{equation}
\label{eq:sgs-ratios}
    x_B-x_A = \frac{1}{\varphi^2} \quad x_B'-x_B = \frac{1}{\varphi^3} \quad x_C - x_B' = \frac{1}{\varphi^2}~.
\end{equation}
Since either $x_B-x_A$ or $x_C - x_B'$ are eliminated at each stage, at each stage SGS shrinks the search interval by a factor of $1 - \varphi^{-2} = \tfrac{1}{\varphi}$.

Let $[x_{A,s},x_{B,s}]$ be the search interval at the beginning of stage $s+1$, where $x_{A,0} = 0$ and $x_{C,0} = 1$.
\begin{lemma}
If $\ve_s = C_L\varphi^{-(s+3)}$ then
\[
    \P\Bigl(x^* \not\in [x_{A,s},x_{C,s}]\Bigr) \le \frac{s}{n}
\]
holds uniformly over all stages $s \ge 1$.
\end{lemma}
\begin{proof}
Once the interval containing $x^*$ is eliminated it is never recovered, thus we have
\begin{align}
\nonumber
    \P&\Bigl(x^* \not\in [x_{A,s},x_{C,s}]\Bigr)
\le
    \P\Bigl(x^* \not\in [x_{A,s-1},x_{C,s-1}]\Bigr)
\\ &\quad
\label{eq:sgs-rec}
    + \P\Bigl(x^* \not\in [x_{A,s},x_{C,s}] \,\Big|\, x^* \in [x_{A,s-1},x_{C,s-1}]\Bigr)~.
\end{align}
Let $X_s = \{x_{A,s-1},x_{B,s-1},x_{B,s-1}',x_{C,s-1}\}$ where $x_{B,s-1} < x_{B,s-1}'$ are computed in step $1$ of stage $s$.
Let $\muhat_s(x)$ be the sample loss of point $x\in X_s$ in stage $s$ and let $\xhat_s = \argmin_{x\in X_s}\muhat(x)$. Since at stage $s$ every point in $X_s$ is played $\tfrac{2}{\ve_s^2}\ln(6n)$ times\footnote{For simplicity, we assume these numbers are integers.}, Hoeffding bounds imply that $\bigl|\mu(x)-\muhat_s(x)\bigr| \le \tfrac{1}{2}\ve_s$ with probability at least $1-\tfrac{1}{6n}$ for all $x\in X_s$ simultaneously. Let
\[
    x^*_s = \argmin_{x\in X_s}\mu(x)~.
\]
Now assume $x^* \in [x_{A,s-1},x_{B,s-1}]$. Then $x^* \not\in [x_{A,s},x_{C,s}]$ implies $\muhat_s(x_{B',s-1}) < \muhat(x_{B,s-1})$ or $\muhat_s(x_{C,s-1}) < \muhat(x_{B,s-1})$. Similarly, assume $x^* \in [x_{B',s-1},x_{C,s-1}]$. Then $x^* \not\in [x_{A,s},x_{C,s}]$ implies $\muhat_s(x_{A,s-1}) < \muhat(x_{B',s-1})$ or $\muhat_s(x_{B,s-1}) < \muhat(x_{B',s-1})$. In both cases, we need to compare three values of $\mu$ on the same side with respect to $x^*$. (When $x^* \in [x_{B,s-1},x_{B',s-1}]$ we always have $x^* \in [x_{A,s},x_{C,s}]$.) Hence, we can apply our assumption involving $C_L$. More precisely, (\ref{eq:sgs-ratios}) implies that after $s$ stages the search interval has size $\varphi^{-s}$ and $\min\{x_{B,s}-x_{A,s}, x_{B,s}'-x_{B,s}, x_{C,s} - x_{B,s}'\} = \varphi^{-(s+3)}$~. Hence, introducing
\[
    \Delta_s = \min\Bigl\{\big|\mu(x_{B,s})-\mu(x_{A,s})\big|, \big|\mu(x_{B,s}')-\mu(x_{B,s})\big|, \big|\mu(x_{C,s}) - \mu(x_{B,s}')\big|\Bigr\}
\]
we have
\[
    \Delta_s
\ge
    C_L\min\{x_{B,s}-x_{A,s}, x_{B,s}'-x_{B,s}, x_{C,s} - x_{B,s}'\} \ge C_L\varphi^{-(s+3)} = \ve_s~.
\]
Let $T_s = \tfrac{8}{\ve_s^2}\ln(6n)$ the length of stage $s$. We can write
\begin{align*}
    \P&\Bigl(x^* \not\in [x_{A,s},x_{C,s}] \,\Big|\, x^* \in [x_{A,s-1},x_{C,s-1}]\Bigr)
=
    \P\bigl(\muhat_s(\xhat_s) < \muhat(x^*_s)\bigr)
\\ &\le
    \sum_{x \in X_s\setminus\{x^*_s\}}\P\bigl(\muhat_s(x) < \muhat(x^*_s)\bigr)
\\ &\le
    \sum_{x \in X_s\setminus\{x^*_s\}}\left( \P\left(\muhat_s(x) < \mu(x) - \frac{\Delta_s}{2}\right) + \P\left(\mu(x^*_s) < \muhat(x^*_s) - \frac{\Delta_s}{2}\right)\right)
\\ &\le
    6\,e^{-T_s\Delta_s^2/8} 
\\ &\le
    6\,e^{-T_s\ve_s^2/8} 
\le
    \frac{1}{n}~.
\end{align*}
Substituting this in~(\ref{eq:sgs-rec}) and recurring gives the desired result.
\end{proof}
\begin{theorem}[Regret of SGS]
For any unimodal and $C_H$-Lipschitz mean loss function $\mu : [0,1] \to [0,1]$ that satisfies~(\ref{eq:strongmax}),
if the SGS algorithm is run with $\ve_s = C_L\varphi^{-(s+3)}$ then
\[
    \oR_n \le \frac{C_H}{C_L^2}8\varphi^6\ln(6n) \left[ \frac{2\varphi}{\varphi-1}\sqrt{1 + C_L^2 n} + \frac{1}{4}\log_{\phi}^2\bigl(1 + C_L^2 n\bigr) \right]~.
\]
\end{theorem}
\begin{proof}
We start by decomposing the pseudo-regret as follows,
$$\oR_n \leq \sum_{s=1}^S T_s \left(\min_{x \in A_s}\mu(x) - \mu(x^*)\right) + \sum_{s=1}^S \left( \sum_{t\in T_s} \mu(x_t) - T_s \min_{x \in A_s} \mu(x) \right) .$$
Using the Lipschitz assumption
\[
    \max_{x,x'\in A_s}\bigl|\mu(x)-\mu(x')\bigr| \le C_H\,\bigl|x_{C,s} - x_{A,s}\bigr|
\]
and recalling that $\bigl|x_{C,s} - x_{A,s}\bigr| \le \varphi^{-s}$, we bound the first term in this decomposition as follows
\begin{align*}
    \sum_{s=1}^S T_s \left(\min_{x \in A_s}\mu(x) - \mu(x^*)\right)
&\le
    T_s C_H \bigl|x_{C,s} - x_{A,s}\bigr| \, \P\Bigl(x^* \in [x_{A,s},x_{C,s}]\Bigr)
\\ &\quad
    + T_s C_H \P\Bigl(x^* \not\in [x_{A,s},x_{C,s}]\Bigr)
\\ &\le
    \frac{T_s C_H}{\varphi^s} + T_s C_H \frac{s}{n}~.
\end{align*}
The second term is controlled similarly,
\begin{align*}
    \sum_{s=1}^S \left( \sum_{t\in T_s} \mu(x_t) - T_s \min_{x \in A_s} \mu(x) \right)
\le
    T_s C_H \bigl|x_{C,s} - x_{A,s}\bigr|
\le
    \frac{T_s C_H}{\varphi^s}~.
\end{align*}
Hence we get an easy expression for the regret,
\begin{align}
\nonumber
    \oR_n
&\le
    C_H \sum_{s=1}^S T_s\left(\frac{2}{\varphi^s} + \frac{s}{n}\right)
\\& \le
\label{eq:sgs-th}
    \frac{C_H}{C_L^2}8\varphi^6\ln(6n) \sum_{s=1}^S \varphi^{2s}\left(\frac{2}{\varphi^s} + \frac{s}{n}\right)~.
\end{align}
We now compute an upper bound on the number $S$ of stages,
\begin{align*}
    n \le \sum_{s=1}^S T_s = \frac{8\varphi^6}{C_L^2}\ln(6n)\sum_{s=1}^S \varphi^{2s} \le \frac{8\varphi^6}{C_L^2}\ln(6n)\frac{\varphi^{2S+2}}{\varphi^2 - 1}~.
\end{align*}
Solving for $n$ and overapproximating we get
\[
    S \le \frac{1}{2}\log_{\phi}\bigl(1 + C_L^2 n\bigr)~.
\]
Therefore, the sum in~(\ref{eq:sgs-th}) is bounded as follows 
\begin{align*}
    2\sum_{s=1}^S \varphi^s + S^2
&\le
    \frac{2\varphi}{\varphi-1}\varphi^S + S^2
\\&\le
    \frac{2\varphi}{\varphi-1}\sqrt{1 + C_L^2 n} + \frac{1}{4}\log_{\phi}^2\bigl(1 + C_L^2 n\bigr)~.
\end{align*}
Substituting the above in~(\ref{eq:sgs-th}) concludes the proof.
\end{proof}

\section{Bibliographic remarks}
Gradient-free methods for stochastic approximation have been studied for a long time ---see the bibliographic remarks at the end of Chapter~\ref{linear} for some references. However, relatively few results provide regret bounds. The approach presented in this chapter for online convex optimization with bandit feedback was pioneered by \citep{FKM05} and~\citep{kleinberg2004nearly}. The improved rate for the two-points model was later shown in \cite{ADX10}. 

While the results for nonlinear bandits in the adversarial model are still scarse, there is a far richer body of work in the stochastic model. The result based on golden section search presented in Section~\ref{s:nonlinear-stochastic} is due to~\citep{yu2011unimodal}. It represents only a tiny portion of the known results in the stochastic model. In the general case of Lipschitz mean-payoff on a compact subset of $\R^d$, it can be shown that the minimax regret is $\Omega\bigl(n^{\frac{d+1}{d+2}}\bigr)$. Thus the rate rapidly deteriorates as the dimension increases, a phenomenon known as the {\em curse of dimensionality}. However it was shown in \citep{KSU08} and \citep{BMSS09} that under a generalized version of equation \eqref{eq:strongmax} it is possible to circumvent the curse of dimensionality and obtain a regret of $\tilde{O}(\sqrt{n})$. This result builds upon and generalizes a sequence of works that include the discretization approach (for the $1$-dimensional case) of \citep{kleinberg2004nearly} and \citep{auer2007improved}, as well as the method of \citep{cope2009regret} based on the Kiefer-Wolfowitz procedure (a classical method of stochastic optimization). The key new algorithmic idea introduced in \citep{KSU08} and \citep{BMSS09}  is to adaptively partition the set of actions in order to exploit the smoothness of the mean-payoff function around its maximum. We refer the reader to \cite{BMSS11} for the details of this result (which is much more general than what we briefly outlined, in particular it applies for metric spaces, or even more general action sets), as well as a more precise historical account. 

Another direction for nonlinear stochastic bandits was recently investigated in \citep{agarwal2011stochastic}. In this work the authors study the case of a convex mean loss function, and they show how to combine the zeroth-order optimization method of \citep{NY83} with a ``center point device'' to obtain a regret of $\tilde{O}(\sqrt{n})$.

A more general version of nonlinear stochastic bandit was also studied in \cite{AKS11}. In this paper the authors assume that the mean-payoff function lies in some known set of functions $\cF$. They define a notion of complexity for the class $\cF$, the haystack dimension $\mathrm{HD}(\cF)$, and they show that the worst case regret in $\cF$ is lower bounded by $\mathrm{HD}(\cF)$. Unfortunately their upper bound does not match the lower bound, and the authors suggest that the definition of the haystack dimension should be modified in order to obtain matching upper and lower bound.

Finally, a related problem in a Bayesian setting was studied in~\cite{SKKS10} and~\cite{GAOS10}, where it is assumed that the payoff functions are drawn from Gaussian processes.

%
%
%
%

%% file: variants.tex
In the previous chapters we explored a few fundamental variations around the basic multi-armed bandit problem. In both the stochastic and adversarial frameworks, these variants basically revolved around a single principle: by adding constraints on the losses (or rewards), it is possible to compete against larger sets of arms. While this is indeed a fundamental axis in the space of bandit problems, it is important to realize that there are many other directions. Indeed, we might sketch a ``bandit space'' spanning the following coordinates:
\begin{itemize}
\item Evolution of payoffs over time: stochastic, adversarial, Markovian, \dots
\item Structure of payoff functions: linear, Lipschitz, Gaussian process, \dots
\item Feedback structure: full information, bandit, semi-bandit, partial monitoring, \dots
\item Context structure (if any).
\item Notion of regret.
\end{itemize}
Clearly, such extensions greatly increase the number of potential applications of bandit models. While many of these extensions were already discussed in the previous chapters, in the following we focus on others (such as the sleeping bandits or the thruthful bandits) so to visit more exotic regions of the bandit space.

\section{Markov Decision Processes, restless and sleeping bandits}
\label{s:mdp}
Extending further the model of Markovian bandits (mentioned at the end of Chapter~\ref{intro}), one can also define a general Markov Decision Process (MDP) ---see also Section~\ref{s:mdp}. For example, the stochastic bandit of Chapter~\ref{stochastic} corresponds to a single-state MDP.

In full generality, a finite MDP can be described by a set of states $\{1, \hdots, S\}$, a set of actions $\{1,\hdots, K\}$, a set $\{p_{i,s},\, 1 \leq i \leq K,\, 1 \leq s \leq S\}$ of transition distributions over $S$, and a set $\{\nu_{i,s},\, 1 \leq i \leq K,\, 1 \leq s \leq S\}$ of reward distributions over $[0,1]$. In this model, taking action $i$ in state $s$ generates a stochastic reward drawn from $\nu_{i,s}$ and a Markovian transition to a state drawn from $p_{i,s}$. Similarly to the multi-armed bandit problem, here one typically assumes that the reward distributions and transition distributions are unknown, and the goal is to navigate through the MDP so as to maximize some function of the obtained rewards. The field that studies this type of problem is called Reinforcement Learning. The interested reader is addressed to \cite{SB98, Kak03, Sze10}. Reinforcement learning results with a flavor similar to those described in the previous chapters can be found in \cite{YMS09, BM10, JOA10, NGSA10}.

An intermediate model, between stochastic multi-armed bandits and MDPs, is the one of restless bandits. As in Markovian bandits, each arm is associated with a Markovian reward process with its own state space. Each time an arm is chosen, the associated Markov process generates an observable reward and makes a transition to a new state, which is also observed. However, unlike Markovian bandits an unobserved transition occurs for each arm that is not chosen. Using concentration inequalities for Markov chains ---see, e.g., \cite{Lez98}, one can basically show that, under suitable assumptions, UCB attains a logarithmic regret for restless bandits as well ---see \cite{TL11} and~\cite{filippi2011optimally}. A more general regret bound for restless bandits has been recently proven by~\cite{ortner2012regret}.

An apparently similar problem was studied by \cite{GM11}, where they assume that the reward distributions can abruptly change at unknown time instants (and there is a small number of such change-points). Within this model, the authors prove that the best possible regret is of order $\sqrt{n}$, which is matched by the Exp3.P algorithm ---see the discussion in Section \ref{sec:exp3S}. Thus, while the two problems (restless bandits and bandits with change-points) might look similar, they are fundamentally different. In particular, note that the latter problem cannot be cast as a MDP.

Another intermediate model, with important applications, is that of the sleeping bandits. There, it is assumed that the set of available actions is varying over time. We refer the interested reader to \cite{KNS10, KMB09, slivkins2009contextual, KS12} for the details of this model as well as the theoretical guarantees that can be obtained. A somewhat related problem was also studied in \cite{GKSS07} where it is assumed that the set of arms becomes unavailable for a random time after each arm pull (and the distribution of this random time depends on the selected arm).

\section{Pure exploration problems}
The focus of bandits, and most of their variants, is on problems where there is a notion of cumulative rewards, which is to be maximized. This criterion leaves out a number of important applications where there is an online aspect (e.g., sequential decisions), but the goal is not maximizing cumulative rewards. The simplest example is perhaps the pure exploration version of stochastic bandits. In this model, at the end of round $n$ the algorithm has to output a recommendation $J_n$ which represents its estimate for the optimal arm. The focus here is on the control of the so-called simple regret, introduced by \cite{BMS09, BMS11} and defined as $r_n = \mu^* - \E\,\mu_{J_n}$.

\cite{BMS09} prove that minimizing the simple regret is fundamentally different from minimizing the pseudo-regret $\oR_n$, in the sense that one always have $r_n \geq \exp(- C \oR_n)$ for some constant $C>0$ (which depends on the reward distributions). Thus, this regret calls for different bandit algorithms. \cite{ABM10} exhibit a simple strategy with optimal performances up to a logarithmic factor. The idea is very simple: the strategy SR (Successive Rejects) works in $K-1$ phases. SR keeps a set of active arms, that are sampled uniformly in each phase. At the end of a phase, the arm with smallest empirical mean is removed from the set of active arms. It can be shown that this strategy has a simple regret of order $\exp\left(- c \frac{n}{H \ln K}\right)$, where $H = \sum_{i \neq i^*} \frac{1}{\Delta_i^2}$ is the complexity measure of identifying the best arm, and $c$ is a numerical constant. Moreover, a matching lower bound (up to logarithmic factors) was also proved. These ideas were extended in different ways by \cite{GGLB11, BJM11, BWV12}.

A similar problem was studied in a PAC model by \cite{EMM02}. The goal is to find, with probability at least $1-\delta$, an arm with mean at least $\epsilon$ close the optimal mean, and the relevant quantity is the number of pulls needed to achieve this goal. For this problem, the authors derive an algorithm called Successive Elimination that achieves an optimal number of pulls (up to logarithmic factors). Successive Elimination works as follows: it keeps an estimate of the mean of each arm, together with a confidence interval. If two confidence intervals are disjoint, then the arm with the lowest confidence interval is eliminated. Using this procedure, one can achieve the $(\epsilon, \delta)$ guarantee with a number of pulls of order $H \ln \tfrac{K}{\Delta}$. A matching lower bound is due to \cite{MT04}, and further results are discussed by \cite{EMM06}.

In some applications one is not interested in the best arm, but rather in having a good estimate of the mean $\mu_i$ for each arm. In this setting a reasonable measure of performance is given by
$$L_n = \E \sum_{i=1}^K \bigl(\mu_i - \hat{\mu}_{i,T_i(n)}\bigr)^2~.$$
Clearly, the optimal static allocation depends only on the variances of the arms, and we denote by $L_n^*$ the performance of this strategy. 
This setting was introduced by \cite{AGS08}, where the authors studied the regret $L_n - L_n^*$, and showed that a regret of order $n^{-3/2}$ was achievable. This result was then refined by \cite{CLGMA11, CM11}. The basic idea in these papers is to resort to the optimism in face of uncertainty principle, and to approximate the optimal static allocation by replacing the true variance with an upper confidence bound on it.

\section{Dueling bandits} \label{sec:dueling}
An interesting variation of stochastic bandits was recently studied by \cite{YBKJ09}. The model considered in that paper is called dueling bandits. The main idea is that the player has to choose a pair or arms $(I_t, J_t)$ at each round, and can only observe the relative performances of these two arms, i.e., the player only knows which arm had the highest reward. More formally, in dueling bandits we assume that there exists a total ordering $\succ$ on $\{1,\hdots,K\}$ with the following properties:
\begin{enumerate}
\item If $i \succ j$, then the probability that the reward of arm $i$ is larger than the reward of arm $j$ is equal to $\frac12 + \Delta_{i,j}$ with $\Delta_{i,j} >0$.
\item If $i \succ j \succ k$, then $\Delta_{i,j} + \Delta_{j,k} \geq \Delta_{i,k} \geq \max\bigl\{\Delta_{i,j}, \Delta_{j,k}\bigr\}$.
\end{enumerate}
Upon selecting a pair $(I_t, J_t)$, the player receives a random variable drawn from a Bernoulli distribution with parameter $\frac12+\Delta_{i,j}$. In this setting a natural regret notion is the following quantity, where $i^*$ is the largest element in the ordering $\succ$,
$$\E \sum_{t=1}^n \bigl(\Delta_{i^*, I_t} + \Delta_{i^*, J_t} \bigr) .$$
It was proved in \cite{YBKJ09} that the optimal regret for this problem is of order $\frac{K}{\Delta} \log n$, where $\Delta = \min_{i \neq j} \Delta_{i,j}$. A simple strategy that attains this rate, based on the Successive Elimination algorithm of \cite{EMM02}, was proposed by \cite{YJ11}.

\section{Discovery with probabilistic expert advice}
\cite{BEG11} study a model with a stochastic bandit flavor (in fact it can be cast as an MDP), where the key for the analysis is a sort of 'non-linear' regret bound.
In this model rewards represent items in some set $\cX$ which is partitioned in a subset  $A \subset \cX$ of interesting items and in a subset $\cX\setminus A$ of non-interesting items. The goal is to maximize the total expected number of interesting items found after $n$ pulls, where observing twice the same item does not help. A natural notion of regret is obtained by comparing the number of interesting items $F(n)$ found by a given strategy to the number $F^*(n)$ found by the optimal strategy. It turns out that analyzing such regret directly is difficult. The first issue is that in this problem the notion of a ``good'' arm is dynamic, in the sense that an arm could be very good for a few pulls and then completely useless. Furthermore a strategy making bad decisions in the beginning will have better opportunities in the future than the optimal strategy (which already discovered some interesting items). Taking into account these issues, it turns out that it is easier to show that for good strategies, $F(n)$ is not too far from $F^*(n')$, where $n'$ is not much smaller than $n$. Such a statement -- which can be interepreted as a non-linear regret bound -- shows that the analyzed strategy slightly 'lags' behind the optimal strategy. In \cite{BEG11} a non-linear regret bound is derived for an algorithm based on estimating the mass of interesting items left on each arm (the so-called Good-Turing estimator), combined with the optimism in face of uncertainty principle of Chapter \ref{stochastic}. We refer the reader to \cite{BEG11} for more precise statements.

\section{Many-armed bandits}
The many-armed bandit setting was introduced by \cite{BCZHS97}, and then extended and refined by \cite{WAM08}. This setting corresponds to a stochastic bandit with an infinite number of arms. The extra assumption that makes this problem feasible is a prior knowledge on the distribution of the arms. More precisely, when the player ask to ``add'' a new arm to his current set of active arms, one assumes that the probability that this arm is $\epsilon$-optimal is of order $\epsilon^{\beta}$, for some known $\beta > 0$. Thus the player faces a trade-off between exploitation, exploration, and discovery, where the last component comes from the fact that the player needs to consider new arms to make sure that he has an active $\epsilon$-optimal arm. Using a UCB strategy on the active arms, and by adding new arms at a rate which depends on $\beta$, \cite{WAM08} prove that a regret of order
\[
    n^{\max \bigl\{\frac12,\, \frac{\beta}{1+\beta}\bigr\}}
\]
is achievable in this setting.

\section{Truthful bandits} \label{sec:truthful}
A popular application domain for bandit algorithms is ad placement on the Web. In the pay-per-click model, for each incoming user $t=1,2,\dots$ the publisher selects an advertiser $I_t$ from a pool of $K$ advertisers, and display the corresponding ad to the user. The publisher then gets a reward if the ad is clicked by the user. This problem is well modeled by the multi-armed bandit setting. However, there is a fundamental aspect of the ad placement process which is overlooked by this formulation. Indeed, prior to running an ad-selection algorithm (i.e., a bandit algorithm), each advertiser $i \in \{1, \hdots, K\}$ issues a bet $b_i$. This number is how much $i$ is willing to pay for a click. Each bidder keeps also a private value $v_i$, which is the true value $i$ assigns to a click. Because a rational bidder ensures that $b_i \le v_i$, the difference $v_i-b_i$ defines the utility for bidder $i$. The basic idea of {\em truthful} bandits is to construct a bandit algorithm such that each advertiser has no incentive in submitting a bet $b_i$ such that $b_i < v_i$. A natural question to ask is whether this restriction to truthful algorithms changes the dynamics of the multi-armed bandit problem. This has investigated in a number of papers, including~\cite{BSS09, DK09, BKS10, WS12}. Thruthful bandits are part of a more general thread of research at the interface between bandits and Mechanism Design.

\section{Concluding remarks}
As pointed out in the introduction, the growing interest for bandits arises from the large number of industrially relevant problems that can be modeled as a multi-armed bandit. In particular, the sequential nature of the bandit setting makes it perfectly suited to various Internet and Web applications. These include search engine optimization with dueling bandits, or ad placement with contextual bandits and truthful bandits, see the references in, respectively, Section~\ref{sec:dueling}, Section~\ref{sec:contextual} and Section~\ref{sec:truthful}.

Multi-armed bandits also proved to be very useful in other areas. For example, thanks to the strong connections between bandits and Markov Decision Processes, a breakthrough in Monte Carlo Tree Search (MCTS) was achieved using bandits ideas. More precisely, based on the sparse planning idea of \cite{KMN02}, \cite{KS06} introduced a new MCTS strategy called UCT (UCB applied to Trees) that led to a substantial advancement in Computer Go performance, see \cite{GWMT06}. Note that, from a theoretical point of view UCT was proved to perform poorly by \cite{CM07}, and a strategy based on a similar idea, but with improved theoretical performance, was proposed by \cite{BM10}. Other applications in related directions have also been explored, see for example \cite{TT09, HT10b} and many others.

Many new domains of application for bandits problems are currently investigated. For example: multichannel opportunistic communications~\cite{LZK10}, model selection~\cite{ADBL11}, boosting~\cite{BK11}, management of dark pools of liquidity (a recent type of stock exchange)~\cite{ABD10}, security analysis of power systems~\cite{BEG11}.

Given the fast pace of new variants, extensions, and applications coming out every week, we had to make tough decisions about what to present in this survey. We apologize for everything we had to leave out. On the other hand, we do hope that what we decided to put in will enthuse more researchers about entering this exciting field.